%% file: main.tex
\documentclass[11pt,letterpaper]{article}
\usepackage[T1]{fontenc}
\usepackage[utf8]{inputenc}

\usepackage{amsmath,amssymb,amsthm,mathtools}
\usepackage{newpxtext}                 
\usepackage{newpxmath}                 
\usepackage[scaled=0.90]{FiraSans}     

\usepackage[letterspace=80]{microtype}

\usepackage[letterpaper,left=1.15in,right=1.15in,top=1.1in,bottom=1.15in,headsep=0.32in,footskip=0.45in]{geometry}
\usepackage{xcolor}
\definecolor{ink}{HTML}{1A1F2B}
\definecolor{navy}{HTML}{1F3A63}         
\definecolor{refcool}{HTML}{2C5AA0}      
\definecolor{citewarm}{HTML}{B5403B}     
\definecolor{muted}{HTML}{6B6A66}        
\definecolor{rulegray}{HTML}{D3D6DC}
\definecolor{thmback}{HTML}{F2F6FC}
\definecolor{lemback}{HTML}{F5F6F8}
\definecolor{lembar}{HTML}{6B7A90}
\definecolor{absback}{HTML}{F7F8FA}

\usepackage{graphicx,booktabs,array,tabularx,multicol}
\usepackage{enumitem}
\setlist{itemsep=1pt,parsep=0pt,topsep=3pt,partopsep=0pt}
\usepackage[font=small,labelfont={sf,bf,color=navy},labelsep=period,skip=6pt]{caption}
\usepackage[round]{natbib}

\usepackage[most]{tcolorbox}
\usepackage{titlesec,titletoc}
\usepackage{fancyhdr}
\usepackage{bm,aliascnt}
\usepackage{tikz}
\usepackage{pgfplots}
\usepackage{placeins}
\usepackage{etoolbox}
\usetikzlibrary{arrows.meta,calc,positioning,decorations.pathreplacing,patterns,shapes.geometric,fit,backgrounds}
\pgfplotsset{compat=1.18}

\titleformat{\section}{\sffamily\Large\bfseries\color{navy}}{\textcolor{refcool}{\thesection}}{0.7em}{}
\titlespacing*{\section}{0pt}{3.2ex plus 1ex minus .2ex}{1.5ex plus .2ex}
\titleformat{\subsection}{\sffamily\large\bfseries\color{navy}}{\textcolor{refcool}{\thesubsection}}{0.6em}{}
\titlespacing*{\subsection}{0pt}{2.4ex plus .8ex minus .2ex}{1.0ex plus .2ex}
\titleformat{\subsubsection}{\sffamily\normalsize\bfseries\color{navy}}{\textcolor{refcool}{\thesubsubsection}}{0.5em}{}
\titleformat{\paragraph}[runin]{\sffamily\bfseries\color{navy}}{}{0pt}{}
\titlespacing*{\paragraph}{0pt}{1.4ex plus .5ex minus .2ex}{0.55em}

\titlecontents{section}[2.0em]{\addvspace{0.6pt}\sffamily\fontsize{8.4}{10.2}\selectfont}
  {\contentslabel[\textcolor{refcool}{\bfseries\thecontentslabel}]{2.0em}}{\hspace*{-2.0em}}
  {\titlerule*[0.6pc]{\textcolor{rulegray}{.}}\contentspage}

\fancypagestyle{first}{\fancyhf{}\fancyfoot[C]{\sffamily\footnotesize\color{muted}\thepage}}

\newtheoremstyle{boxed}{0pt}{0pt}{\itshape}{}{\sffamily\bfseries\color{navy}}{.}{0.5em}
  {\thmname{#1}\thmnumber{ #2}\thmnote{\ \textmd{(#3)}}}
\newtheoremstyle{boxedrm}{0pt}{0pt}{\normalfont}{}{\sffamily\bfseries\color{navy}}{.}{0.5em}
  {\thmname{#1}\thmnumber{ #2}\thmnote{\ \textmd{(#3)}}}
\theoremstyle{boxed}
\newtheorem{theorem}{Theorem}[section]
\newtheorem{proposition}[theorem]{Proposition}
\newtheorem{lemma}[theorem]{Lemma}
\newtheorem{corollary}[theorem]{Corollary}

\theoremstyle{boxedrm}
\newtheorem{definition}[theorem]{Definition}

\newtheorem{remark}[theorem]{Remark}

\tcbset{thmbox/.style={enhanced,breakable,frame hidden,sharp corners,boxrule=0pt,
  left=8pt,right=8pt,top=5pt,bottom=5pt,before skip=9pt,after skip=9pt}}
\tcbset{coolbox/.style={thmbox,colback=thmback,borderline west={2.4pt}{0pt}{refcool}}}
\tcbset{slatebox/.style={thmbox,colback=lemback,borderline west={2.4pt}{0pt}{lembar}}}
\tcbset{warmbox/.style={thmbox,colback=white,borderline west={2.4pt}{0pt}{citewarm}}}
\tcbset{linebox/.style={blanker,breakable,left=8pt,before skip=7pt,after skip=7pt,borderline west={1pt}{0pt}{rulegray}}}
\tcolorboxenvironment{theorem}{coolbox}
\tcolorboxenvironment{proposition}{coolbox}
\tcolorboxenvironment{corollary}{coolbox}
\tcolorboxenvironment{conjecture}{coolbox}
\tcolorboxenvironment{lemma}{slatebox}
\tcolorboxenvironment{claim}{slatebox}
\tcolorboxenvironment{fact}{slatebox}
\tcolorboxenvironment{observation}{slatebox}
\tcolorboxenvironment{definition}{warmbox}
\tcolorboxenvironment{assumption}{warmbox}
\tcolorboxenvironment{remark}{linebox}
\tcolorboxenvironment{example}{linebox}
\tcolorboxenvironment{problem}{linebox}
\tcolorboxenvironment{question}{linebox}

\makeatletter
\renewenvironment{proof}[1][\proofname]{\par\pushQED{\qed}\normalfont\topsep6\p@\@plus6\p@\relax
  \trivlist\item[\hskip\labelsep{\sffamily\bfseries\color{navy}#1.}]\ignorespaces}{\popQED\endtrivlist\@endpefalse}
\makeatother

\makeatletter
\newcommand{\R}{\mathbb{R}}

\newcommand{\Exp}{\mathbb{E}}
\newcommand{\Prob}{\mathbb{P}}
\newcommand{\one}{\mathbf{1}}
\newcommand{\Phibar}{\overline{\Phi}}
\newcommand{\cF}{\mathcal{F}}
\newcommand{\cU}{\mathcal{U}}
\newcommand{\sigd}{\sigma_d}
\newcommand{\lhat}{\hat\lambda}
\newcommand{\Boxz}[1]{\mathsf{Box}_{#1}}
\newcommand{\Ballz}[1]{\mathsf{Ball}_{#1}}
\newcommand{\Dbox}{D_{\mathsf{box}}}
\newcommand{\Dball}{D_{\mathsf{ball}}}
\newcommand{\Vbox}{V_{\mathsf{box}}}
\newcommand{\Vball}{V_{\mathsf{ball}}}
\newcommand{\THR}{\mathrm{THR}}
\newcommand{\ANDg}{\mathrm{AND}}
\newcommand{\ORg}{\mathrm{OR}}
\newcommand{\MAJg}{\mathrm{MAJ}}
\newcommand{\ip}[2]{\langle #1,#2\rangle}
\newcommand{\abs}[1]{\lvert #1\rvert}
\newcommand{\norm}[1]{\lVert #1\rVert}
\newcommand{\Rep}{\mathrm{Rep}}
\DeclareMathOperator{\Var}{Var}
\DeclareMathOperator{\Cov}{Cov}
\DeclareMathOperator{\Bin}{Bin}

\definecolor{wongblue}{HTML}{0072B2}
\definecolor{wongorange}{HTML}{E69F00}
\definecolor{wongsky}{HTML}{56B4E9}
\definecolor{wonggreen}{HTML}{009E73}
\definecolor{wongyellow}{HTML}{F0E442}
\definecolor{wongverm}{HTML}{D55E00}
\definecolor{wongpurple}{HTML}{CC79A7}
\definecolor{referenceblue}{HTML}{285A7D}
\definecolor{citationrust}{HTML}{9B4D30}
\makeatother

\usepackage[colorlinks=true,linkcolor=refcool,citecolor=citewarm,urlcolor=refcool,
  pdftitle={The Ball and the Box: Two Geometries of Computation in Superposition},pdfauthor={Xiaoyu Li, Lequan Lin, Dai Shi, Jiaojiao Jiang, Junbin Gao, Andi Han},pdfsubject={Two Geometries of Computation in Superposition},pdfkeywords={computation in superposition, sparse feature representations, threshold circuits, Gaussian extreme values, sharp thresholds, reliability, mechanistic interpretability}]{hyperref}
\usepackage[nameinlink,capitalize,noabbrev]{cleveref}
\crefname{lemma}{lemma}{lemmas}
\Crefname{lemma}{Lemma}{Lemmas}
\crefname{proposition}{proposition}{propositions}
\Crefname{proposition}{Proposition}{Propositions}
\crefname{corollary}{corollary}{corollaries}
\Crefname{corollary}{Corollary}{Corollaries}
\hypersetup{hypertexnames=false}
\AddToHook{env/assumption/begin}{\crefalias{theorem}{assumption}}
\AddToHook{env/claim/begin}{\crefalias{theorem}{claim}}
\AddToHook{env/conjecture/begin}{\crefalias{theorem}{conjecture}}
\AddToHook{env/corollary/begin}{\crefalias{theorem}{corollary}}
\AddToHook{env/definition/begin}{\crefalias{theorem}{definition}}
\AddToHook{env/example/begin}{\crefalias{theorem}{example}}
\AddToHook{env/fact/begin}{\crefalias{theorem}{fact}}
\AddToHook{env/lemma/begin}{\crefalias{theorem}{lemma}}
\AddToHook{env/observation/begin}{\crefalias{theorem}{observation}}
\AddToHook{env/problem/begin}{\crefalias{theorem}{problem}}
\AddToHook{env/proposition/begin}{\crefalias{theorem}{proposition}}
\AddToHook{env/question/begin}{\crefalias{theorem}{question}}
\AddToHook{env/remark/begin}{\crefalias{theorem}{remark}}

\begin{document}
\thispagestyle{first}

\noindent{\color{refcool}\rule{\linewidth}{2.2pt}}\par\vspace{7pt}
\noindent{\sffamily\footnotesize\color{muted}\lsstyle PREPRINT\quad\textbullet\quad October 2026}\par\vspace{8pt}
\noindent{\sffamily\bfseries\fontsize{23}{28}\selectfont\color{navy}The Ball and the Box\par}\vspace{7pt}
\noindent{\sffamily\large\color{citewarm}Two Geometries of Computation in Superposition\par}\vspace{12pt}
\noindent\begin{tabular}{@{}>{\raggedright\arraybackslash}p{\dimexpr(\linewidth-20pt)/3\relax}@{\hspace{10pt}}>{\raggedright\arraybackslash}p{\dimexpr(\linewidth-20pt)/3\relax}@{\hspace{10pt}}>{\raggedright\arraybackslash}p{\dimexpr(\linewidth-20pt)/3\relax}@{}}
{\sffamily\bfseries\color{ink}Xiaoyu Li} & {\sffamily\bfseries\color{ink}Lequan Lin} & {\sffamily\bfseries\color{ink}Dai Shi} \\[1pt]
{\small University of New South Wales} & {\small University of Sydney} & {\small University of Cambridge} \\[1pt]
{\sffamily\footnotesize\href{mailto:xiaoyu.li2@unsw.edu.au}{\color{muted}xiaoyu.li2@unsw.edu.au}} & {\sffamily\footnotesize\href{mailto:lequan.lin@sydney.edu.au}{\color{muted}lequan.lin@sydney.edu.au}} & {\sffamily\footnotesize\href{mailto:ds2213@cam.ac.uk}{\color{muted}ds2213@cam.ac.uk}}
\end{tabular}\par\vspace{6pt}\noindent\begin{tabular}{@{}>{\raggedright\arraybackslash}p{\dimexpr(\linewidth-20pt)/3\relax}@{\hspace{10pt}}>{\raggedright\arraybackslash}p{\dimexpr(\linewidth-20pt)/3\relax}@{\hspace{10pt}}>{\raggedright\arraybackslash}p{\dimexpr(\linewidth-20pt)/3\relax}@{}}
{\sffamily\bfseries\color{ink}Jiaojiao Jiang} & {\sffamily\bfseries\color{ink}Junbin Gao} & {\sffamily\bfseries\color{ink}Andi Han} \\[1pt]
{\small University of New South Wales} & {\small University of Sydney} & {\small University of Sydney} \\[1pt]
{\sffamily\footnotesize\href{mailto:jiaojiao.jiang@unsw.edu.au}{\color{muted}jiaojiao.jiang@unsw.edu.au}} & {\sffamily\footnotesize\href{mailto:junbin.gao@sydney.edu.au}{\color{muted}junbin.gao@sydney.edu.au}} & {\sffamily\footnotesize\href{mailto:andi.han@sydney.edu.au}{\color{muted}andi.han@sydney.edu.au}}
\end{tabular}

\vspace{10pt}

\begin{tcolorbox}[enhanced,colback=absback,frame hidden,sharp corners,boxrule=0pt,
  borderline west={3pt}{0pt}{refcool},left=11pt,right=11pt,top=7pt,bottom=7pt]
{\sffamily\bfseries\footnotesize\color{navy}\lsstyle ABSTRACT}\par\vspace{3pt}
\small Neural representations can encode more features than they have dimensions, a phenomenon known as superposition. We study the dimension needed to compute Boolean gates from such representations. For a single threshold layer with a Gaussian random dictionary and uniformly random sparse Boolean inputs, we derive sharp dimension thresholds under two error criteria. A vanishing expected error count can require more dimensions than correctness of every output with high probability. Shared reads explain the gap: rare realizations can produce many errors at once. The expected-count threshold has ball geometry, while joint reliability has box geometry when a gate is evaluated on every feature tuple. Optimizing shared readout weights and biases gives explicit thresholds for conjunction, disjunction, and majority. For pairwise conjunction, the analysis also describes the transition near the threshold, in agreement with exact simulations.\par
\vspace{4pt}{\sffamily\footnotesize\color{navy}\textbf{Keywords:}}\ {\footnotesize computation in superposition, sparse feature representations, threshold circuits, Gaussian extreme values, sharp thresholds, reliability, mechanistic interpretability}
\end{tcolorbox}

\vspace{2pt}
\begin{tcolorbox}[enhanced,colback=white,colframe=rulegray,boxrule=0.5pt,arc=2pt,left=10pt,right=10pt,top=5pt,bottom=4pt]
{\sffamily\bfseries\footnotesize\color{navy}\lsstyle CONTENTS}\par
\setlength{\multicolsep}{4pt plus 1pt minus 1pt}\begin{multicols}{2}
{\hypersetup{linkcolor=ink}\makeatletter\@starttoc{toc}\makeatother}
\end{multicols}
\end{tcolorbox}
\suppressfloats[t]
\clearpage

\suppressfloats[t]
\input{intro}
\input{setting}
\input{exact}
\input{sections_ball}
\input{sections_box}
\input{sections_gates}

\input{sections_window}
\input{experiment_figure}
\input{sections_beyond}
\input{sections_experiments}
\input{discussion}
\input{ai_use}

\bibliographystyle{plainnat}
\bibliography{refs}

\clearpage
\appendix
\crefalias{section}{appendix}\crefalias{subsection}{subappendix}\crefalias{subsubsection}{subsubappendix}
\addtocontents{toc}{\protect\addvspace{5pt}\protect\noindent{\protect\sffamily\protect\small\protect\bfseries\protect\color{muted}Appendix}\protect\par}
\noindent{\color{refcool}\rule{\linewidth}{1.6pt}}\par\vspace{8pt}
\noindent{\sffamily\bfseries\LARGE\color{navy}Appendix}\par\vspace{3pt}
\noindent{\sffamily\small\color{muted}The Ball and the Box}\par\vspace{6pt}
\noindent{\color{rulegray}\rule{\linewidth}{0.5pt}}\par\vspace{4pt}
\suppressfloats[t]
\input{guide}
\input{reduction}
\input{sm_ball}
\input{sm_box}
\input{sm_gates}
\input{sm_window}
\input{extensions}
\input{sm_beyond}
\input{windowdetail}
\input{comparisons}
\input{numerics}
\input{sm_experiments}

\input{applications}

\end{document}

%% file: intro.tex
\setcounter{topnumber}{1}\section{Introduction}\label{sec:intro}

A network computes in superposition when many decisions read a vector that stores more features than it has dimensions \citep{Elhage22,HanniMVC24,AdlerShavit24}. Each feature has a direction in $\R^d$, and a threshold unit reads a weighted sum of these directions. If only $k$ of $m$ features are active, a unit probe sees interference of order $\sqrt{k/d}$. This controls one decision. What dimension is needed when a whole family of decisions shares the same stored vector? We count the width $d$ of that vector; each gate has its own readout unit.

Pairwise conjunction gives a concrete example. Store $m$ binary features in a Gaussian dictionary and read all $\binom m2$ products $x_ix_j$ with probes $e_i+e_j$. There are two choices: require the expected number of wrong outputs to vanish, or require every output to be correct with probability tending to one. The bias can be fixed at the midpoint $3/2$ or optimized. These choices give four sharp thresholds (\cref{tab:four}). When $\log k=o(\log m)$, the constants are $16$, $8$, $4$ and $2$ in units of $k\log m$. When $\log m\sim\log k$, all four thresholds equal $32k\log k$ to first order.

\begin{table*}[t]
\caption{Sharp first-order thresholds for pairwise conjunction, with $\lambda_1=\sqrt{2\log k}$ and $\lambda_0=\sqrt{2\log(m-k)}$. Second lines give the limit when $\log k=o(\log m)$. All entries are $32k\log k$ to first order when $\log m\sim\log k$.}
\label{tab:four}
\centering
\small
\renewcommand{\arraystretch}{1.25}
\begin{tabular}{@{}lll@{}}
\toprule
& Midpoint bias & Optimized bias\\
\midrule
Count errors, $\Exp X\to0$ (ball) & $8k(\lambda_1^2+\lambda_0^2)$ & $2k\bigl(\sqrt2\,\lambda_1+(\lambda_1^2+\lambda_0^2)^{1/2}\bigr)^2$\\
& $\sim16\,k\log m$ & $\sim4\,k\log m$\\
\addlinespace
Forbid errors, $\Prob(X=0)\to1$ (box) & $4k(\lambda_1+\lambda_0)^2$ & $k(3\lambda_1+\lambda_0)^2$\\
& $\sim8\,k\log m$ & $\sim2\,k\log m$\\
\bottomrule
\end{tabular}
\end{table*}
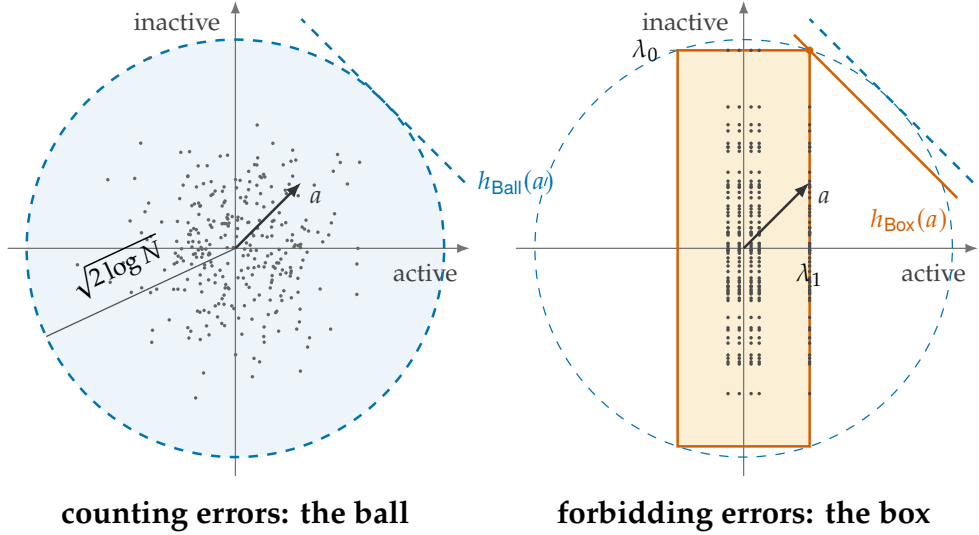
\begin{figure*}[t]
\centering
\scalebox{1.18}{\input{hero}}
\caption{Two geometries for the one-active class of pairwise conjunction. Left: independent vectors have the extreme projections of a disc of radius $\sqrt{2\log N}$, where $N=k(m-k)$. Right: shared active and inactive reads form a grid in $[-\lambda_1,\lambda_1]\times[-\lambda_0,\lambda_0]$. A gate fails beyond the line normal to $a=(1,1)$ at distance $s_z/(\sigd\norm a)$. The expected count must clear the disc; reliability must clear the rectangle. Here $\lambda_0=3\lambda_1$; points are illustrative.}
\label{fig:hero}
\end{figure*}

Reliability retains dependence between gate errors. Group gates by their active-input pattern $z\in\{0,1\}^r$. A gate in this class has weights $a$ and margin $s_z$; it fails when its interference crosses the margin. Counting errors uses only the law of each gate. Its answer is unchanged if the $N_z$ interference vectors are independent Gaussian points, whose extreme projections have the support function of a ball of radius $\sqrt{2\log N_z}$. Reliability uses the actual joint law. In the family of all gates, every interference vector is assembled from the same active and inactive feature reads. Their extremes form a box (\cref{fig:hero}), with half-sides $\lambda_1=\sqrt{2\log k}$ and $\lambda_0=\sqrt{2\log(m-k)}$ according to the input state. In either case, a class costs
\begin{equation}\label{eq:rule}
d\approx k\,\frac{h(a)^2}{s_z^2},
\end{equation}
where $h(a)=\norm a\sqrt{2\log N_z}$ for the ball and $h(a)=\sum_l\abs{a_l}\lambda_{z_l}$ for the box. The box lies inside the ball to first order. The Cauchy--Schwarz inequality between their support functions explains why the expected count can require more dimensions than reliability. Optimizing the bias produces a separate reduction by assigning more margin to classes with larger extremes.

\paragraph{Results}
The exact read law gives a common probabilistic reduction (\cref{thm:exact}). Expected-count thresholds depend on class sizes; reliability depends on Gaussian suprema (\cref{thm:ball,thm:sup}). The optimized count is a second-order cone program with a converse over all gate-local readouts (\cref{thm:weighted}). For the universal family, the box program has a converse over common representations, even if chosen after observing the reads (\cref{thm:box}). Symmetric threshold gates have closed forms (\cref{thm:thr}). Pairwise conjunction has explicit Gumbel--Poisson windows, tested by exact simulation (\cref{thm:window}).

\paragraph{Previous work}
Superposition explains polysemantic neurons in models of sparse features \citep{Olah20,Elhage22}. \citet{HanniMVC24} and \citet{AdlerShavit24} study neural computation with superposed hidden activations. Here we ask which Boolean functions a single threshold layer can read from one stored code. \citet{Newgas25} studies learned universal-AND circuits; \citet{BhagatMGH26} distinguish compressed computation from computation in superposition. \citet{GKP26} distinguish linear storage from threshold accessibility, and \citet{IvanitskiyJKM26} analyze feature recovery from superposition. Our comparison with recovery uses the same dictionary, input law and error criterion.

Sharp thresholds also appear in compressed sensing and support recovery \citep{CandesTao06,Donoho06,DonohoTanner09,FletcherRanganGoyal09,Wainwright09}. Sparse superposition codes supply the Gaussian conditioning identity used here \citep{BarronJoseph11,JosephBarron14}; vector-symbolic architectures also store superposed random vectors \citep{Plate95,ThomasDasguptaRosing21}. Different error criteria change capacity constants in associative memories \citep{McEliece87}. The present geometry identifies that difference for threshold gate families. The proofs use Gaussian concentration and comparison \citep{Borell75,TIS76,Slepian62} and Gaussian extreme-value theory \citep{LeadbetterLindgrenRootzen1983,Resnick87}. The gap between the expected and typical error counts is an annealed--quenched distinction \citep{Derrida81}.

%% file: hero.tex
\begin{tikzpicture}[scale=0.74]
\begin{scope}[xshift=0cm]
\fill[wongblue!7] (0,0) circle (3.1623);
\draw[-{Latex[length=4pt]},black!55] (-3.45,0) -- (3.55,0) node[below=1pt,font=\scriptsize,text=black!75,anchor=north east] {active};
\draw[-{Latex[length=4pt]},black!55] (0,-3.45) -- (0,3.75) node[left=1pt,font=\scriptsize,text=black!75,anchor=north east] {inactive};
\draw[wongblue,thick,dashed] (0,0) circle (3.1623);
\fill[black!60] (-0.558,0.686) circle (0.75pt);
\fill[black!60] (-0.200,0.642) circle (0.75pt);
\fill[black!60] (1.323,0.098) circle (0.75pt);
\fill[black!60] (-0.704,-1.359) circle (0.75pt);
\fill[black!60] (0.627,-0.953) circle (0.75pt);
\fill[black!60] (0.560,0.409) circle (0.75pt);
\fill[black!60] (-0.589,-0.134) circle (0.75pt);
\fill[black!60] (-0.676,0.267) circle (0.75pt);
\fill[black!60] (1.108,0.262) circle (0.75pt);
\fill[black!60] (0.915,1.337) circle (0.75pt);
\fill[black!60] (-0.357,0.318) circle (0.75pt);
\fill[black!60] (-0.564,1.249) circle (0.75pt);
\fill[black!60] (0.199,0.889) circle (0.75pt);
\fill[black!60] (-0.388,-0.087) circle (0.75pt);
\fill[black!60] (-1.073,-0.328) circle (0.75pt);
\fill[black!60] (0.431,0.628) circle (0.75pt);
\fill[black!60] (0.980,-0.922) circle (0.75pt);
\fill[black!60] (0.725,1.265) circle (0.75pt);
\fill[black!60] (0.990,0.061) circle (0.75pt);
\fill[black!60] (0.949,0.042) circle (0.75pt);
\fill[black!60] (0.831,-0.421) circle (0.75pt);
\fill[black!60] (-0.232,-0.332) circle (0.75pt);
\fill[black!60] (-1.595,0.415) circle (0.75pt);
\fill[black!60] (-0.164,-1.171) circle (0.75pt);
\fill[black!60] (-0.191,-0.429) circle (0.75pt);
\fill[black!60] (0.988,0.010) circle (0.75pt);
\fill[black!60] (0.435,-1.191) circle (0.75pt);
\fill[black!60] (0.258,-1.945) circle (0.75pt);
\fill[black!60] (-0.229,1.396) circle (0.75pt);
\fill[black!60] (-1.700,-0.220) circle (0.75pt);
\fill[black!60] (-0.726,-0.471) circle (0.75pt);
\fill[black!60] (0.689,-0.720) circle (0.75pt);
\fill[black!60] (0.291,1.268) circle (0.75pt);
\fill[black!60] (-0.029,0.012) circle (0.75pt);
\fill[black!60] (0.193,-0.103) circle (0.75pt);
\fill[black!60] (0.454,-0.218) circle (0.75pt);
\fill[black!60] (0.003,-1.611) circle (0.75pt);
\fill[black!60] (0.140,-0.012) circle (0.75pt);
\fill[black!60] (-1.192,-0.619) circle (0.75pt);
\fill[black!60] (0.017,0.917) circle (0.75pt);
\fill[black!60] (-0.124,0.009) circle (0.75pt);
\fill[black!60] (0.201,1.166) circle (0.75pt);
\fill[black!60] (0.167,0.082) circle (0.75pt);
\fill[black!60] (0.224,0.048) circle (0.75pt);
\fill[black!60] (-0.437,-1.175) circle (0.75pt);
\fill[black!60] (-0.590,0.930) circle (0.75pt);
\fill[black!60] (1.094,-1.496) circle (0.75pt);
\fill[black!60] (-0.706,0.212) circle (0.75pt);
\fill[black!60] (0.249,0.232) circle (0.75pt);
\fill[black!60] (-0.602,-2.266) circle (0.75pt);
\fill[black!60] (-0.305,-0.403) circle (0.75pt);
\fill[black!60] (-0.341,0.821) circle (0.75pt);
\fill[black!60] (0.815,0.971) circle (0.75pt);
\fill[black!60] (-0.408,0.432) circle (0.75pt);
\fill[black!60] (-1.582,-0.150) circle (0.75pt);
\fill[black!60] (0.734,-0.643) circle (0.75pt);
\fill[black!60] (0.545,-0.082) circle (0.75pt);
\fill[black!60] (0.282,-1.667) circle (0.75pt);
\fill[black!60] (0.868,1.487) circle (0.75pt);
\fill[black!60] (-0.721,-1.202) circle (0.75pt);
\fill[black!60] (0.756,0.881) circle (0.75pt);
\fill[black!60] (0.430,-1.336) circle (0.75pt);
\fill[black!60] (0.528,-0.599) circle (0.75pt);
\fill[black!60] (-0.560,0.231) circle (0.75pt);
\fill[black!60] (0.887,0.021) circle (0.75pt);
\fill[black!60] (-0.406,-0.478) circle (0.75pt);
\fill[black!60] (-0.098,-0.613) circle (0.75pt);
\fill[black!60] (0.336,0.304) circle (0.75pt);
\fill[black!60] (-0.608,-0.080) circle (0.75pt);
\fill[black!60] (-0.307,0.868) circle (0.75pt);
\fill[black!60] (-0.788,-0.190) circle (0.75pt);
\fill[black!60] (0.039,0.472) circle (0.75pt);
\fill[black!60] (-1.933,1.040) circle (0.75pt);
\fill[black!60] (0.547,-0.517) circle (0.75pt);
\fill[black!60] (0.137,-0.869) circle (0.75pt);
\fill[black!60] (0.436,-1.422) circle (0.75pt);
\fill[black!60] (-0.369,0.802) circle (0.75pt);
\fill[black!60] (0.217,1.009) circle (0.75pt);
\fill[black!60] (0.547,0.322) circle (0.75pt);
\fill[black!60] (0.381,0.829) circle (0.75pt);
\fill[black!60] (0.352,-0.272) circle (0.75pt);
\fill[black!60] (-0.119,1.411) circle (0.75pt);
\fill[black!60] (0.040,-0.446) circle (0.75pt);
\fill[black!60] (-0.055,-0.872) circle (0.75pt);
\fill[black!60] (0.208,-0.591) circle (0.75pt);
\fill[black!60] (1.023,-0.179) circle (0.75pt);
\fill[black!60] (-0.719,1.096) circle (0.75pt);
\fill[black!60] (-0.664,-0.240) circle (0.75pt);
\fill[black!60] (-0.466,-0.251) circle (0.75pt);
\fill[black!60] (1.856,-0.001) circle (0.75pt);
\fill[black!60] (-1.577,-0.667) circle (0.75pt);
\fill[black!60] (0.362,-0.707) circle (0.75pt);
\fill[black!60] (-0.438,0.118) circle (0.75pt);
\fill[black!60] (0.002,1.042) circle (0.75pt);
\fill[black!60] (-2.393,-0.249) circle (0.75pt);
\fill[black!60] (0.840,-0.375) circle (0.75pt);
\fill[black!60] (-1.420,-0.598) circle (0.75pt);
\fill[black!60] (0.078,0.768) circle (0.75pt);
\fill[black!60] (-0.108,0.338) circle (0.75pt);
\fill[black!60] (0.905,0.052) circle (0.75pt);
\fill[black!60] (1.070,0.427) circle (0.75pt);
\fill[black!60] (0.848,-0.749) circle (0.75pt);
\fill[black!60] (0.545,0.934) circle (0.75pt);
\fill[black!60] (0.655,-1.255) circle (0.75pt);
\fill[black!60] (-0.510,-0.134) circle (0.75pt);
\fill[black!60] (1.589,1.301) circle (0.75pt);
\fill[black!60] (0.233,-0.034) circle (0.75pt);
\fill[black!60] (-0.396,-0.072) circle (0.75pt);
\fill[black!60] (-0.871,-0.619) circle (0.75pt);
\fill[black!60] (1.048,0.292) circle (0.75pt);
\fill[black!60] (-0.308,-0.522) circle (0.75pt);
\fill[black!60] (-0.710,0.093) circle (0.75pt);
\fill[black!60] (0.017,0.178) circle (0.75pt);
\fill[black!60] (-0.684,0.207) circle (0.75pt);
\fill[black!60] (-0.097,0.382) circle (0.75pt);
\fill[black!60] (0.584,0.423) circle (0.75pt);
\fill[black!60] (-0.644,0.104) circle (0.75pt);
\fill[black!60] (1.250,0.244) circle (0.75pt);
\fill[black!60] (-0.023,-0.913) circle (0.75pt);
\fill[black!60] (-0.140,0.373) circle (0.75pt);
\fill[black!60] (0.203,0.395) circle (0.75pt);
\fill[black!60] (-1.256,-1.241) circle (0.75pt);
\fill[black!60] (-0.593,0.283) circle (0.75pt);
\fill[black!60] (-0.510,-0.660) circle (0.75pt);
\fill[black!60] (0.445,1.082) circle (0.75pt);
\fill[black!60] (-0.155,-0.600) circle (0.75pt);
\fill[black!60] (-0.533,-0.536) circle (0.75pt);
\fill[black!60] (0.170,0.684) circle (0.75pt);
\fill[black!60] (0.161,0.778) circle (0.75pt);
\fill[black!60] (-0.532,-1.381) circle (0.75pt);
\fill[black!60] (-0.401,-0.971) circle (0.75pt);
\fill[black!60] (-0.448,-0.012) circle (0.75pt);
\fill[black!60] (0.656,-0.223) circle (0.75pt);
\fill[black!60] (1.840,-0.698) circle (0.75pt);
\fill[black!60] (-0.387,0.058) circle (0.75pt);
\fill[black!60] (0.853,0.116) circle (0.75pt);
\fill[black!60] (1.483,1.168) circle (0.75pt);
\fill[black!60] (1.414,0.203) circle (0.75pt);
\fill[black!60] (0.407,0.192) circle (0.75pt);
\fill[black!60] (0.763,-0.072) circle (0.75pt);
\fill[black!60] (-0.170,0.377) circle (0.75pt);
\fill[black!60] (-0.816,-0.007) circle (0.75pt);
\fill[black!60] (-1.423,0.236) circle (0.75pt);
\fill[black!60] (0.602,-0.301) circle (0.75pt);
\fill[black!60] (1.548,1.422) circle (0.75pt);
\fill[black!60] (0.489,-0.115) circle (0.75pt);
\fill[black!60] (-0.115,-0.230) circle (0.75pt);
\fill[black!60] (0.763,0.389) circle (0.75pt);
\fill[black!60] (0.287,-0.780) circle (0.75pt);
\fill[black!60] (-0.253,-0.840) circle (0.75pt);
\fill[black!60] (0.966,0.046) circle (0.75pt);
\fill[black!60] (-0.867,-0.118) circle (0.75pt);
\fill[black!60] (0.330,1.868) circle (0.75pt);
\fill[black!60] (-0.529,0.646) circle (0.75pt);
\fill[black!60] (1.195,0.958) circle (0.75pt);
\fill[black!60] (0.949,-0.665) circle (0.75pt);
\fill[black!60] (0.256,-0.401) circle (0.75pt);
\fill[black!60] (0.772,0.830) circle (0.75pt);
\fill[black!60] (-0.539,0.639) circle (0.75pt);
\fill[black!60] (-0.507,-0.014) circle (0.75pt);
\fill[black!60] (-0.406,-0.100) circle (0.75pt);
\fill[black!60] (0.291,-0.159) circle (0.75pt);
\fill[black!60] (0.427,1.121) circle (0.75pt);
\fill[black!60] (1.086,-0.059) circle (0.75pt);
\fill[black!60] (0.365,0.423) circle (0.75pt);
\fill[black!60] (-1.281,0.179) circle (0.75pt);
\fill[black!60] (0.952,0.797) circle (0.75pt);
\fill[black!60] (-0.467,-0.297) circle (0.75pt);
\fill[black!60] (-0.171,0.931) circle (0.75pt);
\fill[black!60] (-1.525,1.059) circle (0.75pt);
\fill[black!60] (0.846,-0.343) circle (0.75pt);
\fill[black!60] (-1.306,2.880) circle (0.75pt);
\fill[black!60] (0.994,-0.293) circle (0.75pt);
\fill[black!60] (-0.196,0.865) circle (0.75pt);
\fill[black!60] (-0.201,-0.763) circle (0.75pt);
\fill[black!60] (0.713,-0.327) circle (0.75pt);
\fill[black!60] (-0.146,-0.548) circle (0.75pt);
\fill[black!60] (0.947,-0.454) circle (0.75pt);
\fill[black!60] (0.767,-0.372) circle (0.75pt);
\fill[black!60] (0.567,0.700) circle (0.75pt);
\fill[black!60] (-1.434,1.594) circle (0.75pt);
\fill[black!60] (-0.112,-0.597) circle (0.75pt);
\fill[black!60] (1.697,0.975) circle (0.75pt);
\fill[black!60] (-0.791,0.957) circle (0.75pt);
\fill[black!60] (1.012,0.322) circle (0.75pt);
\fill[black!60] (-0.114,0.510) circle (0.75pt);
\fill[black!60] (0.338,1.371) circle (0.75pt);
\fill[black!60] (1.502,0.524) circle (0.75pt);
\fill[black!60] (0.091,-0.335) circle (0.75pt);
\fill[black!60] (0.785,0.120) circle (0.75pt);
\fill[black!60] (0.683,-0.394) circle (0.75pt);
\fill[black!60] (-0.926,-0.236) circle (0.75pt);
\fill[black!60] (-1.339,0.132) circle (0.75pt);
\fill[black!60] (-1.369,0.554) circle (0.75pt);
\fill[black!60] (0.189,0.391) circle (0.75pt);
\fill[black!60] (-0.533,-0.403) circle (0.75pt);
\fill[black!60] (-0.172,-0.848) circle (0.75pt);
\fill[black!60] (0.396,1.071) circle (0.75pt);
\fill[black!60] (-0.660,0.646) circle (0.75pt);
\fill[black!60] (1.443,-1.224) circle (0.75pt);
\fill[black!60] (0.868,1.094) circle (0.75pt);
\fill[black!60] (-1.117,0.399) circle (0.75pt);
\fill[black!60] (-1.897,-0.514) circle (0.75pt);
\fill[black!60] (-0.562,-0.026) circle (0.75pt);
\fill[black!60] (0.609,-0.973) circle (0.75pt);
\fill[black!60] (-0.270,-1.039) circle (0.75pt);
\fill[black!60] (-0.326,0.412) circle (0.75pt);
\fill[black!60] (0.282,-0.293) circle (0.75pt);
\fill[black!60] (0.527,-0.784) circle (0.75pt);
\fill[black!60] (0.568,1.591) circle (0.75pt);
\fill[black!60] (-0.886,-0.213) circle (0.75pt);
\fill[black!60] (-0.495,-0.083) circle (0.75pt);
\fill[black!60] (0.121,-0.450) circle (0.75pt);
\fill[black!60] (-0.889,0.106) circle (0.75pt);
\fill[black!60] (0.251,-1.073) circle (0.75pt);
\fill[black!60] (-0.553,0.514) circle (0.75pt);
\fill[black!60] (-0.427,1.252) circle (0.75pt);
\fill[black!60] (0.865,0.206) circle (0.75pt);
\fill[black!60] (-0.115,0.444) circle (0.75pt);
\fill[black!60] (0.325,0.955) circle (0.75pt);
\fill[black!60] (0.703,0.356) circle (0.75pt);
\fill[black!60] (0.326,0.678) circle (0.75pt);
\fill[black!60] (-0.380,0.352) circle (0.75pt);
\fill[black!60] (-0.936,-0.621) circle (0.75pt);
\fill[black!60] (-0.483,0.980) circle (0.75pt);
\fill[black!60] (0.071,1.030) circle (0.75pt);
\fill[black!60] (-1.314,0.686) circle (0.75pt);
\fill[black!60] (-0.658,0.020) circle (0.75pt);
\fill[black!60] (0.627,-0.220) circle (0.75pt);
\fill[black!60] (-0.878,0.408) circle (0.75pt);
\fill[black!60] (0.192,-0.299) circle (0.75pt);
\fill[black!60] (0.553,0.188) circle (0.75pt);
\fill[black!60] (1.123,-1.355) circle (0.75pt);
\fill[black!60] (-0.827,1.455) circle (0.75pt);
\fill[black!60] (-0.694,1.465) circle (0.75pt);
\fill[black!60] (0.098,-1.301) circle (0.75pt);
\fill[black!60] (-0.537,-0.968) circle (0.75pt);
\fill[black!60] (-0.844,0.754) circle (0.75pt);
\fill[black!60] (0.970,-1.114) circle (0.75pt);
\fill[black!60] (-1.267,-1.423) circle (0.75pt);
\fill[black!60] (0.367,-1.720) circle (0.75pt);
\fill[black!60] (-0.355,-0.266) circle (0.75pt);
\fill[black!60] (0.277,0.191) circle (0.75pt);
\fill[black!60] (0.142,0.713) circle (0.75pt);
\fill[black!60] (-0.048,0.481) circle (0.75pt);
\fill[black!60] (-0.520,0.014) circle (0.75pt);
\fill[black!60] (0.557,-0.413) circle (0.75pt);
\fill[black!60] (0.556,-0.482) circle (0.75pt);
\fill[black!60] (0.166,-0.545) circle (0.75pt);
\fill[black!60] (-0.574,-0.951) circle (0.75pt);
\fill[black!60] (-1.002,-1.084) circle (0.75pt);
\fill[black!60] (-0.621,0.673) circle (0.75pt);
\fill[black!60] (1.882,1.304) circle (0.75pt);
\fill[black!60] (0.499,-1.003) circle (0.75pt);
\fill[black!60] (-0.172,0.692) circle (0.75pt);
\fill[black!60] (0.665,0.539) circle (0.75pt);
\fill[black!60] (0.574,-0.275) circle (0.75pt);
\fill[black!60] (-0.821,-0.146) circle (0.75pt);
\fill[black!60] (-0.768,-0.631) circle (0.75pt);
\fill[black!60] (0.663,-0.017) circle (0.75pt);
\fill[black!60] (0.317,-0.031) circle (0.75pt);
\fill[black!60] (-1.285,-0.671) circle (0.75pt);
\fill[black!60] (-0.128,0.467) circle (0.75pt);
\fill[black!60] (-0.155,0.199) circle (0.75pt);
\fill[black!60] (-0.909,-0.074) circle (0.75pt);
\fill[black!60] (0.049,-1.057) circle (0.75pt);
\fill[black!60] (0.033,0.300) circle (0.75pt);
\fill[black!60] (0.985,-1.860) circle (0.75pt);
\fill[black!60] (-0.371,-0.208) circle (0.75pt);
\fill[black!60] (-0.474,-0.940) circle (0.75pt);
\fill[black!60] (-0.482,-1.888) circle (0.75pt);
\fill[black!60] (0.843,1.319) circle (0.75pt);
\fill[black!60] (-0.157,0.717) circle (0.75pt);
\fill[black!60] (1.284,-0.616) circle (0.75pt);
\fill[black!60] (0.481,0.172) circle (0.75pt);
\fill[black!60] (-0.716,0.388) circle (0.75pt);
\fill[black!60] (-0.207,-1.138) circle (0.75pt);
\fill[black!60] (1.454,0.386) circle (0.75pt);
\fill[black!60] (-0.399,-1.310) circle (0.75pt);
\fill[black!60] (0.435,-0.282) circle (0.75pt);
\fill[black!60] (-0.219,0.162) circle (0.75pt);
\fill[black!60] (0.834,-0.290) circle (0.75pt);
\fill[black!60] (1.841,-0.237) circle (0.75pt);
\fill[black!60] (-0.172,-0.709) circle (0.75pt);
\fill[black!60] (0.096,0.451) circle (0.75pt);
\fill[black!60] (0.689,-0.163) circle (0.75pt);
\fill[black!60] (-0.367,1.011) circle (0.75pt);
\fill[black!60] (-0.305,-0.872) circle (0.75pt);
\fill[black!60] (0.704,-0.790) circle (0.75pt);
\fill[black!60] (-0.675,-1.256) circle (0.75pt);
\fill[black!60] (0.929,-1.223) circle (0.75pt);
\fill[black!60] (-0.551,0.094) circle (0.75pt);
\fill[black!60] (-0.202,0.373) circle (0.75pt);
\fill[black!60] (-0.635,-0.091) circle (0.75pt);
\fill[black!60] (0.374,1.088) circle (0.75pt);
\fill[black!60] (0.032,0.427) circle (0.75pt);
\fill[black!60] (-0.723,1.088) circle (0.75pt);
\fill[black!60] (-0.631,0.300) circle (0.75pt);
\fill[black!60] (0.118,-0.112) circle (0.75pt);
\fill[black!60] (0.325,1.263) circle (0.75pt);
\draw[wongblue,thick,dashed] (3.474,0.999) -- (0.999,3.474);
\node[font=\scriptsize,right,text=wongblue] at (3.474,0.999) {$h_{\mathsf{Ball}}(a)$};
\draw[-{Latex[length=5pt]},thick,black!80] (0,0) -- (1.0,1.0) node[below right=-2pt,font=\scriptsize] {$a$};
\draw[black!70,thin] (0,0) -- (-2.8660,-1.3364) node[pos=0.55,above=1pt,sloped,font=\scriptsize,text=black] {$\sqrt{2\log N}$};
\node[font=\small\bfseries] at (0,-4.05) {counting errors: the ball};
\end{scope}
\begin{scope}[xshift=7.7cm]
\fill[wongorange!16] (-1.0,-3.0) rectangle (1.0,3.0);
\draw[-{Latex[length=4pt]},black!55] (-3.45,0) -- (3.55,0) node[below=1pt,font=\scriptsize,text=black!75,anchor=north east] {active};
\draw[-{Latex[length=4pt]},black!55] (0,-3.45) -- (0,3.75) node[left=1pt,font=\scriptsize,text=black!75,anchor=north east] {inactive};
\draw[wongblue,dashed] (0,0) circle (3.1623);
\draw[wongverm,thick] (-1.0,-3.0) rectangle (1.0,3.0);
\fill[black!70] (1.000,1.779) circle (0.75pt);
\fill[black!70] (1.000,0.964) circle (0.75pt);
\fill[black!70] (1.000,-0.622) circle (0.75pt);
\fill[black!70] (1.000,0.848) circle (0.75pt);
\fill[black!70] (1.000,1.873) circle (0.75pt);
\fill[black!70] (1.000,-0.673) circle (0.75pt);
\fill[black!70] (1.000,-0.576) circle (0.75pt);
\fill[black!70] (1.000,-0.036) circle (0.75pt);
\fill[black!70] (1.000,-1.436) circle (0.75pt);
\fill[black!70] (1.000,1.591) circle (0.75pt);
\fill[black!70] (1.000,-1.047) circle (0.75pt);
\fill[black!70] (1.000,-2.201) circle (0.75pt);
\fill[black!70] (1.000,0.248) circle (0.75pt);
\fill[black!70] (1.000,0.310) circle (0.75pt);
\fill[black!70] (1.000,-0.515) circle (0.75pt);
\fill[black!70] (1.000,0.000) circle (0.75pt);
\fill[black!70] (1.000,-0.028) circle (0.75pt);
\fill[black!70] (1.000,0.074) circle (0.75pt);
\fill[black!70] (1.000,-1.204) circle (0.75pt);
\fill[black!70] (1.000,-0.448) circle (0.75pt);
\fill[black!70] (1.000,-0.790) circle (0.75pt);
\fill[black!70] (1.000,-0.690) circle (0.75pt);
\fill[black!70] (1.000,0.224) circle (0.75pt);
\fill[black!70] (1.000,0.545) circle (0.75pt);
\fill[black!70] (1.000,-0.203) circle (0.75pt);
\fill[black!70] (1.000,2.142) circle (0.75pt);
\fill[black!70] (1.000,-1.665) circle (0.75pt);
\fill[black!70] (1.000,-0.755) circle (0.75pt);
\fill[black!70] (1.000,-0.032) circle (0.75pt);
\fill[black!70] (1.000,1.018) circle (0.75pt);
\fill[black!70] (1.000,0.072) circle (0.75pt);
\fill[black!70] (1.000,-0.582) circle (0.75pt);
\fill[black!70] (1.000,-0.351) circle (0.75pt);
\fill[black!70] (1.000,1.153) circle (0.75pt);
\fill[black!70] (1.000,0.766) circle (0.75pt);
\fill[black!70] (1.000,-1.616) circle (0.75pt);
\fill[black!70] (1.000,-0.145) circle (0.75pt);
\fill[black!70] (1.000,1.536) circle (0.75pt);
\fill[black!70] (1.000,-0.646) circle (0.75pt);
\fill[black!70] (1.000,0.238) circle (0.75pt);
\fill[black!70] (1.000,-0.267) circle (0.75pt);
\fill[black!70] (1.000,-1.752) circle (0.75pt);
\fill[black!70] (1.000,0.313) circle (0.75pt);
\fill[black!70] (1.000,-1.719) circle (0.75pt);
\fill[black!70] (1.000,0.395) circle (0.75pt);
\fill[black!70] (1.000,-0.073) circle (0.75pt);
\fill[black!70] (1.000,3.000) circle (0.75pt);
\fill[black!70] (1.000,0.924) circle (0.75pt);
\fill[black!70] (1.000,0.642) circle (0.75pt);
\fill[black!70] (1.000,-1.237) circle (0.75pt);
\fill[black!70] (1.000,0.932) circle (0.75pt);
\fill[black!70] (1.000,-0.570) circle (0.75pt);
\fill[black!70] (1.000,-1.355) circle (0.75pt);
\fill[black!70] (1.000,0.494) circle (0.75pt);
\fill[black!70] (1.000,0.598) circle (0.75pt);
\fill[black!70] (1.000,0.035) circle (0.75pt);
\fill[black!70] (1.000,1.471) circle (0.75pt);
\fill[black!70] (1.000,0.201) circle (0.75pt);
\fill[black!70] (1.000,1.531) circle (0.75pt);
\fill[black!70] (1.000,-0.487) circle (0.75pt);
\fill[black!70] (0.113,1.779) circle (0.75pt);
\fill[black!70] (0.113,0.964) circle (0.75pt);
\fill[black!70] (0.113,-0.622) circle (0.75pt);
\fill[black!70] (0.113,0.848) circle (0.75pt);
\fill[black!70] (0.113,1.873) circle (0.75pt);
\fill[black!70] (0.113,-0.673) circle (0.75pt);
\fill[black!70] (0.113,-0.576) circle (0.75pt);
\fill[black!70] (0.113,-0.036) circle (0.75pt);
\fill[black!70] (0.113,-1.436) circle (0.75pt);
\fill[black!70] (0.113,1.591) circle (0.75pt);
\fill[black!70] (0.113,-1.047) circle (0.75pt);
\fill[black!70] (0.113,-2.201) circle (0.75pt);
\fill[black!70] (0.113,0.248) circle (0.75pt);
\fill[black!70] (0.113,0.310) circle (0.75pt);
\fill[black!70] (0.113,-0.515) circle (0.75pt);
\fill[black!70] (0.113,0.000) circle (0.75pt);
\fill[black!70] (0.113,-0.028) circle (0.75pt);
\fill[black!70] (0.113,0.074) circle (0.75pt);
\fill[black!70] (0.113,-1.204) circle (0.75pt);
\fill[black!70] (0.113,-0.448) circle (0.75pt);
\fill[black!70] (0.113,-0.790) circle (0.75pt);
\fill[black!70] (0.113,-0.690) circle (0.75pt);
\fill[black!70] (0.113,0.224) circle (0.75pt);
\fill[black!70] (0.113,0.545) circle (0.75pt);
\fill[black!70] (0.113,-0.203) circle (0.75pt);
\fill[black!70] (0.113,2.142) circle (0.75pt);
\fill[black!70] (0.113,-1.665) circle (0.75pt);
\fill[black!70] (0.113,-0.755) circle (0.75pt);
\fill[black!70] (0.113,-0.032) circle (0.75pt);
\fill[black!70] (0.113,1.018) circle (0.75pt);
\fill[black!70] (0.113,0.072) circle (0.75pt);
\fill[black!70] (0.113,-0.582) circle (0.75pt);
\fill[black!70] (0.113,-0.351) circle (0.75pt);
\fill[black!70] (0.113,1.153) circle (0.75pt);
\fill[black!70] (0.113,0.766) circle (0.75pt);
\fill[black!70] (0.113,-1.616) circle (0.75pt);
\fill[black!70] (0.113,-0.145) circle (0.75pt);
\fill[black!70] (0.113,1.536) circle (0.75pt);
\fill[black!70] (0.113,-0.646) circle (0.75pt);
\fill[black!70] (0.113,0.238) circle (0.75pt);
\fill[black!70] (0.113,-0.267) circle (0.75pt);
\fill[black!70] (0.113,-1.752) circle (0.75pt);
\fill[black!70] (0.113,0.313) circle (0.75pt);
\fill[black!70] (0.113,-1.719) circle (0.75pt);
\fill[black!70] (0.113,0.395) circle (0.75pt);
\fill[black!70] (0.113,-0.073) circle (0.75pt);
\fill[black!70] (0.113,3.000) circle (0.75pt);
\fill[black!70] (0.113,0.924) circle (0.75pt);
\fill[black!70] (0.113,0.642) circle (0.75pt);
\fill[black!70] (0.113,-1.237) circle (0.75pt);
\fill[black!70] (0.113,0.932) circle (0.75pt);
\fill[black!70] (0.113,-0.570) circle (0.75pt);
\fill[black!70] (0.113,-1.355) circle (0.75pt);
\fill[black!70] (0.113,0.494) circle (0.75pt);
\fill[black!70] (0.113,0.598) circle (0.75pt);
\fill[black!70] (0.113,0.035) circle (0.75pt);
\fill[black!70] (0.113,1.471) circle (0.75pt);
\fill[black!70] (0.113,0.201) circle (0.75pt);
\fill[black!70] (0.113,1.531) circle (0.75pt);
\fill[black!70] (0.113,-0.487) circle (0.75pt);
\fill[black!70] (-0.066,1.779) circle (0.75pt);
\fill[black!70] (-0.066,0.964) circle (0.75pt);
\fill[black!70] (-0.066,-0.622) circle (0.75pt);
\fill[black!70] (-0.066,0.848) circle (0.75pt);
\fill[black!70] (-0.066,1.873) circle (0.75pt);
\fill[black!70] (-0.066,-0.673) circle (0.75pt);
\fill[black!70] (-0.066,-0.576) circle (0.75pt);
\fill[black!70] (-0.066,-0.036) circle (0.75pt);
\fill[black!70] (-0.066,-1.436) circle (0.75pt);
\fill[black!70] (-0.066,1.591) circle (0.75pt);
\fill[black!70] (-0.066,-1.047) circle (0.75pt);
\fill[black!70] (-0.066,-2.201) circle (0.75pt);
\fill[black!70] (-0.066,0.248) circle (0.75pt);
\fill[black!70] (-0.066,0.310) circle (0.75pt);
\fill[black!70] (-0.066,-0.515) circle (0.75pt);
\fill[black!70] (-0.066,0.000) circle (0.75pt);
\fill[black!70] (-0.066,-0.028) circle (0.75pt);
\fill[black!70] (-0.066,0.074) circle (0.75pt);
\fill[black!70] (-0.066,-1.204) circle (0.75pt);
\fill[black!70] (-0.066,-0.448) circle (0.75pt);
\fill[black!70] (-0.066,-0.790) circle (0.75pt);
\fill[black!70] (-0.066,-0.690) circle (0.75pt);
\fill[black!70] (-0.066,0.224) circle (0.75pt);
\fill[black!70] (-0.066,0.545) circle (0.75pt);
\fill[black!70] (-0.066,-0.203) circle (0.75pt);
\fill[black!70] (-0.066,2.142) circle (0.75pt);
\fill[black!70] (-0.066,-1.665) circle (0.75pt);
\fill[black!70] (-0.066,-0.755) circle (0.75pt);
\fill[black!70] (-0.066,-0.032) circle (0.75pt);
\fill[black!70] (-0.066,1.018) circle (0.75pt);
\fill[black!70] (-0.066,0.072) circle (0.75pt);
\fill[black!70] (-0.066,-0.582) circle (0.75pt);
\fill[black!70] (-0.066,-0.351) circle (0.75pt);
\fill[black!70] (-0.066,1.153) circle (0.75pt);
\fill[black!70] (-0.066,0.766) circle (0.75pt);
\fill[black!70] (-0.066,-1.616) circle (0.75pt);
\fill[black!70] (-0.066,-0.145) circle (0.75pt);
\fill[black!70] (-0.066,1.536) circle (0.75pt);
\fill[black!70] (-0.066,-0.646) circle (0.75pt);
\fill[black!70] (-0.066,0.238) circle (0.75pt);
\fill[black!70] (-0.066,-0.267) circle (0.75pt);
\fill[black!70] (-0.066,-1.752) circle (0.75pt);
\fill[black!70] (-0.066,0.313) circle (0.75pt);
\fill[black!70] (-0.066,-1.719) circle (0.75pt);
\fill[black!70] (-0.066,0.395) circle (0.75pt);
\fill[black!70] (-0.066,-0.073) circle (0.75pt);
\fill[black!70] (-0.066,3.000) circle (0.75pt);
\fill[black!70] (-0.066,0.924) circle (0.75pt);
\fill[black!70] (-0.066,0.642) circle (0.75pt);
\fill[black!70] (-0.066,-1.237) circle (0.75pt);
\fill[black!70] (-0.066,0.932) circle (0.75pt);
\fill[black!70] (-0.066,-0.570) circle (0.75pt);
\fill[black!70] (-0.066,-1.355) circle (0.75pt);
\fill[black!70] (-0.066,0.494) circle (0.75pt);
\fill[black!70] (-0.066,0.598) circle (0.75pt);
\fill[black!70] (-0.066,0.035) circle (0.75pt);
\fill[black!70] (-0.066,1.471) circle (0.75pt);
\fill[black!70] (-0.066,0.201) circle (0.75pt);
\fill[black!70] (-0.066,1.531) circle (0.75pt);
\fill[black!70] (-0.066,-0.487) circle (0.75pt);
\fill[black!70] (0.236,1.779) circle (0.75pt);
\fill[black!70] (0.236,0.964) circle (0.75pt);
\fill[black!70] (0.236,-0.622) circle (0.75pt);
\fill[black!70] (0.236,0.848) circle (0.75pt);
\fill[black!70] (0.236,1.873) circle (0.75pt);
\fill[black!70] (0.236,-0.673) circle (0.75pt);
\fill[black!70] (0.236,-0.576) circle (0.75pt);
\fill[black!70] (0.236,-0.036) circle (0.75pt);
\fill[black!70] (0.236,-1.436) circle (0.75pt);
\fill[black!70] (0.236,1.591) circle (0.75pt);
\fill[black!70] (0.236,-1.047) circle (0.75pt);
\fill[black!70] (0.236,-2.201) circle (0.75pt);
\fill[black!70] (0.236,0.248) circle (0.75pt);
\fill[black!70] (0.236,0.310) circle (0.75pt);
\fill[black!70] (0.236,-0.515) circle (0.75pt);
\fill[black!70] (0.236,0.000) circle (0.75pt);
\fill[black!70] (0.236,-0.028) circle (0.75pt);
\fill[black!70] (0.236,0.074) circle (0.75pt);
\fill[black!70] (0.236,-1.204) circle (0.75pt);
\fill[black!70] (0.236,-0.448) circle (0.75pt);
\fill[black!70] (0.236,-0.790) circle (0.75pt);
\fill[black!70] (0.236,-0.690) circle (0.75pt);
\fill[black!70] (0.236,0.224) circle (0.75pt);
\fill[black!70] (0.236,0.545) circle (0.75pt);
\fill[black!70] (0.236,-0.203) circle (0.75pt);
\fill[black!70] (0.236,2.142) circle (0.75pt);
\fill[black!70] (0.236,-1.665) circle (0.75pt);
\fill[black!70] (0.236,-0.755) circle (0.75pt);
\fill[black!70] (0.236,-0.032) circle (0.75pt);
\fill[black!70] (0.236,1.018) circle (0.75pt);
\fill[black!70] (0.236,0.072) circle (0.75pt);
\fill[black!70] (0.236,-0.582) circle (0.75pt);
\fill[black!70] (0.236,-0.351) circle (0.75pt);
\fill[black!70] (0.236,1.153) circle (0.75pt);
\fill[black!70] (0.236,0.766) circle (0.75pt);
\fill[black!70] (0.236,-1.616) circle (0.75pt);
\fill[black!70] (0.236,-0.145) circle (0.75pt);
\fill[black!70] (0.236,1.536) circle (0.75pt);
\fill[black!70] (0.236,-0.646) circle (0.75pt);
\fill[black!70] (0.236,0.238) circle (0.75pt);
\fill[black!70] (0.236,-0.267) circle (0.75pt);
\fill[black!70] (0.236,-1.752) circle (0.75pt);
\fill[black!70] (0.236,0.313) circle (0.75pt);
\fill[black!70] (0.236,-1.719) circle (0.75pt);
\fill[black!70] (0.236,0.395) circle (0.75pt);
\fill[black!70] (0.236,-0.073) circle (0.75pt);
\fill[black!70] (0.236,3.000) circle (0.75pt);
\fill[black!70] (0.236,0.924) circle (0.75pt);
\fill[black!70] (0.236,0.642) circle (0.75pt);
\fill[black!70] (0.236,-1.237) circle (0.75pt);
\fill[black!70] (0.236,0.932) circle (0.75pt);
\fill[black!70] (0.236,-0.570) circle (0.75pt);
\fill[black!70] (0.236,-1.355) circle (0.75pt);
\fill[black!70] (0.236,0.494) circle (0.75pt);
\fill[black!70] (0.236,0.598) circle (0.75pt);
\fill[black!70] (0.236,0.035) circle (0.75pt);
\fill[black!70] (0.236,1.471) circle (0.75pt);
\fill[black!70] (0.236,0.201) circle (0.75pt);
\fill[black!70] (0.236,1.531) circle (0.75pt);
\fill[black!70] (0.236,-0.487) circle (0.75pt);
\fill[black!70] (-0.237,1.779) circle (0.75pt);
\fill[black!70] (-0.237,0.964) circle (0.75pt);
\fill[black!70] (-0.237,-0.622) circle (0.75pt);
\fill[black!70] (-0.237,0.848) circle (0.75pt);
\fill[black!70] (-0.237,1.873) circle (0.75pt);
\fill[black!70] (-0.237,-0.673) circle (0.75pt);
\fill[black!70] (-0.237,-0.576) circle (0.75pt);
\fill[black!70] (-0.237,-0.036) circle (0.75pt);
\fill[black!70] (-0.237,-1.436) circle (0.75pt);
\fill[black!70] (-0.237,1.591) circle (0.75pt);
\fill[black!70] (-0.237,-1.047) circle (0.75pt);
\fill[black!70] (-0.237,-2.201) circle (0.75pt);
\fill[black!70] (-0.237,0.248) circle (0.75pt);
\fill[black!70] (-0.237,0.310) circle (0.75pt);
\fill[black!70] (-0.237,-0.515) circle (0.75pt);
\fill[black!70] (-0.237,0.000) circle (0.75pt);
\fill[black!70] (-0.237,-0.028) circle (0.75pt);
\fill[black!70] (-0.237,0.074) circle (0.75pt);
\fill[black!70] (-0.237,-1.204) circle (0.75pt);
\fill[black!70] (-0.237,-0.448) circle (0.75pt);
\fill[black!70] (-0.237,-0.790) circle (0.75pt);
\fill[black!70] (-0.237,-0.690) circle (0.75pt);
\fill[black!70] (-0.237,0.224) circle (0.75pt);
\fill[black!70] (-0.237,0.545) circle (0.75pt);
\fill[black!70] (-0.237,-0.203) circle (0.75pt);
\fill[black!70] (-0.237,2.142) circle (0.75pt);
\fill[black!70] (-0.237,-1.665) circle (0.75pt);
\fill[black!70] (-0.237,-0.755) circle (0.75pt);
\fill[black!70] (-0.237,-0.032) circle (0.75pt);
\fill[black!70] (-0.237,1.018) circle (0.75pt);
\fill[black!70] (-0.237,0.072) circle (0.75pt);
\fill[black!70] (-0.237,-0.582) circle (0.75pt);
\fill[black!70] (-0.237,-0.351) circle (0.75pt);
\fill[black!70] (-0.237,1.153) circle (0.75pt);
\fill[black!70] (-0.237,0.766) circle (0.75pt);
\fill[black!70] (-0.237,-1.616) circle (0.75pt);
\fill[black!70] (-0.237,-0.145) circle (0.75pt);
\fill[black!70] (-0.237,1.536) circle (0.75pt);
\fill[black!70] (-0.237,-0.646) circle (0.75pt);
\fill[black!70] (-0.237,0.238) circle (0.75pt);
\fill[black!70] (-0.237,-0.267) circle (0.75pt);
\fill[black!70] (-0.237,-1.752) circle (0.75pt);
\fill[black!70] (-0.237,0.313) circle (0.75pt);
\fill[black!70] (-0.237,-1.719) circle (0.75pt);
\fill[black!70] (-0.237,0.395) circle (0.75pt);
\fill[black!70] (-0.237,-0.073) circle (0.75pt);
\fill[black!70] (-0.237,3.000) circle (0.75pt);
\fill[black!70] (-0.237,0.924) circle (0.75pt);
\fill[black!70] (-0.237,0.642) circle (0.75pt);
\fill[black!70] (-0.237,-1.237) circle (0.75pt);
\fill[black!70] (-0.237,0.932) circle (0.75pt);
\fill[black!70] (-0.237,-0.570) circle (0.75pt);
\fill[black!70] (-0.237,-1.355) circle (0.75pt);
\fill[black!70] (-0.237,0.494) circle (0.75pt);
\fill[black!70] (-0.237,0.598) circle (0.75pt);
\fill[black!70] (-0.237,0.035) circle (0.75pt);
\fill[black!70] (-0.237,1.471) circle (0.75pt);
\fill[black!70] (-0.237,0.201) circle (0.75pt);
\fill[black!70] (-0.237,1.531) circle (0.75pt);
\fill[black!70] (-0.237,-0.487) circle (0.75pt);
\draw[wongverm,thick] (3.237,0.763) -- (0.763,3.237);
\node[font=\scriptsize,below left=-1pt,text=wongverm] at (3.237,0.763) {$h_{\mathsf{Box}}(a)$};
\draw[wongblue,thick,dashed] (3.474,0.999) -- (0.999,3.474);
\node[font=\scriptsize,above] at (3.474,0.999) {};
\fill[wongverm] (1.0,3.0) circle (1.6pt);
\draw[-{Latex[length=5pt]},thick,black!80] (0,0) -- (1.0,1.0) node[below right=-2pt,font=\scriptsize] {$a$};
\node[font=\scriptsize,below] at (1.0,-0.05) {$\lambda_1$};
\draw[black!55] (-1.12,3.0) -- (-1.0,3.0); \node[font=\scriptsize,left] at (-1.1,3.0) {$\lambda_0$};
\node[font=\small\bfseries] at (0,-4.05) {forbidding errors: the box};
\end{scope}
\end{tikzpicture}

%% file: setting.tex
\setcounter{topnumber}{3}\section{Model and criteria}\label{sec:setting}
Let $E=[e_1|\cdots|e_m]\in\R^{d\times m}$ have independent columns $e_i\sim N(0,I_d/d)$. Independently choose $S$ uniformly among the $k$-subsets of $[m]$. Write $x=\one_S$, $n=m-k$ and
\begin{equation}\label{eq:residual}
u=Ex=\sum_{i\in S}e_i,\qquad R=E^\top u.
\end{equation}
The feature read is $R_i=\ip{e_i}{u}$. Throughout, $k\to\infty$ and $m/k\to\infty$; other restrictions are stated where needed.

\paragraph{Gates and classes}
A gate is an ordered tuple $I=(I_1,\dots,I_r)$ of distinct features and a nonconstant threshold function $g:\{0,1\}^r\to\{0,1\}$. Its target is $g(x_I)$. A family $\cF$ has common arity $r$ and function $g$. Finitely many types can be treated separately. The class of a gate is its input pattern $z=x_I$, and $\cF_z(S)$ contains the gates of that class. With falling factorials $(a)_j$ and $\abs z=\sum_lz_l$,
\begin{equation}\label{eq:counts}
N_z=\Exp\abs{\cF_z(S)}=\abs{\cF}\,\frac{(k)_{\abs z}(n)_{r-\abs z}}{(m)_r}.
\end{equation}
The universal family $\cU_r(g)$ contains all $(m)_r$ tuples. Its class sizes are deterministic: $N_z=(k)_{\abs z}(n)_{r-\abs z}$. For symmetric gates, unordered subsets change these counts only by constant factors and leave first-order thresholds unchanged.

\paragraph{Readouts and margins}
Gate $\gamma$ outputs $\one\{\ip{w_\gamma}{u}>\theta_\gamma\}$. A gate-local rule chooses $(w_\gamma,\theta_\gamma)$ from its own columns $E_I$ and independent randomness, before drawing the input. Rules may differ between gates. They can use their own column lengths and angles, but cannot use the support or other columns. A canonical rule uses $w_\gamma=\sum_la_le_{I_l}$ and a common bias $\theta$. The pair $(a,\theta)$ represents $g$ when
\begin{equation}\label{eq:margin}
s_z(a,\theta)=(2g(z)-1)\bigl(\ip az-\theta\bigr)>0\quad\text{for all }z.
\end{equation}
Write $\Rep(g)$ for these representations and $\varsigma_z=(-1)^{g(z)}$ for the error direction. Pairwise conjunction with $a=(1,1)$ has clean scores $0,1,2$ and margins $\theta,\theta-1,2-\theta$.

\paragraph{Error criteria}
Let $X$ count wrong outputs on a draw of $(E,S)$. We study
\begin{equation}\label{eq:criteria}
\Exp X\to0\qquad\text{and}\qquad\Prob(X=0)\to1.
\end{equation}
The first implies the second by Markov's inequality. The converse can fail: one rare draw can produce many errors. A sharp threshold $D(k,m)$ has a multiplicative $1\pm\varepsilon$ transition for every fixed $\varepsilon\in(0,1)$. Below it, failure means $\Exp X\to\infty$ for the count and $\Prob(X=0)\to0$ for reliability.

\paragraph{Levels and prices}
Set
\begin{equation}\label{eq:levels}
\lambda_1=\sqrt{2\log k},\quad\lambda_0=\sqrt{2\log n},\quad\beta=\lambda_0/\lambda_1.
\end{equation}
For class $z$, let $\lambda_z=(\lambda_{z_1},\dots,\lambda_{z_r})$. Eventually $\beta\ge1$, and for the universal family $2\log N_z=\norm{\lambda_z}^2(1+o(1))$. Finite-size calculations use $\lhat_1=\Phibar^{-1}(1/k)$ and $\lhat_0=\Phibar^{-1}(1/n)$; $\lhat_p=\lambda_p(1+o(1))$.
\begin{definition}[Ball, box and prices]\label{def:prices}
For a class $z$ let $\rho_z=\sqrt{2\log N_z}$ and
\[
\begin{aligned}\Ballz z&=\{v\in\R^r:\norm v\le\rho_z\},\\\Boxz z&=\prod_{l=1}^r[-\lambda_{z_l},\lambda_{z_l}],\end{aligned}
\]
with support functions $h_{\Ballz z}(a)=\rho_z\norm a$ and $h_{\Boxz z}(a)=\sum_l\abs{a_l}\lambda_{z_l}$. For $(a,\theta)\in\Rep(g)$ define the \emph{ball price} and the \emph{box price}
\begin{equation}\label{eq:prices}
\begin{aligned}\Dball(a,\theta)&=k\max_z\frac{h_{\Ballz z}(a)^2}{s_z(a,\theta)^2},\\\Dbox(a,\theta)&=k\max_z\frac{h_{\Boxz z}(a)^2}{s_z(a,\theta)^2},\end{aligned}
\end{equation}
and the optimal values $\Vball=\inf_{\Rep(g)}\Dball/k$ and $\Vbox=\inf_{\Rep(g)}\Dbox/k$.
\end{definition}
Both prices are invariant under positive rescaling of $(a,\theta)$. The box corners have norm $\norm{\lambda_z}$, so the box lies inside the ball to first order for the universal family.

%% file: exact.tex
\section{The exact read law}\label{sec:exact}
All canonical gates use the same read vector $R$. Conditioning on $u$ gives an exact law, the noiseless case of the lemmas for sparse superposition codes in \citet[Lemma~1]{BarronJoseph11} and \citet[Lemma~3]{JosephBarron14}. By symmetry fix $S=[k]$, and write $A_i=R_i$ for $i\le k$ and $B_j=R_{k+j}$ for $j\le n$.
\begin{theorem}[Exact conditional law]\label{thm:exact}
Let $H\sim\chi^2_d/d$, $Z_1,\dots,Z_k\sim N(0,1)$ and $W_1,\dots,W_n\sim N(0,1)$ be mutually independent. Write $\bar Z=k^{-1}\sum_iZ_i$ and $\sigma=\sqrt{kH/d}$. Then, jointly in $i\le k$ and $j\le n$,
\begin{equation}\label{eq:exact-law}
(A_i,B_j)\overset{d}{=}\bigl(H+\sigma(Z_i-\bar Z),\ \sigma W_j\bigr).
\end{equation}
The identity holds for all $d,k,n\ge1$.
\end{theorem}

\begin{proof}
Condition on $u=\sum_{i\le k}e_i$. Gaussian conditioning gives $e_i=u/k+\eta_i$ for $i\le k$, where $(\eta_i)_{i\le k}$ is independent of $u$ and jointly Gaussian with $\Cov(\eta_i,\eta_l)=d^{-1}(\one\{i=l\}-1/k)I_d$. Hence $A_i=\norm u^2/k+\ip{\eta_i}u$. Given $u$, the vector $(\ip{\eta_i}u)_{i\le k}$ is centered Gaussian with covariance $(\norm u^2/d)(I_k-\one\one^\top/k)$, which is the covariance of $(\norm u/\sqrt d)(Z_i-\bar Z)_{i\le k}$. The inactive columns are independent of $(u,\eta)$, so given $u$ the reads $B_j$ are independent $N(0,\norm u^2/d)$ variables, independent of the active ones. The conditional law depends on $u$ only through $\norm u$, and $H=\norm u^2/k\sim\chi^2_d/d$; since $\norm u/\sqrt d=\sigma$, this is \eqref{eq:exact-law}.
\end{proof}
The active sample is centered at its own mean and shifted by $H$; the inactive sample is independent. Both share the scale $\sigma$. Without conditioning on $H$, the reads are not jointly Gaussian: $\Exp A_i=1$, $\Var A_i=(k+1)/d$ and $\Cov(A_i,A_l)=1/d$ for $i\ne l$. We use the conditional representation throughout.

Put $\xi_i=Z_i-\bar Z$ on $S$ and $\xi_{k+j}=W_j$ off $S$. A gate of class $z$ has score $H\ip az+\sigma\ip a{\xi_I}$. Let $M_z=\max_{I\in\cF_z(S)}\ip{\varsigma_za}{\xi_I}$, with $\max\emptyset=-\infty$. For a fixed representation, the whole network is correct almost surely exactly when every class satisfies $\sigma M_z<s_z^H$, where $s_z^H=(2g(z)-1)(H\ip az-\theta)$ (\cref{prop:cloud}). Counting errors instead uses the following marginal law.
\begin{proposition}[Exact one-gate law]\label{prop:onegate}
Let $a\neq0$. The score of a canonical gate of class $z$ has the law of
\begin{equation}\label{eq:onegate}
Y_z=\ip azJ+\sqrt{\bigl(k\norm a^2-\ip az^2\bigr)J/d}\;\zeta,
\end{equation}
where $J\sim\chi^2_d/d$ and $\zeta\sim N(0,1)$ are independent. Consequently $\Exp X=\sum_zN_zp_z$ with $p_z=\Prob\bigl(\varsigma_z(Y_z-\theta)\ge0\bigr)$.
\end{proposition}
The proof is in \cref{sm:reduction}. For pairwise conjunction, $a=(1,1)$ gives $Y_j=jJ+\sqrt{(2k-j^2)J/d}\,\zeta$, $j\in\{0,1,2\}$. Thus each class error probability is a one-dimensional chi-square integral of a normal tail.

%% file: sections_ball.tex
\section{Counting errors: the ball}\label{sec:ball}
Linearity of expectation removes dependence between gates. At the exponential scale, a class with margin $y\sigd\norm a$ contributes $N_z\exp(-y^2/2)$ errors. The class size sets the Gaussian tail level.
\begin{theorem}[Ball law]\label{thm:ball}
Let $\cF$ be a gate family of fixed arity $r$ with $N_z\to\infty$ for every $z$, and let $(a,\theta)\in\Rep(g)$ be any sequence of representations for the canonical rule. For each fixed $\varepsilon\in(0,1)$,
\[
\begin{aligned}d\ge(1+\varepsilon)\Dball(a,\theta)&\ \Longrightarrow\ \Exp X\to0,\\d\le(1-\varepsilon)\Dball(a,\theta)&\ \Longrightarrow\ \Exp X\to\infty .\end{aligned}
\]
\end{theorem}
\noindent\emph{Proof sketch.} Concentrate $J$ in \eqref{eq:onegate} in a window proportional to the margin. This keeps the bound uniform when margins shrink. Gaussian upper tails control every class above the threshold; Mills' lower bound gives divergence in a binding class below it. See \cref{sm:ball}.

Optimizing the price also gives a converse over all local readouts.
\begin{theorem}[Weighted margin law]\label{thm:weighted}
Let $\cF$ be as in \cref{thm:ball}, and let
\begin{equation}\label{eq:program}
\Vball=\min_{(a,\theta)\in\Rep(g)}\ \max_z\frac{2\log N_z\,\norm a^2}{s_z(a,\theta)^2}.
\end{equation}
Fix $\varepsilon\in(0,1)$. If $d\ge(1+\varepsilon)k\Vball$, the canonical rule with a minimizing representation has $\Exp X\to0$. If $d\le(1-\varepsilon)k\Vball$, every gate-local readout rule satisfies
\[
\Exp X\ge\bigl(\min_zN_z\bigr)^{\varepsilon-o(1)}\to\infty .
\]
\end{theorem}
\noindent\emph{Proof sketch.} Normalize a local probe to unit length. Conditional on its own columns and class $z$, its interference is $N(0,(k-\abs z)/d)$. On a high-probability Gram-matrix event, its clean weights have norm at most $1+o(1)$. The program forces a small margin in some class; a Gaussian lower tail then bounds that gate's error. Summing these marginal bounds needs no independence between errors. Bounded $d$ uses a separate two-class argument (\cref{sm:ball}).

For computation, maximize $\kappa$ over $(a,\theta,\kappa)$ with $\norm a\le1$ and
\begin{equation}\label{eq:socp}
(2g(z)-1)\bigl(\ip az-\theta\bigr)\ge\kappa\sqrt{2\log N_z}\quad\text{for all }z.
\end{equation}
Then $\Vball=\kappa_*^{-2}$. This second-order cone program has $r+2$ variables and $2^r$ linear constraints. It maximizes margin relative to class size: a less numerous class can tolerate a smaller margin. Every class remains a constraint.

%% file: sections_box.tex
\section{Reliability: the supremum and the box}\label{sec:box}
Reliability depends on the largest interference in each class. Let $G\sim N(0,I_m)$ be independent of $S$, and define
\begin{equation}\label{eq:gamma}
\Gamma_z(S)=\Exp\Bigl[\max_{I\in\cF_z(S)}\ip{\varsigma_za}{G_I}\ \Big|\ S\Bigr],
\end{equation}
with $\max\emptyset=-\infty$.
\begin{theorem}[Supremum law]\label{thm:sup}
Let $\cF$ be any gate family of fixed arity $r$, let $(a,\theta)\in\Rep(g)$ be fixed, and assume $d/k\to\infty$. Define the random threshold
\[
d^*(S)=k\max_z\frac{\Gamma_z(S)_+^2}{s_z(a,\theta)^2}.
\]
Then for each fixed $\varepsilon\in(0,1)$,
\[
\begin{aligned}\Prob\bigl(X=0,\ d\le(1-\varepsilon)d^*(S)\bigr)&\to0,\\\Prob\bigl(X\ne0,\ d\ge(1+\varepsilon)d^*(S)\bigr)&\to0 .\end{aligned}
\]
In particular, when $d^*$ is deterministic it is the sharp reliability threshold among sequences with $d/k\to\infty$; for the universal family \cref{thm:box} removes this restriction.
\end{theorem}
\noindent\emph{Proof sketch.} The class maximum is $\norm a$-Lipschitz in $G$, so Gaussian concentration controls its deviation from $\Gamma_z(S)$ \citep{Borell75,TIS76}. The terms from $H-1$ and $\bar Z$ vanish relative to fixed margins when $d/k\to\infty$. The two threshold inequalities then give success or failure (\cref{sm:reduction}).
\begin{corollary}[Counting is the union bound]\label{cor:union}
In the setting of \cref{thm:sup}, $\Gamma_z(S)\le\norm a\sqrt{2\log\abs{\cF_z(S)}}$, so
\[
d^*(S)\le k\max_z\frac{2\log\abs{\cF_z(S)}\,\norm a^2}{s_z(a,\theta)^2},
\]
and for the universal family the right side is $\Dball(a,\theta)$. If the tuples in $\cF_z(S)$ are pairwise disjoint and $\abs{\cF_z(S)}\to\infty$, then $\Gamma_z(S)=\norm a\sqrt{2\log\abs{\cF_z(S)}}\,(1+o(1))$: without shared features, the reliability price of the class equals its ball price.
\end{corollary}
For nonnegative weights, shared features raise covariances and can lower the expected supremum \citep{Sudakov71,Fernique75}. Disjoint gates attain the ball price. A family with private partner sets shows when sharing first reduces the price.
\begin{proposition}[Partial sharing]\label{prop:star}
Consider pairwise gates with $a=(1,1)$, and suppose that on the support $S$ the gates of the one-active class are the pairs $(i,j)$ with $i\in S$ and $j\in C(i)$, where the partner sets $C(i)$, $i\in S$, are disjoint sets of $D=D(k)$ inactive features. Then, with $\lambda_D=\sqrt{2\log D}$,
\[
\Gamma_{10}(S)=(1+o(1))\begin{cases}2\sqrt{\log(kD)}&\text{if }D\le k,\\ \lambda_1+\lambda_D&\text{if }D\ge k.\end{cases}
\]
Equivalently, $\Gamma_{10}(S)=(1+o(1))\,h_{K_D}(1,1)$ for $K_D=\{v\in\R^2:v_1^2\le\lambda_1^2,\ v_1^2+v_2^2\le2\log(kD)\}$, the ball of the class cut by the slab of its shared coordinate.
\end{proposition}
The maximum is $\max_{i\in S}(Z_i+M_i)$, where $M_i$ is the largest of $D$ independent partner reads. For $D\le k$, both coordinates can split the extreme-value budget evenly. For $D\ge k$, the active coordinate saturates at $\lambda_1$ and the partner contributes $\lambda_D$. The expressions agree at $D=k$; disjoint partner sets require $kD\le n$. The proof is in \cref{sm:star}.

\paragraph{The universal family}
Every tuple is available, so the largest class projection is a rearranged sum of extreme feature reads (\cref{lem:boxid}). Each position contributes its weight magnitude times the level of its active or inactive pool. The finitely many relevant order statistics approximate these levels uniformly over weights. This gives the box law.
\begin{theorem}[Box law]\label{thm:box}
Let $\cF=\cU_r(g)$ with $k\to\infty$ and $m/k\to\infty$.
\begin{enumerate}[label=\textup{(\roman*)},leftmargin=2em]
\item For every sequence $(a,\theta)\in\Rep(g)$ and fixed $\varepsilon\in(0,1)$: $\Prob(X=0)\to1$ if $d\ge(1+\varepsilon)\Dbox(a,\theta)$, and $\Prob(X=0)\to0$ if $d\le(1-\varepsilon)\Dbox(a,\theta)$.
\item If $d\le(1-\varepsilon)k\Vbox$, then with probability tending to one no common representation makes the canonical network correct: $\Prob\bigl(\exists(a,\theta)\in\Rep(g):X_{a,\theta}=0\bigr)\to0$, where $X_{a,\theta}$ is the error count of the rule $(a,\theta)$ on the same draw.
\item $\Gamma_z=h_{\Boxz z}(a)(1+o(1))$ uniformly in $a\ne0$.
\end{enumerate}
In particular $k\Vbox$ is the sharp reliability threshold over common representations, and it is attained by any sequence of near-minimizers of $\Dbox$.
\end{theorem}
\noindent\emph{Proof sketch.} On one event of probability tending to one, all upper and lower order statistics needed by any gate are within a factor $1\pm\delta$ of their levels. The box identity then gives the support function uniformly over all representations. The common-scale and centering errors are smaller, so correctness reduces to the box-margin inequalities (\cref{sm:box}).

The representation in part (ii) may depend on the reads and is shared by all gates. Extending the reliability converse to all gate-local rules is an open problem.
\begin{corollary}[Counting overprices reliability]\label{cor:cs}
For the universal family, every representation and every class,
\[
h_{\Boxz z}(a)=\ip{\abs a}{\lambda_z}\le\norm a\norm{\lambda_z}=h_{\Ballz z}(a)\,(1+o(1)),
\]
with equality if and only if $\abs a$ is parallel to $\lambda_z$. Hence $\Dbox\le\Dball(1+o(1))$ and $\Vbox\le\Vball(1+o(1))$, and the ball price of a class exceeds its box price by the factor $\sec^2\angle(\abs a,\lambda_z)\,(1+o(1))$.
\end{corollary}
Mixed active and inactive classes create the gap: their side lengths need not align with the weights.

For pairwise conjunction, the one-active class gives a concrete calculation. Its weights are $(1,1)$ and its size is $kn$, so the ball support is $\sqrt{2(\lambda_1^2+\lambda_0^2)}\,(1+o(1))$, whereas the box support is $\lambda_1+\lambda_0$. Dividing both by the same margin gives the ratio $2(1+\beta^2)/(1+\beta)^2$ between the two class prices. This ratio equals one at $\beta=1$ and tends to two as $\beta\to\infty$. The largest inactive read is reused by every active partner. Counting each pair separately charges for $kn$ opportunities; reliability uses the two shared pools.

\paragraph{Optimizing the box}
Write
\begin{equation}\label{eq:vbox}
\begin{aligned}
\Vbox&=\lambda_1^2v_g(\beta),\\
v_g(\beta)&=\inf_{(a,\theta)\in\Rep(g)}\max_z\Bigl(\frac{\sum_l\abs{a_l}\beta^{1-z_l}}{s_z(a,\theta)}\Bigr)^2.
\end{aligned}
\end{equation}
For fixed $\kappa$, the constraints $s_z\ge\kappa\sum_l\abs{a_l}\beta^{1-z_l}$ form a convex polyhedral cone. Introduce $u_l\ge\pm a_l$, require $s_z\ge1$, and use linear feasibility with bisection on $\kappa$. This computes $v_g(\beta)^{-1/2}$ (\cref{prop:boxprogram}). For symmetric threshold gates, permutation averaging makes equal weights optimal in both programs.

%% file: sections_gates.tex
\section{Threshold gates}\label{sec:gates}
For $1\le t\le r$, let $\THR^r_t(z)=\one\{\abs z\ge t\}$. This includes conjunction $\ANDg_r$, disjunction $\ORg_r$, majority $\MAJg_{2t-1}$ and feature recovery $\THR^1_1$. Equal weights $a=\one$ are optimal, leaving only the bias. Set $\rho_j=(j\lambda_1^2+(r-j)\lambda_0^2)^{1/2}$ for $0\le j\le r$.
\begin{theorem}[Threshold gates]\label{thm:thr}
Let $\cF=\cU_r(\THR^r_t)$ with $k\to\infty$ and $m/k\to\infty$. Only the classes with $t-1$ and $t$ active inputs can bind; both bind at the optimized biases, and the class $t-1$ binds at the midpoint. The following are sharp thresholds.
\begin{enumerate}[label=\textup{(\alph*)},leftmargin=2em]
\item Forbidding errors, optimized: $k\Vbox=k\bigl((2t-1)\lambda_1+(2r-2t+1)\lambda_0\bigr)^2$. It is attained by $a=\one$ with the \emph{half-slack bias} $\theta=t-\tfrac12+\tfrac12\sigd(\lhat_0-\lhat_1)$, which leaves equal slack to the two binding classes (the levels $\lambda_p$ in place of $\lhat_p$ give the same threshold), and below it no common representation succeeds.
\item Forbidding errors, midpoint $a=\one$, $\theta=t-\tfrac12$: $4k\bigl((t-1)\lambda_1+(r-t+1)\lambda_0\bigr)^2$.
\item Counting errors, midpoint: $4rk\,\rho_{t-1}^2\,(1+o(1))=8rk\log N_{t-1}(1+o(1))$.
\item Counting errors, optimized: $k\Vball=rk(\rho_t+\rho_{t-1})^2(1+o(1))$. It is attained by $a=\one$ with $\theta=t-1+\rho_{t-1}/(\rho_{t-1}+\rho_t)$, and below it every gate-local rule has $\Exp X\to\infty$.
\end{enumerate}
\end{theorem}
\noindent\emph{Proof sketch.} With equal weights, class $j$ has score $j$, box support $j\lambda_1+(r-j)\lambda_0$ and ball support $\sqrt r\rho_j$. Add the constraints for $j=t-1,t$ to optimize the bias. Classes farther from the boundary gain enough clean margin to stay slack. The converses follow from \cref{thm:box,thm:weighted}; details are in \cref{sm:gates}.

At first order, start with the two clean scores beside the decision boundary. They are $t-1$ and $t$, so their margins are $\theta-(t-1)$ and $t-\theta$ and sum to one. Moving the bias transfers margin from one class to the other. To minimize the largest support-to-margin ratio, assign the margins in proportion to the two supports. For the ball, these are $\sqrt r\rho_{t-1}$ and $\sqrt r\rho_t$; for the box, they are $(t-1)\lambda_1+(r-t+1)\lambda_0$ and $t\lambda_1+(r-t)\lambda_0$. Adding the two inequalities gives the dimension, and balancing them gives the bias.

For pairwise conjunction this gives \cref{tab:four}. The half-slack bias assigns more margin to negative classes because their inactive inputs have larger extremes.
\begin{corollary}[One point, four limits]\label{cor:markov}
If $\beta\to1$, all four thresholds of \cref{thm:thr} equal $4r^2k\lambda_1^2(1+o(1))=8r^2k\log k\,(1+o(1))$. If $\beta\to\infty$, in units of $2k\log m$ they tend to (a) $(2r-2t+1)^2$, (b) $4(r-t+1)^2$, (c) $4r(r-t+1)$ and (d) $r(\sqrt{r-t}+\sqrt{r-t+1})^2$. For conjunction these are $1,4,4r,r$, so counting overprices reliability by a factor tending to $r$ at both biases. For disjunction the midpoint prices (b) and (c) coincide for every $\beta$.
\end{corollary}

\begin{corollary}[Computation versus recovery]\label{cor:recovery}
For feature recovery the two criteria share the threshold $k(\lambda_1+\lambda_0)^2$, attained by $\theta=\tfrac12+\tfrac12\sigd(\lhat_0-\lhat_1)$; below it the probes $e_i$ fail with every common bias, and every gate-local recovery rule has a diverging expected number of errors. A decoder that recovers every feature by such threshold probes and then applies exact Boolean conjunctions therefore needs this dimension, while direct canonical $r$-way conjunction needs $k((2r-1)\lambda_1+\lambda_0)^2$. Their ratio $\bigl((2r-1+\beta)/(1+\beta)\bigr)^2$ decreases from $r^2$ at $\beta=1$ to $1$ as $\beta\to\infty$.
\end{corollary}
The ratio comes from the readout architecture. A common conjunction bias must separate the $r$ smallest active reads from the $r-1$ largest active reads plus one inactive read. It therefore pays for $2r-1$ active levels; recovery pays for one. At $m=k^4$, recovery needs $18k\log k$ dimensions and direct pairwise conjunction needs $50k\log k$. The midpoint pairwise rule needs $72k\log k$ for reliability and $80k\log k$ for the expected count.

Other decoders have different requirements. The Lasso recovers support from about $2k\log(m-k)$ measurements \citep{Wainwright09}, and $\ell_1$ minimization recovers a typical nonnegative sparse vector from about $2k\log(m/k)$ noiseless Gaussian measurements \citep{DonohoTanner09}. In the skewed regime both match $2k\log m$ to first order; when $\log m\sim\log k$, the latter is $o(k\log k)$. The constants here describe one threshold layer reading a fixed random code. Further comparisons, including asymmetric comparator gates, are in \cref{sm:comparisons}.

%% file: sections_window.tex
\section{The critical window}\label{sec:window}
Near the threshold, a few extreme reads decide success. Their fluctuations have scales $1/\lambda_1$ and $1/\lambda_0$ in units of $\sigd$. For pairwise conjunction, define the quantile centers
\begin{equation}\label{eq:centers}
\begin{aligned}
d_0&=4k\bigl(\lhat_1+\lhat_0\bigr)^2,\\
d_0^{\rm opt}&=k\bigl(3\lhat_1+\lhat_0\bigr)^2,
\end{aligned}
\end{equation}
and $\Psi_\beta(x)=\Prob(G_A+G_B/\beta\le x)$ for independent standard Gumbel variables $G_A,G_B$.
\begin{theorem}[Window profiles]\label{thm:window}
Let $\cF$ be the universal family of pairwise conjunctions, $k\to\infty$ and $m/k\to\infty$. Uniformly on compact sets of $x$:
\begin{enumerate}[label=\textup{(\roman*)},leftmargin=2em]
\item the midpoint rule at $d=d_0+8k(1+\beta)x$ has $\Prob(X=0)=\Psi_\beta(x)+o(1)$;
\item at $d=d_0^{\rm opt}+2k(3+\beta)x$, the best deterministic common bias has
\[
\sup_\theta\Prob(X=0)=\sup_{y\in\R}\Psi_\beta(y)\,F_+(x-y)+o(1),
\]
where $F_+$ is the distribution function of $-\log[\omega_1(\omega_1+\omega_2)]$ for independent unit exponentials $\omega_1,\omega_2$, and the half-slack bias of \cref{thm:thr}(a) has $\Prob(X=0)=\Psi_\beta(x/2)F_+(x/2)+o(1)$.
\end{enumerate}
\end{theorem}
\noindent\emph{Proof sketch.} The midpoint rule is controlled by the largest active and inactive reads. The optimized bias must also protect the two smallest active reads. Gaussian upper and lower extremes become independent, giving Gumbel variables and the first two points of a unit Poisson process. The bias divides the available slack between these constraints. Common-scale and centering errors vanish in the window units (\cref{sm:window}).

The two factors in the optimized profile have separate roles. The one-active negative class is protected when the sum of the largest active read and the largest inactive read stays below the bias; this gives $\Psi_\beta$. The positive class is protected when the sum of the two smallest active reads stays above it; this gives $F_+$. Increasing the bias helps the first event and hurts the second. The variable $y$ allocates the window slack between them. The half-slack rule splits it equally; the supremum chooses the allocation that maximizes the product.

The optimized window is narrower by $4(1+\beta)/(3+\beta)\in[2,4)$. The half-slack bias is optimal to first order; a bias tuned to a finite-size success target can improve its limiting profile. The profiles have explicit integral forms:
\[
\Psi_\beta(x)=\int_0^\infty\exp(-\tau-\tau^{-\beta}e^{-\beta x})\,d\tau,
\]
and $F_+(x)=e^{-\sqrt c}+\int_0^{\sqrt c}e^{-c/s}\,ds$ with $c=e^{-x}$. For $\beta=1$, $\Psi_1(x)=2e^{-x/2}K_1(2e^{-x/2})$, where $K_1$ is the modified Bessel function; as $\beta\to\infty$, $\Psi_\beta(x)\to\exp(-e^{-x})$. At the centers, the midpoint profile is $\Psi_\beta(0)$ and the half-slack profile is $\Psi_\beta(0)F_+(0)$, with $F_+(0)=0.516$.

At the midpoint, the number of inactive features that cause an error has a mixed Poisson limit. Along $\beta\to\beta_0\in[1,\infty)$ its intensity is $\Lambda_\infty=e^{\beta_0(G_A-x)}$. Thus $\Prob(X=0)\to\Exp e^{-\Lambda_\infty}=\Psi_{\beta_0}(x)\in(0,1)$, while $\Exp\Lambda_\infty=\infty$ and $\Exp X\to\infty$. Rare large active reads make many gates fail together. The full argument is retained in \cref{sm:windowdetail}.

%% file: experiment_figure.tex
\begin{figure*}[t]
\centering
\includegraphics[width=\linewidth]{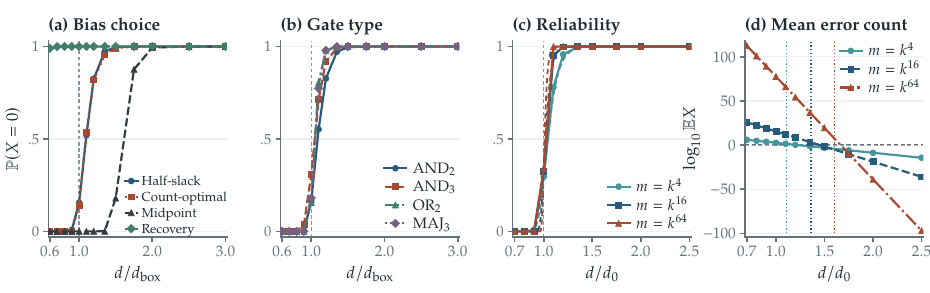}
\caption{Bias, gate type and the separation of reliability from expected errors, all at $k=1024$. (a) Three pairwise biases and feature recovery on shared draws at $m=k^4$. (b) Four gates with the half-slack bias at $m=k^4$. Here $d_{\rm box}=k((2t-1)\lhat_1+(2r-2t+1)\lhat_0)^2$, with $(r,t)=(2,2)$ in (a) and each gate's own $(r,t)$ in (b). (c,d) The midpoint pairwise rule at $m=k^4,k^{16},k^{64}$, normalized by $d_0$ of \eqref{eq:centers}. Points and bands in (a--c): 4096 exact draws per cell and pointwise $95\%$ Wilson intervals. Dashed curves in (c): $\Psi_{\hat\beta}$ from \cref{thm:window}(i). Panel (d): quadrature of \cref{prop:onegate}; dotted lines mark the first-order ball centers.}
\label{fig:e3}
\end{figure*}

%% file: sections_beyond.tex
\section{Extensions}\label{sec:beyond}

Three variations show how the dimension requirement depends on the model: supports chosen from the dictionary, non-Gaussian entries and multilayer computation. The proofs, circuit statements and encoding constructions are in \cref{sm:beyond}.

\subsection{Adversarial supports}
If the support can depend on the dictionary, the class cloud need not be a product of two samples. An adversary can select $k-1$ columns that align with one probe. The resulting interference grows linearly in $k$, rather than as $\sqrt k$.

\begin{proposition}[Uniform correctness]\label{prop:uniform}
Let $k\to\infty$ and $\log k=o(\log m)$, and use canonical pairwise probes with bias $3/2$. There are absolute constants $0<c<C$ such that the network is correct for every support of size at most $k$ with probability tending to one if $d\ge Ck^2\log m$, and fails on some support of size $k$ with probability tending to one if $d\le ck^2\log m$ and $d\to\infty$.
\end{proposition}

Related bounds for uniform feature recovery appear in \citep{GKP26,IvanitskiyJKM26}; random and worst-case supports also differ for sparse regression \citep{CandesPlan09}. The detailed comparisons are in \cref{sm:extensions}.

\subsection{Other entry laws}
The ball uses one-gate tails. For the midpoint pairwise network, their leading exponent is universal. Let the entries of $\sqrt d\,E$ be independent with a fixed centered, unit-variance, sub-Gaussian law, and let $q_{k,d}$ be the probability that a midpoint pairwise read with one active input exceeds $3/2$.

\begin{theorem}[Universality of the one-gate exponent]\label{thm:universality}
If $k\to\infty$ and $d/k\to\infty$, then
\[
\log q_{k,d}=-\frac d{16k}\Bigl[1+O\Bigl(\frac1k+\frac kd\log\frac dk\Bigr)\Bigr],
\]
with constants depending on the entry law. Consequently the midpoint pairwise network has $\Exp X\to0$ once $d\ge(16+\varepsilon)k\log(km)$.
\end{theorem}

Exponential tilting of coordinate rows proves the result. \Cref{sm:universality} gives the next term and sparse-entry extensions. The reliability law uses the Gaussian joint identity; its universality remains open.

\subsection{Composition and unrestricted encodings}\label{sec:composition}
A later layer must store the outputs it reads: pairwise conjunction has $\binom k2$ active outputs. For a Gaussian stream in which each active output is written once, \cref{thm:writebudget} gives the sufficient budget $d\ge(1+\varepsilon)\kappa_{\min}^{-1}K\log Q$. Here $Q$ counts reads, $K$ is the largest stream activity before a read, and $2\kappa_{\min}$ is the smallest squared clean margin of a unit-norm read. The proof applies a union bound to an ideal execution.

Without the linear dictionary, order $k$ dimensions suffice for all pairwise conjunctions. In the skewed regime they are also necessary for polynomially many threshold units, by sign-pattern counting \citep{Cover65} (\cref{prop:counting}). The coefficients of $\prod_{l\in S}(y-l)^2$ give a $2k$-dimensional nonlinear code, with large coefficients and small normalized margins \citep[Theorem~15.2]{Barvinok2010} (\cref{prop:polycode}). Thus the $\Theta(k\log m)$ scale describes accessibility from a random linear dictionary.

%% file: sections_experiments.tex
\section{Exact simulations}\label{sec:experiments}
The exact law samples a network without constructing $E$. Draw $H$, $k$ active Gaussians and the few inactive order statistics required by \cref{lem:boxid}. These order statistics can be sampled exactly by sequential uniform tail draws and transformed by $\Phi^{-1}$. Each draw costs $O(k\log k)$, independent of $m$ and $d$.

Each simulation point uses 4096 independent draws and pointwise $95\%$ Wilson intervals. Expected counts use quadrature of \cref{prop:onegate} with exact unordered class sizes: rare draws can dominate the mean. The protocol and complete results are in \cref{sm:experiments,sm:comparisons}.

\paragraph{Bias and gate type}
At $k=1024$, $m=1024^4$ and $d=1.35k(3\lhat_1+\lhat_0)^2$, the half-slack pairwise rule (\cref{fig:e3}(a)) succeeds on $3984$ of $4096$ draws. The count-optimal bias succeeds on $3912$, and the midpoint on $2$. At the quantile center, the half-slack success frequencies for $k=64,256,1024$ are $0.143,0.139,0.155$; the predicted limits are $0.155,0.154,0.154$. For $\ANDg_2$, $\ANDg_3$, $\ORg_2$ and $\MAJg_3$ at $m=k^4$, the transitions cluster near their box centers (\cref{fig:e3}(b)). At $k=1024$, success is at most $3.8\%$ at ratio $0.9$ and between $82.7\%$ and $97.9\%$ at $1.2$. The curves steepen as $k$ grows (\cref{fig:e2}).

\paragraph{Separation of the criteria}
Fix $k=1024$ and use the midpoint pairwise rule at $m=k^4,k^{16},k^{64}$ (\cref{fig:e3}(c,d)). At the quantile box center, success frequencies $0.295,0.324,0.340$ agree with profile values $0.298,0.326,0.346$. The expected count crosses one at about $1.23,1.45,1.67$ times $d_0$, tracking the first-order ball positions $1.11,1.36,1.60$. At $m=k^{64}$ and $d=1.5d_0$, all $4096$ draws are correct, yet quadrature gives $\Exp X\approx2.8\times10^{19}$. Here the network is correct on every sampled draw while its expected error count remains large. The exact sampler reaches these dictionary sizes by sampling extremes rather than constructing the dictionary.

%% file: discussion.tex
\section{Discussion}\label{sec:discussion}
Expected error counts use marginal tails and give a ball. Joint reliability uses shared reads through a Gaussian supremum; for the universal family, its geometry is a box.

The converse for reliability leaves a question about the readout class: can gate-local probes adapted to their own Gram matrices improve on common representations? Structured gate families raise a related geometric problem. The disjoint and partial-sharing examples suggest looking for sharp supremum constants; generic chaining gives the order \citep{Talagrand14}, while a variational description may resemble the generalized random energy model \citep{Derrida85}. Reliability for non-Gaussian, correlated or learned dictionaries is another question. For general gates, a window law would require the relevant extreme order statistics. For circuits, a box law would also need to track the dependence created by earlier writes; the present composition bound controls reads by a union bound.

%% file: ai_use.tex
\section*{AI Use Statement}\addcontentsline{toc}{section}{AI Use Statement}
The original research ideas and initial manuscript draft were developed by the authors. AI tools were used to assist with checking mathematical details and the correctness of arguments, filling in routine proof steps, coding assistance, and language polishing. The authors take full responsibility for the content and correctness of the final manuscript.

%% file: guide.tex
\numberwithin{equation}{section}\numberwithin{figure}{section}\numberwithin{table}{section}\setcounter{equation}{0}\setcounter{figure}{0}\setcounter{table}{0}
\section{Guide to the appendix}\label{sm:guide}

The article keeps the main threshold statements and the proof of the exact read law. The appendix gives the read-vector reductions, the supremum proof, the box identity and its optimization program (\cref{sm:reduction}); the remaining ball, box, gate and window proofs (\cref{sm:ball,sm:box,sm:gates,sm:window}); and the full extensions of the model (\cref{sm:extensions,sm:beyond}). Additional window formulas, comparisons and experimental results are in \cref{sm:windowdetail,sm:comparisons}. \Cref{sm:numerics} records finite computational checks; \cref{sm:experiments} gives the frozen experimental protocol. \Cref{sm:applications} turns the thresholds into design rules for width, capacity, sizing and evaluation, and states open problems. The notation table and dependency map provide a reading guide.

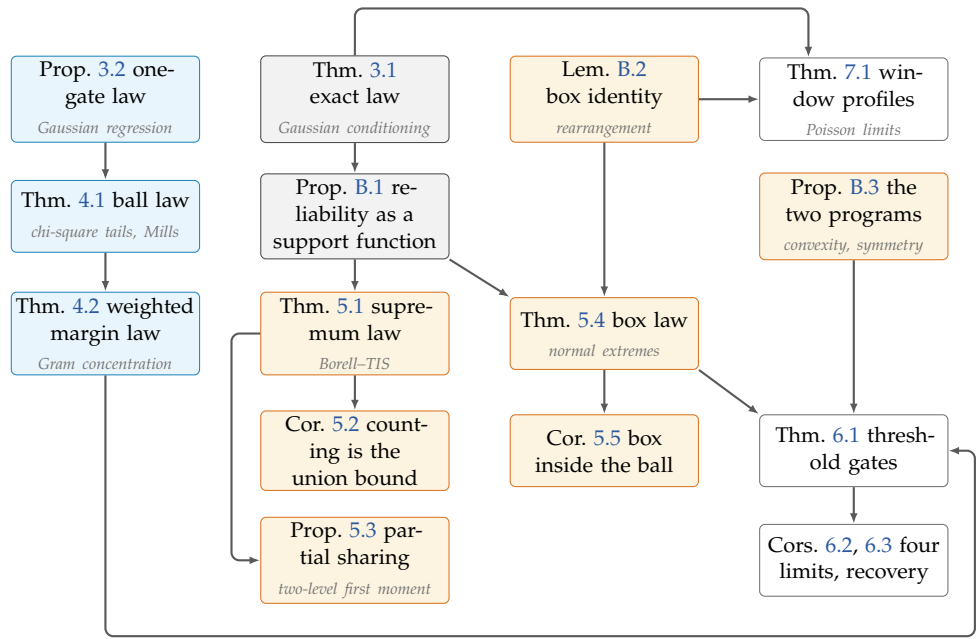
\begin{figure}[!ht]
\centering
\input{deps}
\caption{Dependency map. An arrow points from a result to one that builds on it. Blue: expected errors (the ball); orange: reliability (the supremum and the box); grey: the exact law and its reduction. Small italic text names the classical tool each proof adds.}
\label{fig:deps}
\end{figure}

\begin{table}[ht]
\caption{Notation.}
\label{tab:notation}
\centering
\small
\renewcommand{\arraystretch}{1.12}
\begin{tabular}{@{}p{0.27\linewidth}p{0.66\linewidth}@{}}
\toprule
Symbol & Meaning\\
\midrule
$E=[e_1|\cdots|e_m]$ & dictionary, independent columns $e_i\sim N(0,I_d/d)$\\
$S$, $x=\one_S$, $k$, $n=m-k$ & active set (uniform $k$-subset), input, sparsity, number of inactive features\\
$u=Ex$, $R=E^\top u$ & stored vector and feature reads; $A_i=R_i$ ($i\in S$), $B_j$ (inactive reads)\\
$H$, $Z$, $W$, $\sigma$ & $H\sim\chi^2_d/d$, independent Gaussian samples of sizes $k$ and $n$, $\sigma=\sqrt{kH/d}$ (\cref{thm:exact})\\
$\sigd=\sqrt{k/d}$ & deterministic interference scale\\
$\xi$, $\xi_I$ & $\xi_i=Z_i-\bar Z$ on $S$ and $W_j$ off $S$; interference vector of a tuple $I$\\
$(I,g)$, $r$, $\cF$, $\cU_r(g)$ & gate (tuple and threshold function), arity, gate family, universal family\\
$z$, $\cF_z(S)$, $N_z$ & class (input pattern), gates of class $z$, expected class size \eqref{eq:counts}\\
$(a,\theta)$, $\Rep(g)$ & weights and common bias of the canonical rule; representations of $g$\\
$s_z(a,\theta)$, $\varsigma_z$ & class margin \eqref{eq:margin}; error direction $1-2g(z)$\\
$X$ & number of wrong outputs\\
$\lambda_1,\lambda_0,\lambda_z,\beta$ & levels $\sqrt{2\log k}$, $\sqrt{2\log n}$, their vector for class $z$, ratio $\lambda_0/\lambda_1$\\
$\rho_z=\sqrt{2\log N_z}$ & ball radius of class $z$; $\rho_j$ in \cref{sec:gates} for classes with $j$ active inputs\\
$\Ballz z$, $\Boxz z$, $h_K$ & ball and box of \cref{def:prices}; support function of a set $K$\\
$\Dball$, $\Dbox$, $\Vball$, $\Vbox$ & ball and box prices \eqref{eq:prices} and their optimal values\\
$\Gamma_z(S)$, $d^*(S)$ & expected supremum of the class process \eqref{eq:gamma}; reliability threshold of \cref{thm:sup}\\
$\THR^r_t$ & threshold gate firing when at least $t$ of $r$ inputs are active\\
$\lhat_1,\lhat_0$ & quantile levels $\Phibar^{-1}(1/k)$ and $\Phibar^{-1}(1/n)$\\
$\Psi_\beta$, $F_+$ & window profiles of \cref{thm:window}\\
$G_A,G_B$, $\omega_1,\omega_2$ & standard Gumbel and unit exponential variables of \cref{thm:window}\\
$D$, $C(i)$, $\lambda_D$ & number of partners, partner set and level $\sqrt{2\log D}$ of \cref{prop:star}\\
$Q$, $K$, $\kappa_{\min}$ & number of reads, stream activity and margin rate of \cref{thm:writebudget}\\
\bottomrule
\end{tabular}
\end{table}

\FloatBarrier

%% file: deps.tex
\begin{tikzpicture}[
  x=1cm, y=1cm, >={Latex[length=4pt]},
  res/.style={draw=black!55, rounded corners=2pt, fill=white, align=center, font=\scriptsize, text width=2.35cm, minimum height=0.95cm, inner sep=2pt},
  exact/.style={res, draw=black!70, fill=black!5},
  ball/.style={res, draw=wongblue!85, fill=wongsky!13},
  box/.style={res, draw=wongverm!85, fill=wongorange!12},
  tool/.style={font=\tiny\itshape, text=black!55},
  dep/.style={->, black!65, thick, rounded corners=3pt},
]
\node[ball]  (og)  at (0,0)       {Prop.~\ref{prop:onegate} one-gate law\\[1pt]{\tiny\itshape\color{black!55}Gaussian regression}};
\node[exact] (ex)  at (3.3,0)     {Thm.~\ref{thm:exact} exact law\\[1pt]{\tiny\itshape\color{black!55}Gaussian conditioning}};
\node[box]   (bi)  at (6.6,0)     {Lem.~\ref{lem:boxid} box identity\\[1pt]{\tiny\itshape\color{black!55}rearrangement}};
\node[res]   (win) at (9.9,0)     {Thm.~\ref{thm:window} window profiles\\[1pt]{\tiny\itshape\color{black!55}Poisson limits}};
\node[ball]  (bl)  at (0,-1.55)    {Thm.~\ref{thm:ball} ball law\\[1pt]{\tiny\itshape\color{black!55}chi-square tails, Mills}};
\node[exact] (cl)  at (3.3,-1.55)  {Prop.~\ref{prop:cloud} reliability as a support function};
\node[box]   (prog) at (9.9,-1.55) {Prop.~\ref{prop:boxprogram} the two programs\\[1pt]{\tiny\itshape\color{black!55}convexity, symmetry}};
\node[ball]  (wm)  at (0,-3.1)    {Thm.~\ref{thm:weighted} weighted margin law\\[1pt]{\tiny\itshape\color{black!55}Gram concentration}};
\node[box]   (sup) at (3.3,-3.1)  {Thm.~\ref{thm:sup} supremum law\\[1pt]{\tiny\itshape\color{black!55}Borell--TIS}};
\node[box]   (box) at (6.6,-3.1)  {Thm.~\ref{thm:box} box law\\[1pt]{\tiny\itshape\color{black!55}normal extremes}};
\node[box]   (un)  at (3.3,-4.65)  {Cor.~\ref{cor:union} counting is the union bound};
\node[box]   (cs)  at (6.6,-4.65)  {Cor.~\ref{cor:cs} box inside the ball};
\node[res]   (thr) at (9.9,-4.65)  {Thm.~\ref{thm:thr} threshold gates};
\node[box]   (star) at (3.3,-6.1) {Prop.~\ref{prop:star} partial sharing\\[1pt]{\tiny\itshape\color{black!55}two-level first moment}};
\node[res]   (cor) at (9.9,-6.1)  {Cors.~\ref{cor:markov}, \ref{cor:recovery} four limits, recovery};
\draw[dep] (og) -- (bl);
\draw[dep] (bl) -- (wm);
\draw[dep] (ex) -- (cl);
\draw[dep] (cl) -- (sup);
\draw[dep] (sup) -- (un);
\draw[dep] (sup.west) -- ++(-0.38,0) |- (star.west);
\draw[dep] (bi) -- (box);
\draw[dep] (box) -- (cs);
\draw[dep] (prog) -- (thr);
\draw[dep] (thr) -- (cor);
\draw[dep] (cl.south east) -- (box.north west);
\draw[dep] (box.south east) -- (thr.north west);
\draw[dep] (bi) -- (win);
\draw[dep] (ex.north) -- ++(0,0.62) -| ([xshift=-6mm]win.north);
\draw[dep] (wm.south) -- ++(0,-3.45) -| ($(thr.east)+(0.35,0)$) -- (thr.east);
\end{tikzpicture}

%% file: reduction.tex
\section{Read-vector reductions and the supremum}\label{sm:reduction}
The following statements and proofs give the reductions used in Sections~\ref{sec:exact} and~\ref{sec:box}.
\subsection{Reliability as a support function}
\begin{proposition}[Reliability as a support function]\label{prop:cloud}
For the canonical rule with $(a,\theta)\in\Rep(g)$ and any gate family, let $C_z=\{\xi_I:I\in\cF_z(S)\}$ and
\[
M_z=\max_{I\in\cF_z(S)}\ip{\varsigma_za}{\xi_I}=h_{C_z}(\varsigma_za),\qquad s_z^H=(2g(z)-1)\bigl(H\ip az-\theta\bigr),
\]
with $\max\emptyset=-\infty$. For every realization, $\bigcap_z\{\sigma M_z<s_z^H\}\subseteq\{X=0\}\subseteq\bigcap_z\{\sigma M_z\le s_z^H\}$, and for a fixed representation the two outer events differ by an event of probability zero.
\end{proposition}

\begin{proof}
A gate of class $z$ is correct when $(2g(z)-1)(\ip a{R_I}-\theta)>0$, that is, when $\sigma\ip{\varsigma_za}{\xi_I}<s_z^H$, and it may also be correct when its score equals $\theta$; this gives the two inclusions. For a fixed representation such ties have probability zero: given $H$ the score is Gaussian, and when its conditional variance vanishes the score equals $H\ip az$, which has a density if $\ip az\ne0$ and equals $0\ne\theta$ otherwise, because $s_z>0$. Take the maximum over the class and intersect over the finitely many classes.
\end{proof}
\subsection{Proof of the exact one-gate law}
\begin{proof}
The coordinates of $(w,u)$ are independent pairs of centered Gaussians with variances $\norm a^2/d$ and $k/d$ and covariance $\ip az/d$. Hence $u=(\ip az/\norm a^2)w+v$ with $v\sim N\bigl(0,(k-\ip az^2/\norm a^2)I_d/d\bigr)$ independent of $w$, and $\ip wu=(\ip az/\norm a^2)\norm w^2+\ip wv$. Given $w$, the last term is centered Gaussian with variance $\norm w^2(k\norm a^2-\ip az^2)/(d\norm a^2)$. Put $J=\norm w^2/\norm a^2$. The formula for $\Exp X$ is linearity of expectation over gates, each gate contributing its class error probability with weight $\Prob(z_\gamma(S)=z)$.
\end{proof}
\subsection{Proof of the supremum law}
\begin{proof}
Work on the representation of \cref{thm:exact} and let $G\in\R^m$ collect $Z$ on $S$ and $W$ off $S$; it is $N(0,I_m)$ and independent of $H$. For a tuple of class $z$, $\ip a{\xi_I}=\ip a{G_I}-\bar Z\ip az$, so $M_z=T_z-\varsigma_z\bar Z\ip az$ with $T_z=\max_{I\in\cF_z(S)}\ip{\varsigma_za}{G_I}$. The map $G\mapsto T_z$ is $\norm a$-Lipschitz, because the coordinates of $G_I$ are distinct coordinates of $G$. By Gaussian concentration \citep{Borell75,TIS76},
\begin{equation}\label{eq:tis}
\Prob\bigl(\abs{T_z-\Gamma_z(S)}\ge y\mid S\bigr)\le2e^{-y^2/(2\norm a^2)}\qquad(y>0).
\end{equation}
Moreover $\abs{H-1}=O_\Prob(d^{-1/2})$, $\abs{\bar Z}=O_\Prob(k^{-1/2})$, $\sigma=\sigd\sqrt H$ and $s_z^H=s_z+(2g(z)-1)(H-1)\ip az$ with $\abs{\ip az}\le\sqrt r\norm a$. Because $(a,\theta)$ is fixed and $\sigd\to0$, every margin satisfies $s_z/(\sigd\norm a)\to\infty$. Choose $L_d\to\infty$ with $\sigd L_d\to0$ and $\delta_d\to0$ slowly, so that with probability tending to one $\abs{H-1}\le\delta_d$, $\abs{\bar Z}\le\delta_d$, and, by \eqref{eq:tis}, $T_z\le\Gamma_z(S)+L_d\norm a$ for all $z$.

On $\{d\ge(1+\varepsilon)d^*(S)\}$ every class has $\sigd\Gamma_z(S)_+\le s_z/\sqrt{1+\varepsilon}$, and then
\[
\sigma M_z\le(1+\delta_d)\Bigl(\frac{s_z}{\sqrt{1+\varepsilon}}+\sigd\norm a\,(L_d+\sqrt r\,\delta_d)\Bigr)<s_z-\delta_d\sqrt r\norm a\le s_z^H
\]
for all large $d$, because $\sigd\norm a(L_d+\sqrt r\delta_d)=o(s_z)$ and $\delta_d\norm a=o(s_z)$. On $\{d\le(1-\varepsilon)d^*(S)\}$ some class has $\sigd\Gamma_z(S)\ge s_z/\sqrt{1-\varepsilon}$, so $\Gamma_z(S)/\norm a\to\infty$ and \eqref{eq:tis} with $y=\tfrac\varepsilon4\Gamma_z(S)$ gives $T_z\ge(1-\tfrac\varepsilon4)\Gamma_z(S)$ with probability tending to one. For this class
\[
\sigma M_z\ge(1-\delta_d)\Bigl(\frac{(1-\varepsilon/4)\,s_z}{\sqrt{1-\varepsilon}}-\sigd\sqrt r\norm a\,\delta_d\Bigr)>s_z+\delta_d\sqrt r\norm a\ge s_z^H,
\]
since $(1-\varepsilon/4)/\sqrt{1-\varepsilon}>1$. There are $2^r$ classes, so both statements follow.
\end{proof}
\subsection{The box identity}
For the universal family the maximizing tuple can be written down. Group the positions of a class-$z$ gate by the pool they draw from ($p=1$ for the active features, $p=0$ for the inactive ones) and by the sign of their weight in the direction $b=\varsigma_za$.

\begin{lemma}[Box identity]\label{lem:boxid}
Let $\cF=\cU_r(g)$ and assume that each pool $p$ has at least as many features as the class has positions with $z_l=p$. For every $\xi\in\R^m$, $S$, $z$ and $b\in\R^r$,
\[
\max_{I\in\cF_z(S)}\ip b{\xi_I}=\sum_{p\in\{0,1\}}\ \sum_{\pm}\ \sum_{j=1}^{c^\pm_p}\abs b^{(p,\pm)}_{[j]}\,(\pm\xi)^{(p)}_{(j)},
\]
where $c^+_p$ and $c^-_p$ count the positions $l$ with $z_l=p$ and $b_l\ge0$, respectively $b_l<0$, $\abs b^{(p,\pm)}_{[1]}\ge\abs b^{(p,\pm)}_{[2]}\ge\cdots$ are their absolute weights in decreasing order, and $(\pm\xi)^{(p)}_{(1)}\ge(\pm\xi)^{(p)}_{(2)}\ge\cdots$ are the values $\pm\xi_i$, $i$ in pool $p$, in decreasing order.
\end{lemma}

\begin{proof}
For any admissible tuple, split $\ip b{\xi_I}$ into the four groups. Within a group the features are distinct members of one pool, so by the rearrangement inequality the group sum is at most the displayed sum of products of decreasing weights with decreasing values of $\pm\xi$. The bound is attained by placing, in each pool, the $c^+_p$ largest values on the positions with $b_l\ge0$ and the $c^-_p$ smallest on those with $b_l<0$; these features are distinct because the pool has at least $c^+_p+c^-_p$ elements, which is the assumption.
\end{proof}

Each group sum involves at most $r$ extreme order statistics of one Gaussian sample, and each of these is $\lambda_p(1+o_\Prob(1))$. Hence the maximal projection of the class is $\sum_l\abs{a_l}\lambda_{z_l}=h_{\Boxz z}(a)$ to first order: the cloud fills the box. Because the relative error comes from finitely many order statistics that do not depend on $a$, the approximation is uniform over all representations at once.

\subsection{Computing the box program}
\begin{proposition}[Computing the box]\label{prop:boxprogram}
For each $\kappa\ge0$ the set of $(a,\theta)$ with $s_z(a,\theta)\ge\kappa\sum_l\abs{a_l}\beta^{1-z_l}$ for all $z$ is a convex polyhedral cone, so $v_g(\beta)^{-1/2}$ is the largest $\kappa$ for which this cone contains a point with $s_z\ge1$ for all $z$; after introducing $u_l\ge\pm a_l$ this is a linear feasibility problem, and $\kappa$ is found by bisection. If $g$ is symmetric, both $\Vbox$ and $\Vball$ are attained at $a=\one$ when $g$ is nondecreasing, as for $\THR^r_t$, and at $a=-\one$ when $g$ is nonincreasing; every symmetric threshold function is one of the two.
\end{proposition}

\begin{proof}
Each constraint says that a linear function minus $\kappa$ times a convex piecewise-linear function is nonnegative, which defines a convex polyhedral set; homogeneity makes it a cone. If $g$ is symmetric, permuting the coordinates of $a$ permutes the classes and preserves the constraint system of both programs, so averaging a feasible point over all permutations gives a feasible point with equal weights. Its margins are averages of margins over permutation orbits, hence positive, so its common weight $c$ is nonzero, with $c>0$ if $g$ is nondecreasing and $c<0$ if $g$ is nonincreasing; rescale to $a=\pm\one$.
\end{proof}
\subsection{Partial sharing}\label{sm:sharing}
For nonnegative weights, two gates of the same class that share a feature have positively correlated interference, and every additional shared feature raises the correlation. By the Sudakov--Fernique inequality \citep{Sudakov71,Fernique75} the expected supremum of a class can only fall below that of independent gates with the same variances, which is the ball. The universal family is the opposite extreme to disjoint gates: every tuple of features is a gate, and the supremum splits into a sum of independent maxima, one for each position. Between the two extremes the supremum depends on how much sharing there is, and a simple family already shows where the supremum leaves the ball.

The partial-sharing statement is given in \cref{prop:star}.

\noindent\emph{Proof sketch} (details in \cref{sm:star}). The class supremum is $\max_{i\in S}(Z_i+M_i)$ with $M_i=\max_{j\in C(i)}W_j$, a maximum of $k$ independent copies of a Gaussian plus a maximum of $D$ Gaussians. When $D\le k$, many active features have moderately large reads, and among them one finds a partner with an equally large read: the two coordinates split the budget $2\log(kD)$ evenly, as in the ball. When $D\ge k$ the shared coordinate saturates at $\lambda_1$ and the partner adds $\lambda_D$. \hfill$\square$

\medskip
Sharing therefore lowers the price of a class only once every active feature has more partners than there are active features. The two formulas agree at $D=k$. Disjoint partner sets require $kD\le n$; continued formally to $D=n$, where partner sets overlap as in the universal family, the second formula becomes the box value $\lambda_1+\lambda_0$ of \cref{thm:box}(iii). Partner sets drawn at random, as in a random regular graph of gates, satisfy the disjointness hypothesis with probability tending to one when $k^2D^2=o(m)$.

%% file: sm_ball.tex
\section{Proofs for the ball}\label{sm:ball}

Throughout, $\Phibar=1-\Phi$ is the standard normal tail, and we use the two Mills bounds
\begin{equation}\label{sm:eq:mills}
\Phibar(x)\le e^{-x^2/2}\quad(x\ge0),\qquad\Phibar(x)\ge\frac{x}{1+x^2}\,\frac{e^{-x^2/2}}{\sqrt{2\pi}}\quad(x>0),
\end{equation}
and the chi-square bounds of \citet{LaurentMassart00}: for $J\sim\chi^2_d/d$ and $0<\eta\le1$,
\begin{equation}\label{sm:eq:chi}
\Prob(\abs{J-1}\ge\eta)\le2e^{-d\eta^2/8},\qquad\Var J=2/d .
\end{equation}
For a representation $(a,\theta)$ of $g$, the bias lies strictly between the largest rejected and the smallest accepted clean value, and every clean value has modulus at most $\norm a_1\le\sqrt r\norm a$; hence
\begin{equation}\label{sm:eq:marginbound}
0<s_z(a,\theta)\le2\sqrt r\,\norm a\qquad\text{for every }z .
\end{equation}

\subsection{Proof of \texorpdfstring{\cref{thm:ball}}{the ball law}}
Fix a class $z$ and abbreviate $s=s_z(a,\theta)$, $c=\ip az$ and $v^2=k\norm a^2-c^2$. Since $c^2\le\abs z\,\norm a^2\le r\norm a^2$, we have $(k-r)\norm a^2\le v^2\le k\norm a^2$. By \cref{prop:onegate} and $\varsigma_z(c-\theta)=-s$, a gate of class $z$ errs exactly when $v\sqrt{J/d}\,G'\ge s-\varsigma_zc(J-1)$ with $G'=\varsigma_zG\sim N(0,1)$ independent of $J$, so
\begin{equation}\label{sm:eq:pz}
p_z=\Exp\,\Phibar\Bigl(\frac{s-\varsigma_zc(J-1)}{v\sqrt{J/d}}\Bigr).
\end{equation}
It suffices to prove each claim along a subsequence of an arbitrary subsequence.

\smallskip\noindent\emph{Upper bound.} Let $d\ge(1+\varepsilon)\Dball(a,\theta)$, so that $ds^2/(k\norm a^2)\ge(1+\varepsilon)\rho_z^2$ for every $z$. Put $\eta=s/(2\sqrt r\norm a\,k^{1/4})$, which is at most $k^{-1/4}$ by \eqref{sm:eq:marginbound}. On $\{\abs{J-1}\le\eta\}$ we have $\abs{c(J-1)}\le\sqrt r\norm a\eta=\tfrac12sk^{-1/4}$ and $v\sqrt{J/d}\le\norm a\sqrt{k(1+\eta)/d}$, so the argument of $\Phibar$ in \eqref{sm:eq:pz} is at least
\[
x_z=\frac{s(1-\tfrac12k^{-1/4})}{\norm a\sqrt{k(1+k^{-1/4})/d}},\qquad x_z^2\ge(1+\varepsilon)(1-3k^{-1/4})\rho_z^2 .
\]
By \eqref{sm:eq:mills} this part contributes at most $N_z^{-(1+\varepsilon)(1-3k^{-1/4})}\le N_z^{-1-\varepsilon/2}$ for large $k$. By \eqref{sm:eq:chi} the complement has probability at most
\[
2\exp\Bigl(-\frac{ds^2}{32r\norm a^2k^{1/2}}\Bigr)\le2\exp\Bigl(-\frac{(1+\varepsilon)k^{1/2}}{32r}\,\rho_z^2\Bigr)=2N_z^{-(1+\varepsilon)k^{1/2}/(16r)} .
\]
Hence $N_zp_z\le N_z^{-\varepsilon/2}+2N_z^{1-k^{1/2}/(16r)}\to0$ for every class, and $\Exp X=\sum_zN_zp_z\to0$ because there are $2^r$ classes.

\smallskip\noindent\emph{Lower bound.} Let $d\le(1-\varepsilon)\Dball(a,\theta)$ and pass to a subsequence along which one class $z$ attains the maximum in $\Dball$; then $x_0^2:=ds^2/(k\norm a^2)\le(1-\varepsilon)\rho_z^2$. Pass to a further subsequence along which either $x_0\to\infty$ or $x_0$ stays bounded.

If $x_0\to\infty$, put $\eta=s/(\sqrt r\norm a\log k)\le2/\log k$. By \eqref{sm:eq:marginbound}, $d\ge kx_0^2/(4r)\to\infty$, and $d\eta^2=kx_0^2/(r\log^2k)\to\infty$, so $\Prob(\abs{J-1}\le\eta)\to1$. On this event $\abs{c(J-1)}\le s/\log k$ and $v\sqrt{J/d}\ge\norm a\sqrt{(k-r)(1-\eta)/d}$, so the argument of $\Phibar$ is at most $x_1=x_0(1+1/\log k)\{(1-r/k)(1-\eta)\}^{-1/2}$, with $x_1^2\le(1-\varepsilon)(1+o(1))\rho_z^2$. The lower Mills bound gives $p_z\ge(1-o(1))N_z^{-(1-\varepsilon)-o(1)}$, and $N_zp_z\ge N_z^{\varepsilon-o(1)}\to\infty$.

If $x_0$ stays bounded, consider first subsequences with $d\to\infty$. By Chebyshev's inequality $\Prob(\abs{J-1}\le2\sqrt{2/d})\ge3/4$; on this event $\abs{c(J-1)}\le2\sqrt{2r}\norm a/\sqrt d$ and $v\sqrt{J/d}\ge\norm a\sqrt{(k-r)/(2d)}$ for $d\ge32$, so the argument of $\Phibar$ is at most $\sqrt2\,x_0\sqrt{k/(k-r)}+4\sqrt{r/(k-r)}$, which is bounded. Hence $p_z$ is bounded below and $N_zp_z\to\infty$. If instead $d$ stays bounded, then $\Prob(J\in[1/2,2])$ is bounded below, and on that event the argument of $\Phibar$ is at most $(s+\sqrt r\norm a)/(\norm a\sqrt{(k-r)/(2d)})\to0$ by \eqref{sm:eq:marginbound}; again $p_z$ is bounded below. In every case $\Exp X\ge N_zp_z\to\infty$. \hfill$\square$

\subsection{Proof of \texorpdfstring{\cref{thm:weighted}}{the weighted margin law}}
\emph{The program.} Write $L=\max_z\log N_z$. A fixed nonconstant threshold function has a representation with positive margins, which gives $\Vball\le C_gL$. Conversely, by \eqref{sm:eq:marginbound} every representation with $\norm a\le1$ has margins at most $2\sqrt r$, so $\Vball\ge L/(2r)$; thus $\Vball=\Theta_g(L)$. In the cone form \eqref{eq:socp} the bias may be restricted to $\abs\theta\le\sqrt r$ and $t$ to a bounded interval, so the feasible set is compact and the maximum $t_*>0$ is attained; at a maximizer all margins are positive. This minimizer of \eqref{eq:program} has $\Dball=k\Vball$, and the first claim is \cref{thm:ball}.

\smallskip\noindent\emph{Converse.} Fix a gate $\gamma$ with tuple $I$ and condition on $E_I$ and on the independent randomness of its rule; these determine $(w,\theta_\gamma)$. Write $q_z=N_z/\abs\cF$ for the probability that $\gamma$ has class $z$, which is the same for every gate of the family. If $w=0$ the output is constant and the gate errs with probability at least $\min_zq_z$. Otherwise rescale $(w,\theta_\gamma)$ so that $\norm w=1$, which does not change the output, and put $a=E_I^\top w$. Given the class $z$, the set $S\setminus I$ is a uniform $(k-\abs z)$-subset of the other features, whose columns are independent of $E_I$ and $w$; hence the score is $\ip az+\zeta$ with $\zeta\sim N(0,(k-\abs z)/d)$ given $E_I$, $w$ and the class. Only the marginal law of the class of one gate enters, so no exchangeability of the rules and no independence between gates is needed.

Pass to subsequences along which either $d\to\infty$ or $d$ stays bounded, and consider first $d\to\infty$. Let $\eta_d=d^{-1/4}$ and let $\mathcal G_\gamma=\{\norm{E_I^\top E_I-I_r}_{\rm op}\le\eta_d\}$. Gram concentration for a fixed number $r$ of Gaussian columns \citep[Theorem~4.6.1]{Vershynin18} gives $\Prob(\mathcal G_\gamma^c)\le C_re^{-c_rd\eta_d^2}\to0$. On $\mathcal G_\gamma$, $\norm a^2\le\norm{E_I^\top E_I}_{\rm op}\le1+\eta_d$; put $c=(1+\eta_d)^{1/2}$, so $\norm{a/c}\le1$.

If $(a/c,\theta_\gamma/c)$ does not represent $g$, some class $z$ has $s_z(a,\theta_\gamma)\le0$, and since $\zeta$ is symmetric the gate errs in that class with conditional probability at least $1/2$; its error probability is at least $\min_zq_z/2$. Otherwise, by the definition of $\Vball$ some class $z$ has $s_z(a/c,\theta_\gamma/c)^2\le2\log N_z/\Vball$, hence $s_z(a,\theta_\gamma)^2\le(1+\eta_d)2\log N_z/\Vball$. For $d\le(1-\varepsilon)k\Vball$ the squared margin-to-noise ratio of this class is
\[
\frac{d\,s_z(a,\theta_\gamma)^2}{k-\abs z}\le(1-\varepsilon)(1+\eta_d)\frac{k}{k-r}\,2\log N_z=2(1-\varepsilon)(1+o(1))\log N_z,
\]
uniformly in the rule and in the class. By \eqref{sm:eq:mills} its conditional error probability is at least $N_z^{-(1-\varepsilon)-o(1)}$ (and at least a constant if the ratio stays bounded), so the gate errs with probability at least $q_zN_z^{-(1-\varepsilon)-o(1)}=N_z^{\varepsilon-o(1)}/\abs\cF$. In both cases, on $\mathcal G_\gamma$,
\[
\Prob(\gamma\text{ errs}\mid E_I,\text{rule})\ge\frac{(\min_zN_z)^{\varepsilon-o(1)}}{\abs\cF}.
\]
Integrating over $\mathcal G_\gamma$, whose probability tends to one uniformly in $\gamma$ when $d\to\infty$, and summing over the $\abs\cF$ gates gives $\Exp X\ge(\min_zN_z)^{\varepsilon-o(1)}$. No event simultaneous over all gates is used.

It remains to treat subsequences along which $d\le D$ stays bounded. Let $\Omega_\gamma=\{\norm{E_I}_{\rm op}^2\le2r\}$; since $\Exp\norm{E_I}_F^2=r$, Markov's inequality gives $\Prob(\Omega_\gamma)\ge1/2$. Because $g$ is nonconstant, some edge of the cube joins a class $z$ with $g(z)=1$ to a class $z'$ with $g(z')=0$, differing in one position $l$. With $\norm w=1$ and $a=E_I^\top w$, on $\Omega_\gamma$ we have $\abs{a_l}\le\sqrt{2r}$. Put $u=\theta_\gamma-\ip az$. If $u\ge0$, class $z$ errs with conditional probability $\Prob(\zeta\le u)\ge1/2$. If $u<0$, class $z'$ errs when its noise exceeds $\theta_\gamma-\ip a{z'}\le u+\abs{a_l}<\sqrt{2r}$, which has probability at least $\Phibar(\sqrt{2rD/(k-r)})\to1/2$. Hence each gate errs with probability at least $(\tfrac14-o(1))\min_zq_z$, and $\Exp X\ge(\tfrac14-o(1))\min_zN_z\to\infty$. \hfill$\square$

\begin{remark}[Adapting to the local Gram matrix]\label{sm:rem:dual}
A gate-local rule may use the dual frame $w=E_I(E_I^\top E_I)^{-1}a$, whose clean weights equal $a$ exactly. By \cref{thm:weighted} this cannot lower the first-order expected-error threshold below $k\Vball$, and by \cref{thm:ball} the canonical rule already attains it: the Gram fluctuations of a canonical probe are absorbed exactly by the chi-square variable $J$ of \cref{prop:onegate}.
\end{remark}

%% file: sm_box.tex
\section{Proofs for the box}\label{sm:box}

\subsection{Extremes of a Gaussian sample}
\begin{lemma}\label{sm:lem:extremes}
Let $Z_1,\dots,Z_q$ be independent standard Gaussians, $q\to\infty$, and fix $j\ge1$. The $j$ largest values divided by $\sqrt{2\log q}$ converge to $1$ in probability, and so do the $j$ smallest values divided by $-\sqrt{2\log q}$. Their expectations, divided by $\pm\sqrt{2\log q}$, also converge to $1$.
\end{lemma}

\begin{proof}
Fix $\delta\in(0,1)$ and put $\ell=\sqrt{2\log q}$. By \eqref{sm:eq:mills}, $q\Phibar((1+\delta)\ell)\le q^{-2\delta-\delta^2}\to0$, so by a union bound no value exceeds $(1+\delta)\ell$ with probability tending to one. Also $q\Phibar((1-\delta)\ell)\ge q^{\delta}/(C\ell)\to\infty$, so the number of values above $(1-\delta)\ell$ is binomial with a diverging mean and is at least $j$ with probability tending to one. Apply the same argument to $-Z$. For the expectations, each of these order statistics is at most $\max_i\abs{Z_i}$ in modulus, and $\Exp\max_iZ_i^2=(\Exp\max_i\abs{Z_i})^2+\Var(\max_i\abs{Z_i})\le2\log(2q)+1$ by the maximal inequality \citep[Section~2.5]{BoucheronLugosiMassart13} and the Gaussian Poincar\'e inequality for the $1$-Lipschitz function $\max_i\abs{Z_i}$; so the ratios are bounded in $L^2$ and hence uniformly integrable.
\end{proof}

\subsection{Proof of \texorpdfstring{\cref{prop:star}}{the partial-sharing proposition}}\label{sm:star}
Let $c$ be the value claimed in the proposition, that is, $c=2\sqrt{\log(kD)}$ if $D\le k$ and $c=\lambda_1+\lambda_D$ if $D\ge k$. In both cases $2\log(kD)\le c^2\le4\log(kD)$. Since $a=(1,1)$ and $G$ has the law of $-G$, we may take $\varsigma_{10}=1$; the class supremum is then $V^*=\max_{i\in S}V_i$ with $V_i=Z_i+M_i$ and $M_i=\max_{j\in C(i)}W_j$, and $\Gamma_{10}(S)=\Exp V^*$; the $2k$ variables $Z_i,M_i$ are independent, and each $M_i$ is the maximum of $D$ independent standard Gaussians. For $x\ge0$ and $u\in[0,x]$ put
\[
F_x(u)=\frac{u^2}2+\Bigl(\frac{(x-u)^2}2-\log D\Bigr)_+ .
\]

\smallskip\noindent\emph{The exponent.} On $\{u\le x-\lambda_D\}$, $F_x(u)=u^2/2+(x-u)^2/2-\log D$, and on $\{u\ge x-\lambda_D\}$, $F_x(u)=u^2/2$. Hence $\min_uF_x(u)=x^2/4-\log D$ if $x\ge2\lambda_D$ and $(x-\lambda_D)^2/2$ if $\lambda_D\le x\le2\lambda_D$. If $D\le k$ then $c\ge2\lambda_D$ and $c^2/4-\log D=\log k$; if $D\ge k$ then $\lambda_D\le c\le2\lambda_D$ and $(c-\lambda_D)^2/2=\log k$. So $\min_uF_c(u)=\log k$ in both cases. For $s\ge1$, substituting $u=su'$ and using $s^2\log D\ge\log D$ gives $F_{sx}(su')\ge s^2F_x(u')$, hence $\min_uF_{sc}(u)\ge s^2\log k$. Maximizing $v_1+v_2$ over the set $K_D$ of \cref{prop:star} gives the same two cases, so $h_{K_D}(1,1)=c$.

\smallskip\noindent\emph{Upper bound.} Fix $\delta\in(0,1)$, let $x=(1+\delta)c$, $L=\lceil2(1+\delta)/\delta\rceil$ and $u_l=lx/L$. If $V_i>x$, then $M_i>x$ and $Z_i\le0$, or $Z_i\in(u_l,u_{l+1}]$ and $M_i>x-u_{l+1}$ for some $l<L$, or $Z_i>x$. With $\Phibar(v)\le e^{-v^2/2}$ and $\Prob(M_i>v)\le\min\{1,De^{-v^2/2}\}$ for $v\ge0$, and $u_{l+1}=u_l+x/L$,
\[
k\Prob(V_i>x)\le2kD\,e^{-x^2/2}+k\sum_{l<L}e^{-F_{x-x/L}(u_l)}\le2(kD)^{1-(1+\delta)^2}+L\,k^{1-(1+\delta/2)^2},
\]
because $x^2\ge2(1+\delta)^2\log(kD)$ and $x-x/L\ge(1+\delta/2)c$. Both terms tend to zero, so $\Prob(V^*>(1+\delta)c)\to0$.

\smallskip\noindent\emph{Lower bound.} If $D\ge k$, let $i^*$ maximize $Z_i$. By \cref{sm:lem:extremes}, $Z_{i^*}\ge(1-\delta)\lambda_1$ with probability tending to one, and $M_{i^*}$ is independent of $(Z_i)$ and at least $(1-\delta)\lambda_D$ with probability tending to one because $D\ge k\to\infty$; hence $V^*\ge(1-\delta)c$. If $D\le k$, let $y=(1-\delta)\sqrt{\log(kD)}$ and let $N$ count the $i$ with $Z_i\ge y$ and $M_i\ge y$, a binomial variable with success probability $p=\Phibar(y)\Prob(M_1\ge y)$. Since $\Prob(M_1\ge y)=1-(1-\Phibar(y))^D\ge\tfrac12\min\{1,D\Phibar(y)\}$, $\Phibar(y)\ge e^{-y^2/2}/(C(1+y))$ and $kD\le k^2$,
\[
kp\ge\frac12\min\Bigl\{\frac{(kD)^{1-(1-\delta)^2}}{C^2(1+y)^2},\ \frac{k^{1-(1-\delta)^2}}{C(1+y)}\Bigr\}\to\infty,
\]
so $\Prob(N\ge1)\to1$, and on $\{N\ge1\}$, $V^*\ge2y=(1-\delta)c$.

\smallskip\noindent\emph{Expectations.} $\abs{V^*}\le\max_i\abs{Z_i}+\max_j\abs{W_j}$, the second maximum over the $kD$ partner reads, so the bound in the proof of \cref{sm:lem:extremes} gives $\Exp(V^*)^2\le4\log(2k)+4\log(2kD)+4=O(c^2)$. By Cauchy--Schwarz, $\Exp(V^*-(1+\delta)c)_+$ and $\Exp((1-\delta)c-V^*)_+$ are at most $O(c)$ times the square roots of $\Prob(V^*>(1+\delta)c)$ and $\Prob(V^*<(1-\delta)c)$, which tend to zero. Hence $(1-\delta)c-o(c)\le\Gamma_{10}(S)\le(1+\delta)c+o(c)$ for every $\delta$. \hfill$\square$

\subsection{Proof of \texorpdfstring{\cref{thm:box}}{the box law}}
Work in the representation of \cref{thm:exact} and let $G\in\R^m$ collect $Z$ on $S$ and $W$ off $S$. By \cref{prop:cloud}, for every realization and every representation $(a,\theta)$ the strict criterion
\begin{equation}\label{sm:eq:boxcrit}
\sigd\sqrt H\,\bigl(T_z(a)-\varsigma_z\bar Z\ip az\bigr)<s_z(a,\theta)+(2g(z)-1)(H-1)\ip az\qquad\text{for all }z,
\end{equation}
where $T_z(a)=\max_{I\in\cF_z(S)}\ip{\varsigma_za}{G_I}$, implies $X_{a,\theta}=0$, and $X_{a,\theta}=0$ implies \eqref{sm:eq:boxcrit} with $<$ replaced by $\le$. The success statements use the first implication and the failure statements only the second, so no almost-sure qualification over the uncountably many representations is needed. By \cref{lem:boxid} applied to $\xi=G$ and $b=\varsigma_za$, $T_z(a)$ is a sum over the positions of $\abs{a_l}$ times one of the $r$ largest values of $G$ or of $-G$ in the pool of that position, the values being assigned by rearrangement. Both pools have at least $r$ elements for large $k$, so \cref{lem:boxid} applies to every class.

Fix $\delta\in(0,1/8)$ and let $\mathcal E_\delta$ be the event on which (i) the $r$ largest and the $r$ smallest values of $(Z_i)_{i\le k}$ lie within a factor $1\pm\delta$ of $\pm\lambda_1$, and those of $(W_j)_{j\le n}$ within a factor $1\pm\delta$ of $\pm\lambda_0$; (ii) $\abs{\bar Z}\le\delta\lambda_1$; (iii) $\abs{H-1}\le\delta\sigd\lambda_1$; and (iv) $\abs{\sqrt H-1}\le\delta$. By \cref{sm:lem:extremes}, (i) and (ii) hold with probability tending to one. For (iii), Chebyshev's inequality gives $\Prob(\abs{H-1}>\delta\sigd\lambda_1)\le(2/d)/(\delta^2k\lambda_1^2/d)=2/(\delta^2k\lambda_1^2)\to0$, for every $d$. Condition (iv) holds with probability tending to one when $d\to\infty$; we return to bounded $d$ at the end.

On $\mathcal E_\delta$, for every $a\ne0$ and every class, (i) gives $T_z(a)\in[1-\delta,1+\delta]\,h_{\Boxz z}(a)$. Since $\lambda_0\ge\lambda_1$ eventually, $h_{\Boxz z}(a)\ge\lambda_1\norm a_1\ge\lambda_1\abs{\ip az}$, so by (ii) $\abs{\bar Z\ip az}\le\delta h_{\Boxz z}(a)$ and by (iii) $\abs{(H-1)\ip az}\le\delta\sigd h_{\Boxz z}(a)$. With (iv), the non-strict form of \eqref{sm:eq:boxcrit} implies $s_z\ge(1-\delta)(1-2\delta)\sigd h_{\Boxz z}(a)-\delta\sigd h_{\Boxz z}(a)\ge(1-4\delta)\sigd h_{\Boxz z}(a)$ for all $z$, and it is implied by $s_z\ge(1+5\delta)\sigd h_{\Boxz z}(a)$ for all $z$, since then the left side is at most $(1+\delta)(1+2\delta)\sigd h_{\Boxz z}(a)$ and the right side at least $(1+4\delta)\sigd h_{\Boxz z}(a)$. These two implications hold simultaneously for all representations, because $\mathcal E_\delta$ does not depend on $(a,\theta)$.

\smallskip\noindent\emph{(i)} If $d\ge(1+\varepsilon)\Dbox(a,\theta)$, every class has $s_z\ge\sqrt{1+\varepsilon}\,\sigd h_{\Boxz z}(a)$, which exceeds $(1+5\delta)\sigd h_{\Boxz z}(a)$ once $\delta$ is small; so $X_{a,\theta}=0$ on $\mathcal E_\delta$. Here $d\ge\Dbox\ge4k\lambda_1^2\to\infty$ (see below), so (iv) is available. If $d\le(1-\varepsilon)\Dbox(a,\theta)$, some class has $s_z\le\sqrt{1-\varepsilon}\,\sigd h_{\Boxz z}(a)<(1-4\delta)\sigd h_{\Boxz z}(a)$, so $X_{a,\theta}\ne0$ on $\mathcal E_\delta$.

\smallskip\noindent\emph{(ii)} If $d\le(1-\varepsilon)k\Vbox$, every representation has $\Dbox(a,\theta)\ge k\Vbox\ge d/(1-\varepsilon)$, so the previous argument applies to all of them at once on $\mathcal E_\delta$.

\smallskip\noindent\emph{(iii)} By \cref{lem:boxid}, $\Gamma_z=\Exp T_z(a)$ is a combination of expectations of the $r$ largest values of $\pm G$ in the two pools with coefficients $\abs{a_l}$, and \cref{sm:lem:extremes} gives $\Gamma_z=h_{\Boxz z}(a)(1+o(1))$ uniformly in $a$.

\smallskip\noindent\emph{Small dimensions.} Because $g$ is nonconstant, some edge of the cube joins a class $z$ with $g(z)=1$ to a class $z'$ with $g(z')=0$; they differ in one position $l$, so $s_z+s_{z'}=\abs{a_l}\le\norm a_1$ and $s^H_z+s^H_{z'}=H\ip a{z-z'}\le H\norm a_1$. Since $h_{\Boxz z}(a)\ge\lambda_1\norm a_1$ for every class, the first identity gives $\Dbox(a,\theta)\ge4k\lambda_1^2$ for every representation; in particular $d\to\infty$ in the first half of (i). Adding the two non-strict criteria for $z$ and $z'$ shows that success requires $\sigd\sqrt H\,(M_z+M_{z'})\le H\norm a_1$, with $M_z=T_z(a)-\varsigma_z\bar Z\ip az$. On the event (i)--(ii), $M_z+M_{z'}\ge2(1-2\delta)\lambda_1\norm a_1$, so success requires $H\ge4(1-2\delta)^2\lambda_1^2k/d$. Along subsequences with bounded $d$ this has probability tending to zero, simultaneously for all representations, which completes the failure statements in (i) and (ii) without (iv). \hfill$\square$

\begin{remark}[What the converse covers]\label{sm:rem:oracle}
The event $\mathcal E_\delta$ does not depend on $(a,\theta)$, and on it every common representation fails. The converse in \cref{thm:box}(ii) therefore covers any common representation, including one chosen after observing $(A_i)$ and $(B_j)$, for instance the sample-adaptive bias recorded as an oracle in \cref{sm:experiments}. It does not cover rules in which different gates use different weights; for those, \cref{thm:weighted} bounds only the expected count from below, and no lower bound on their reliability threshold is known.
\end{remark}

%% file: sm_gates.tex
\section{Proofs for threshold gates}\label{sm:gates}

\subsection{Proof of \texorpdfstring{\cref{thm:thr}}{the threshold-gate theorem}}
By \cref{prop:boxprogram} it suffices, for the optimized thresholds, to consider $a=\one$; for the midpoint thresholds $a=\one$ by definition. With $a=\one$, all classes with $j$ active inputs have clean score $j$, box support $h_j=j\lambda_1+(r-j)\lambda_0$ and ball support $\sqrt r\,\rho_z$ with $\rho_z^2=2\log N_z=\rho_j^2(1+o(1))$, where $\rho_j^2=j\lambda_1^2+(r-j)\lambda_0^2$. A bias represents $\THR^r_t$ if and only if $\theta\in(t-1,t)$, and then the margins are $s_j=\theta-j$ for $j\le t-1$ and $s_j=j-\theta$ for $j\ge t$. Since $\lambda_0\ge\lambda_1$, both $h_j$ and $\rho_j$ decrease in $j$.

\smallskip\noindent\emph{Only the classes $t-1$ and $t$ bind.} For $j=t-1-i$ with $i\ge1$, $s_j=s_{t-1}+i$ and $h_j=h_{t-1}+i(\lambda_0-\lambda_1)$, so $h_j/s_j\le\max\{h_{t-1}/s_{t-1},\lambda_0-\lambda_1\}$. For $j=t+i$, $s_j=s_t+i$ and $h_j\le h_t$, so $h_j/s_j\le h_t/s_t$. Since $s_{t-1}<1$ and $h_{t-1}\ge\lambda_0>\lambda_0-\lambda_1$, the first maximum is $h_{t-1}/s_{t-1}$: the classes below $t-1$ never exceed the ratio of class $t-1$. For the ball the same argument uses $\rho_j^2=\rho_{t-1}^2+i(\lambda_0^2-\lambda_1^2)\le\rho_{t-1}^2+i\lambda_0^2\le(1+i)\rho_{t-1}^2$, so $\rho_j/s_j\le\sqrt{1+i}\,\rho_{t-1}/(s_{t-1}+i)\le\rho_{t-1}/(\sqrt2\,s_{t-1})$ because $s_{t-1}\le1$; the classes above $t$ satisfy $\rho_j/s_j\le\rho_t/(2s_t)$. These strict factors absorb the $1+o(1)$ difference between $\rho_z$ and $\rho_{\abs z}$.

\smallskip\noindent\emph{(a)} Equalizing $h_{t-1}/(\theta-t+1)=h_t/(t-\theta)$ gives $\theta-t+1=h_{t-1}/(h_{t-1}+h_t)$ and the common value $h_{t-1}+h_t=(2t-1)\lambda_1+(2r-2t+1)\lambda_0$; any other bias increases one of the two ratios. Hence $\Vbox=((2t-1)\lambda_1+(2r-2t+1)\lambda_0)^2$ exactly. For the half-slack bias $\theta_d=t-\tfrac12+\tfrac12\sigd(\lhat_0-\lhat_1)$ put $\hat h_j=j\lhat_1+(r-j)\lhat_0=h_j(1+o(1))$; one checks $s_{t-1}-\sigd\hat h_{t-1}=s_t-\sigd\hat h_t=\tfrac12(1-\sigd(\hat h_{t-1}+\hat h_t))$. If $d\ge(1+\varepsilon)k\Vbox$, then $\sigd(\hat h_{t-1}+\hat h_t)\le(1+\varepsilon)^{-1/2}(1+o(1))$, so both binding classes, and by the first step all classes, satisfy $s_j\ge(1+\eta_\varepsilon)\sigd h_j$ for some $\eta_\varepsilon>0$ and all large $k$; that is, $\Dbox(\one,\theta_d)\le d/(1+\eta_\varepsilon)^2$, and \cref{thm:box}(i) applies. The same argument applies with $\lambda_p$ in place of $\lhat_p$. The converse is \cref{thm:box}(ii).

\smallskip\noindent\emph{(b)} At $\theta=t-\tfrac12$ both binding margins equal $\tfrac12$ and $h_{t-1}\ge h_t$, so $\Dbox=k(2h_{t-1})^2$, and \cref{thm:box}(i) gives the threshold.

\smallskip\noindent\emph{(c)} At the midpoint the binding ball ratio is $2\sqrt r\rho_{t-1}$, so $\Dball=4rk\rho_{t-1}^2(1+o(1))$ and \cref{thm:ball} gives the threshold.

\smallskip\noindent\emph{(d)} Equalizing $\sqrt r\rho_{t-1}/(\theta-t+1)=\sqrt r\rho_t/(t-\theta)$ gives $\theta-t+1=\rho_{t-1}/(\rho_{t-1}+\rho_t)$ and $\Vball=r(\rho_{t-1}+\rho_t)^2(1+o(1))$. \Cref{thm:ball} at this bias and \cref{thm:weighted} give achievability and the gate-local converse. \hfill$\square$

\smallskip
For $\beta\to1$ every $h_j$ equals $r\lambda_1(1+o(1))$ and every $\rho_j$ equals $\sqrt r\lambda_1(1+o(1))$, which gives the first statement of \cref{cor:markov}; dividing (a)--(d) by $\lambda_0^2$ and letting $\beta\to\infty$ gives the second. For \cref{cor:recovery}, $r=t=1$ gives $h_0+h_1=\rho_0+\rho_1=\lambda_0+\lambda_1$, and the ratio $((2r-1+\beta)/(1+\beta))^2$ has derivative of the sign of $2-2r\le0$ in $\beta$.

\subsection{Unweighted margins of symmetric gates}
For a representation $(a,\theta)$ write $\tau=\min_zs_z(a,\theta)$ and $\kappa=\tau^2/(2\norm a^2)$, the \emph{unweighted margin rate}; its supremum $\kappa^*(g)$ over representations is the price of $g$ when all classes are equally numerous. For $\THR^r_t$, the representation $a=\one$, $\theta=t-\tfrac12$ has $\tau=\tfrac12$ and $\kappa=1/(8r)$. It is optimal: for each position $j$ choose a pattern of weight $t-1$ that excludes $j$; turning $j$ on changes the output, so $a_j=s_z+s_{z+e_j}\ge2\tau$, and $\norm a^2\ge4r\tau^2$. Thus conjunction, disjunction and majority have the same unweighted rate $1/(8r)$; their weighted prices differ because their class sizes differ.

\subsection{Comparators}
The comparator at threshold $T\in\{1,\dots,2^r-1\}$ is $\mathrm{COMP}^T_r(z)=\one\{\sum_{l=1}^r2^{l-1}z_l\ge T\}$. Its unweighted rate depends on $T$, because $T$ selects a Boolean function and different functions have different best representations.

\emph{Binary weights realize every threshold.} With $a=(1,2,\dots,2^{r-1})$, $\norm a^2=(4^r-1)/3$. The subset sums of distinct powers of two are all integers $0,\dots,2^r-1$, so the bias $T-\tfrac12$ has $\tau=\tfrac12$, and
\begin{equation}\label{sm:eq:cmp-rate}
\kappa=\frac{(1/2)^2}{2(4^r-1)/3}=\frac{3}{8(4^r-1)}=\Theta(4^{-r}).
\end{equation}
This is the rate of a representation, not of the gate: the dictator $T=2^{r-1}$ is the function $z_r$ of one input, with $\kappa^*=1/8$, and for the hardest thresholds the best representation is not binary.

\emph{The hardest thresholds have Fibonacci weights.} Let $F_1=F_2=1$ and $F_{i+1}=F_i+F_{i-1}$. The alternating threshold $T_r=(2^{r+1}+(-1)^r)/3$ and its Boolean dual attain
\begin{equation}\label{sm:eq:cmp-fib}
\kappa^*(\mathrm{COMP}^{T_r}_r)=\frac1{8F_rF_{r+1}}=\Theta\Bigl(\Bigl(\frac{3+\sqrt5}2\Bigr)^{-r}\Bigr),
\end{equation}
The representation $a_i=2F_i$ has minimum margin one and $\norm a^2=4\sum_iF_i^2=4F_rF_{r+1}$. The following induction proves its optimality and that no other comparator threshold has a smaller optimal rate.

\emph{Every comparator has a bounded Fibonacci representation.} Trailing zero bits of $T$ make the lowest inputs irrelevant, so assume the lowest bit of $T$ is one. Starting from the least significant one bit, the comparator is a nested monotone chain $h_1=z_1$, $h_i=z_i\circ_ih_{i-1}$, where $\circ_i$ is AND for a one bit of $T$ and OR for a zero bit. Let $P_i$ be the minimum accepted weighted sum of the prefix, $Q_i$ the maximum rejected sum, and $A_i$ the sum of the prefix weights; maintain $P_i-Q_i=1$ and write $D_i=A_i-P_i$, $E_i=Q_i$. Start with $a_1=1$, $D_1=E_1=0$. An AND step sets $a_i=D_{i-1}+1$, after which $D_i=D_{i-1}$ and $E_i=A_{i-1}$; an OR step sets $a_i=E_{i-1}+1$, after which $D_i=A_{i-1}$ and $E_i=E_{i-1}$. These identities follow from the two values of $z_i$: AND has $P_i=a_i+P_{i-1}$ and $Q_i=\max\{A_{i-1},a_i+Q_{i-1}\}$, and OR has $P_i=\min\{a_i,P_{i-1}\}$ and $Q_i=Q_{i-1}$. In the coordinates $U=D+1$, $V=E+1$ each step is $(U,V)\mapsto(U,U+V)$ or $(U+V,V)$, and the new weight is the unchanged coordinate. By induction the two sorted coordinates after $i-1$ steps are at most $(F_i,F_{i+1})$; before the $i$th weight is assigned there have been $i-2$ updates, so that weight is at most $F_i$. The bias $(P_r+Q_r)/2$ has margin $\tfrac12$, giving rate at least $1/(8\sum_iF_i^2)=1/(8F_rF_{r+1})$ for every threshold; discarded irrelevant inputs only decrease the norm.

\emph{No representation beats the alternating chain.} Take any representation of an alternating chain with minimum margin $\tau>0$. Every input is essential and monotone, so fixing the other inputs across an output-changing edge shows that its weight is at least $2\tau$. Fixing the higher inputs to their pass-through values, each prefix inherits a separating threshold with gap at least $2\tau$. With $P_i,Q_i,A_i,D_i,E_i$ as above, an AND step forces $a_i\ge D_{i-1}+2\tau$ and $E_i\ge E_{i-1}+a_i$ while $D_i=D_{i-1}$; an OR step forces $a_i\ge E_{i-1}+2\tau$ and $D_i\ge D_{i-1}+a_i$ while $E_i=E_{i-1}$. Let $\psi_i$ be the driver $D_{i-1}$ or $E_{i-1}$ used at step $i$. Alternation gives $\psi_2=0$, $\psi_3\ge2\tau$ and $\psi_i\ge\psi_{i-2}+a_{i-1}$ for $i\ge4$, and simultaneous induction yields $\psi_i\ge2\tau(F_i-1)$ and $a_i\ge2\tau F_i$. Hence $\norm a^2\ge4\tau^2F_rF_{r+1}$, which is the reverse bound in \eqref{sm:eq:cmp-fib}. The binary rate $\Theta(4^{-r})$ is therefore only achievable, and it overstates the dimension cost of the hardest comparators by a factor $\Theta((6-2\sqrt5)^r)$; the base $(3+\sqrt5)/2$ is the square of the golden ratio.

These constructions belong to the classical theme of weight geometry in threshold logic \citep{Muroga71,Hastad94}; here they show that for an asymmetric gate the choice of representation, not only of the bias, enters the price. Exhaustive finite checks are recorded in \cref{sm:numerics}.

%% file: sm_window.tex
\section{The critical window}\label{sm:window}

Throughout this section $\cF$ is the universal family of pairwise conjunctions, $k\to\infty$, $n=m-k$ and $n/k\to\infty$. We use the representation of \cref{thm:exact}, write $\alpha_i=A_i-1$ for the centered active reads, and use the quantile levels $\lhat_1=\Phibar^{-1}(1/k)$ and $\lhat_0=\Phibar^{-1}(1/n)$ with $\beta=\sqrt{\log n/\log k}$. By Mills' ratio, $\Phibar^{-1}(1/q)^2=2\log q-\log(4\pi\log q)+o(1)$ as $q\to\infty$; applied to $q=k$ and $q=n$ this gives $\lhat_1=\lambda_1(1+o(1))$ and $\lhat_0/\lhat_1=\beta(1+o(1))$.

\subsection{Normalization}
For independent standard Gaussians, Mills' ratio gives the classical normalization of maxima \citep{LeadbetterLindgrenRootzen1983}:
\[
k\Phibar(\lhat_1+y/\lhat_1)\to e^{-y},\qquad n\Phibar(\lhat_0+y/\lhat_0)\to e^{-y},
\]
so $\lhat_1(\max_iZ_i-\lhat_1)\Rightarrow G_A$ and $\lhat_0(\max_jW_j-\lhat_0)\Rightarrow G_B$, independent standard Gumbel variables. On either window of \cref{thm:window}, $d\asymp k(\lhat_1+\lhat_0)^2$, and the exact representation contributes only negligible perturbations in units of $\sigd/\lhat_1$:
\begin{equation}\label{sm:eq:window-errors}
\frac{\lhat_1(H-1)}{\sigd}=O_\Prob\Bigl(\frac{\lhat_1}{\sqrt k}\Bigr),\qquad\lhat_1\bar Z=O_\Prob\Bigl(\frac{\lhat_1}{\sqrt k}\Bigr),\qquad\lhat_1(\lhat_1+\lhat_0)(\sqrt H-1)=O_\Prob\Bigl(\frac{\lhat_1}{\sqrt k}\Bigr).
\end{equation}
The centers \eqref{eq:centers} use the scale $\sigd=\sqrt{k/d}$ of the Gaussian part $\sigma Z_i$ of a read; using the variance $(k+1)/d$ of an active read instead changes them by $O(\lhat_1^2/k)=o(1)$ in these units. Since $\lhat_1^2\sim2\log k$, all these errors vanish without any restriction on $\log n/k$, even when $\lhat_0/\lhat_1\to\infty$.

\subsection{The midpoint profile}
The one-active class succeeds when $\max_iA_i+\max_jB_j\le3/2$, that is, when
\[
\max_iZ_i+\max_jW_j\le\frac{1/2-(H-1)}{\sigd\sqrt H}+\bar Z .
\]
At $d=d_0+wx$ with $w=8k(\lhat_1+\lhat_0)/\lhat_1=8k(1+\beta)(1+o(1))$, the right side equals $\lhat_1+\lhat_0+x/\lhat_1+o_\Prob(1/\lhat_1)$ by \eqref{sm:eq:window-errors}. Multiplying the slack by $\lhat_1$, the success event of this class becomes
\[
G_A+\frac{\lhat_1}{\lhat_0}G_B\le x+o_\Prob(1).
\]
For $\beta\to\beta_0<\infty$ its probability tends to $\Prob(G_A+G_B/\beta_0\le x)=\Exp\exp(-e^{\beta_0(G_A-x)})$; for $\beta\to\infty$ the second fluctuation vanishes and the limit is the Gumbel distribution function. The convergence is uniform on compact $x$-intervals, because the limits are continuous and the errors in \eqref{sm:eq:window-errors} do not depend on $x$.

The two other classes do not enter. The smallest both-active score is $2-2\sigd\lhat_1+O_\Prob(\sigd/\lhat_1)$, while the midpoint balance is $\sigd(\lhat_1+\lhat_0)=\tfrac12+O(\sigd/\lhat_1)$; so the extra slack of this class, in units of $\sigd/\lhat_1$, is of order $\lhat_1(\lhat_0-\lhat_1)$, which tends to infinity because $\lhat_0^2-\lhat_1^2\sim2\log(n/k)\to\infty$ when $\lhat_0/\lhat_1\to1$, and trivially otherwise. The largest zero-active score is $2\sigd\lhat_0+o_\Prob(1)\le1+o_\Prob(1)<3/2$. This proves \cref{thm:window}(i).

The change of variables $\tau=e^{-G_A}$ gives $\Psi_\beta(x)=\int_0^\infty\exp(-\tau-\tau^{-\beta}e^{-\beta x})\,d\tau$. At $\beta=1$ the identity $\int_0^\infty e^{-\tau-c/\tau}\,d\tau=2\sqrt cK_1(2\sqrt c)$ gives the Bessel form, dominated convergence gives the Gumbel limit as $\beta\to\infty$, and minimizing $\tau+\tau^{-\beta}e^{-\beta x}$ shows $\log\Psi_\beta(x)\asymp-e^{-x\beta/(1+\beta)}$ as $x\to-\infty$ for fixed $\beta$.

\subsection{Conditional firing counts}
Given the active columns, the inactive reads are independent $N(0,\sigma^2)$ variables. Let $X_{\rm abs}$ be the number of inactive features that form at least one wrong one-active gate at the midpoint. Then, with $M=\max_i\alpha_i$,
\[
X_{\rm abs}\mid E_S\sim\Bin(n,q_S),\qquad q_S=\Phibar\bigl((1/2-M)/\sigma\bigr),\qquad\Lambda=nq_S .
\]
An inactive feature can create several gate errors, so $X_{\rm abs}\le X$ and in general $X_{\rm abs}\ne X$.

\begin{proposition}[Poisson approximation]\label{sm:prop:poisson}
Along window sequences with $x=(d-d_0)/w\to x_0$ and $\beta\to\beta_0\in[1,\infty)$,
\[
\Exp\,d_{TV}\bigl(\mathcal L(X_{\rm abs}\mid E_S),\ \mathrm{Poisson}(\Lambda)\bigr)\le\Exp\min\{1,nq_S^2\}\to0,
\]
and $\Lambda\Rightarrow e^{\beta_0(G_A-x_0)}$.
\end{proposition}

\begin{proof}
The bound is the binomial--Poisson estimate of Le Cam type \citep{ArratiaGoldsteinGordon1989,BarbourHolstJanson1992}. By the computation of the previous subsection, $(1/2-M)/\sigma=\lhat_0+\beta_0(x_0-G_A)/\lhat_0+o_\Prob(1/\lhat_0)$, so $\Lambda=n\Phibar\bigl(\lhat_0+\beta_0(x_0-G_A)/\lhat_0+o_\Prob(1/\lhat_0)\bigr)\Rightarrow e^{\beta_0(G_A-x_0)}$. Since $\Lambda$ is tight and $n\to\infty$, $nq_S^2=\Lambda^2/n\to0$ in probability, and bounded convergence gives the limit of the expectation.
\end{proof}

The limit intensity $e^{\beta_0(G_A-x_0)}$ has infinite mean for $\beta_0\ge1$, which is the statement used at the end of \cref{sec:window}.

\subsection{Two-sided extremes}
Standardize the active extremes by $T(y)=\Phibar^{-1}(e^{-y}/k)$: let $\max_i\alpha_i=\sigd T(Y^+)$, and let the two smallest be $-\sigd T(Y_1)$ and $-\sigd T(Y_2)$ with $Y_1\ge Y_2$.

\begin{lemma}[Asymptotic independence of the two tails]\label{sm:lem:twosided}
Jointly, $(Y^+,Y_1,Y_2)\Rightarrow(G_A,-\log\omega_1,-\log(\omega_1+\omega_2))$, where $G_A$ is standard Gumbel, $\omega_1,\omega_2$ are unit exponentials, and the three are independent; $\omega_1<\omega_1+\omega_2$ are the first two points of a unit Poisson process. In particular $Y_1+Y_2\Rightarrow-\log[\omega_1(\omega_1+\omega_2)]$, whose distribution function is
\begin{equation}\label{sm:eq:Fplus}
F_+(x)=e^{-\sqrt c}+\int_0^{\sqrt c}e^{-c/s}\,ds,\qquad c=e^{-x},
\end{equation}
with right tail $1-F_+(x)=(\tfrac12+o(1))\,x\,e^{-x}$ as $x\to\infty$.
\end{lemma}

\begin{proof}
For the independent Gaussians $(Z_i)$ of \cref{thm:exact}, partition the upper and lower tails into finitely many disjoint extreme-level intervals. An observation belongs to at most one interval, and the interval probabilities are $(c_j+o(1))/k$ by Mills' ratio. The joint probability generating function of the counts is $[1+\sum_j(c_j+o(1))(t_j-1)/k]^k\to\exp(\sum_jc_j(t_j-1))$, so the upper and lower point processes converge jointly to independent Poisson processes with intensity $e^{-y}\,dy$. Their first upper point and first two lower points have the stated laws. The perturbations \eqref{sm:eq:window-errors} are $o_\Prob(1)$ on both extreme scales, so the convergence transfers to $(\alpha_i)$, and $T(y)=\lhat_1+y/\lhat_1+o(1/\lhat_1)$ uniformly on bounded $y$-intervals. Conditionally on $\omega_1=s$, the event $s(s+\omega_2)\ge c$ has probability $e^{-(c/s-s)}$ for $s<\sqrt c$ and one otherwise, which integrates to \eqref{sm:eq:Fplus}. Finally $1-F_+(x)=1-e^{-\sqrt c}-\sqrt c+\int_0^{\sqrt c}(1-e^{-c/s})\,ds$; the first three terms are $O(c)$, and splitting the integral at $s=c$ shows that it is $\tfrac12c\log(1/c)+O(c)$. Finite-sample upper and lower extremes can be dependent; only their limits are independent.
\end{proof}

\subsection{The optimal common bias}
At the optimal bias two classes bind, and their extremes come from opposite tails of the same active sample; \cref{sm:lem:twosided} decouples them.

\begin{proof}[Proof of \cref{thm:window}(ii)]
Use the slack coordinates
\[
x_-=\frac{\lhat_1}\sigd\bigl(\theta-1-\sigd(\lhat_1+\lhat_0)\bigr),\qquad x_+=\frac{\lhat_1}\sigd\bigl(2-\theta-2\sigd\lhat_1\bigr),
\]
for the one-active and both-active classes. For bounded $(x_-,x_+)$, \cref{sm:lem:twosided} and the previous subsections give the joint success limit $\Psi_\beta(x_-)F_+(x_+)$: the upper active extreme and the bottom pair are independent in the limit, and the inactive sample is independent of both. Their total slack is $x_-+x_+=\lhat_1(\sqrt{d/k}-3\lhat_1-\lhat_0)=x+o(1)$ at $d=d_0^{\rm opt}+2k(3+\beta)x$. The zero-active class is negligible by \cref{sm:lem:neither}. Maximizing over the deterministic bias is therefore the optimization over $y=x_-$. To pass the supremum through the limit, first restrict $x_-$ to a compact interval: if $x_-\to-\infty$ the one-active class fails with probability tending to one, and if $x_-\to+\infty$ with bounded total slack the both-active class does. Tightness of the limiting extremes makes these bounds uniform, and on compact intervals the continuous joint distribution functions converge uniformly.
\end{proof}

\begin{lemma}[The zero-active class stays slack]\label{sm:lem:neither}
On the optimal window, for biases with bounded slack coordinates, $\Prob(B_{(1)}+B_{(2)}>\theta)\to0$, where $B_{(1)}\ge B_{(2)}$ are the two largest inactive reads.
\end{lemma}

\begin{proof}
At the center, $1=\sigd(3\lhat_1+\lhat_0)+o(\sigd/\lhat_1)$ and $\theta=1+\sigd(\lhat_1+\lhat_0)+O(\sigd/\lhat_1)$, so $\theta-2\sigd\lhat_0=4\sigd\lhat_1+O(\sigd/\lhat_1)$ uniformly on compact windows. The two largest inactive reads sum to $2\sigd\lhat_0+O_\Prob(\sigd/\lhat_0)+o_\Prob(\sigd/\lhat_1)$. The gap is of order $\lhat_1^2\to\infty$ in units of $\sigd/\lhat_1$, also when $\lhat_0/\lhat_1\to\infty$.
\end{proof}

The half-slack bias of \cref{thm:thr}(a) is $\theta=\tfrac32+\tfrac12\sigd(\lhat_0-\lhat_1)$, which gives $x_-=x_+=x/2$ and hence $\Psi_\beta(x/2)F_+(x/2)$; this need not equal the supremum. With $\lambda_p$ in place of $\lhat_p$ the split becomes $x_--x_+=\lhat_1[(\lambda_0-\lambda_1)-(\lhat_0-\lhat_1)]=(\log\beta)/\beta-\tfrac12(1-1/\beta)\log(4\pi\log k)+o(1)$, which tends to $-\infty$ when $\beta\to\beta_0>1$; that bias has the same first-order threshold but not this profile. The optimal window is narrower than the midpoint window by $w/w_{\rm opt}=4(1+\beta)/(3+\beta)\in[2,4)$ for every $\beta\ge1$. The profile concerns a deterministic common bias; it does not optimize gate-specific weights, and it does not determine the law of the full count $X$, which would require the number of active features attached to each firing inactive feature.

%% file: extensions.tex
\section{Extensions of the model}\label{sm:extensions}

The box law has three ingredients: the support is random and independent of the dictionary, the dictionary is Gaussian, and one layer of gates reads a stored input. We change each in turn, and then drop the linear dictionary altogether. Proofs and the precise statements of \cref{sec:composition} are in \cref{sm:beyond}.

\subsection{Adversarial supports}
If the support may be chosen after the dictionary, the cloud of a class is no longer a product of two samples: an adversary can place $k-1$ inactive columns that all align with one probe, and the interference grows linearly in $k$ instead of as $\sqrt k$.

The statement is in \cref{prop:uniform}.

For recovery by one threshold layer, \citet[Theorem~12]{GKP26} prove the analogous $\Omega(k^2\log(m/k)/\log k)$ bound for arbitrary linear encoders when $k<\sqrt m$, and for Gaussian dictionaries \citet[Theorem~8 and Proposition~6]{IvanitskiyJKM26} place the largest sparsity that is recoverable on every support between the orders $(d/\log m)^{1/2}$ and $d^{1/2}$. The factor $k$ between typical and uniform correctness is the price of letting the support depend on the dictionary; random and worst-case supports differ in the same way for sparse regression \citep{CandesPlan09}.

\subsection{Other entry laws}
The ball needs only one-gate tails, and for the midpoint pairwise network these are universal. Let the entries of $\sqrt d\,E$ be independent with a fixed centered, unit-variance, sub-Gaussian law, and let $q_{k,d}$ be the probability that a midpoint pairwise read with one active input exceeds $3/2$.

The statement is in \cref{thm:universality}.

The proof tilts independent coordinate rows exponentially. \Cref{sm:universality} also computes the next term, $\log(q_{k,d}/\Phibar(x_K))=c_\nu d/k^2+\cdots$ with $x_K=\sqrt{d/[8(k-1)]}$, in which the entry law enters through its skewness and excess kurtosis, and treats sparse Rademacher entries. The theorem concerns the midpoint pairwise network. The box is different: it rests on the exact joint law of \cref{thm:exact}, which has no counterpart for general entries, and whether the reliability threshold is universal is open (\cref{sec:discussion}).

\subsection{Composition and unrestricted encodings}
Writing outputs for a later layer costs dimensions of its own: pairwise conjunction has $\binom k2$ active outputs. For feedforward threshold circuits in which every wire has its own Gaussian direction and every active output is written once into a shared stream, \cref{thm:writebudget} shows that $d\ge(1+\varepsilon)\kappa_{\min}^{-1}K\log Q$ suffices, where $Q$ is the number of reads, $K$ the largest number of active directions in the stream before a read and $2\kappa_{\min}$ the smallest squared clean margin of a unit-norm read. Its proof is a union bound over an ideal execution, that is, a ball argument. Without the linear dictionary, order $k$ dimensions suffice for all pairwise conjunctions and, when $\log k=o(\log m)$, are necessary for polynomially many threshold units: counting sign patterns \citep{Cover65} gives $d\ge\bigl(k\log(m/k)+\log(1-\delta)\bigr)/(1+\log T)$ for $T$ units correct on a fraction $1-\delta$ of supports (\cref{prop:counting}), and the coefficients of $\prod_{l\in S}(y-l)^2$, the polynomial behind the neighborliness of cyclic polytopes \citep[Theorem~15.2]{Barvinok2010}, form a $2k$-dimensional code with large coefficients and small normalized margins (\cref{prop:polycode}). The $\Theta(k\log m)$ scale of this paper therefore belongs to accessible computation from a random linear dictionary rather than to the task itself.

%% file: sm_beyond.tex
\section{Statements and proofs for Section~\ref{sec:beyond}}\label{sm:beyond}

\subsection{Uniform correctness}\label{sm:uniform}
\begin{proof}[Proof of \cref{prop:uniform}]
For the upper bound, chi-square tail bounds and a union bound give $\max_i\abs{\norm{e_i}^2-1}\le\delta$ with probability tending to one for fixed small $\delta>0$, since $d\ge Ck^2\log m$. Conditionally on one column and its norm bound, the inner product with another column is Gaussian, so a second union bound gives $\max_{i\ne j}\abs{\ip{e_i}{e_j}}\le C_0\sqrt{\log m/d}$. For any support of size at most $k$, the score of any gate differs from its clean value $0$, $1$ or $2$ by at most $2\delta+2kC_0\sqrt{\log m/d}$. Choosing $\delta<1/8$ and then $C$ large makes this less than $1/2$ on one dictionary event that controls every support and every gate.

For the lower bound fix the gate $(1,2)$ and let $w=e_1+e_2$. As $d\to\infty$, $\norm w^2\to2$ and $\ip w{e_1}\to1$ in probability. Conditionally on these two columns, the alignments $\ip w{e_j}$, $j\ge3$, are independent $N(0,\norm w^2/d)$ variables. At the level $(1-\delta)\sqrt{2\log(m/k)}$ in standard units, Mills' bound makes the expected number of exceedances among the $m-2$ alignments much larger than $k$, and Chernoff's inequality makes fewer than $k-1$ exceedances unlikely. Let the support consist of feature $1$ and the $k-1$ best aligned features other than feature $2$. The gate $(1,2)$ then has one active input and score at least
\[
1+o_\Prob(1)+(1-o_\Prob(1))(k-1)\norm w\sqrt{2\log(m/k)/d}.
\]
Since $\log(m/k)\sim\log m$, a sufficiently small $c$ makes this exceed $3/2$ with probability tending to one when $d\le ck^2\log m$. The support is allowed to depend on the dictionary because the claim is uniform over supports.
\end{proof}

The upper bound is an existence statement for the canonical network and the lower bound concerns canonical probes; neither is a statement about all decoders.

\subsection{Other entry laws}\label{sm:universality}
Let the columns be $e_i=y_i/\sqrt d$, where all entries of all $y_i$ are independent with a fixed centered, variance-one, sub-Gaussian law $\nu$, and write $\gamma_1=\Exp Z^3$ and $\gamma_2=\Exp Z^4-3$ for the skewness and excess kurtosis of $Z\sim\nu$; in this subsection $Z$ and $Z_l$ denote entries, not the active sample of \cref{thm:exact}. Examples are Gaussian $(0,0)$, Rademacher $(0,-2)$ and uniform on $[-\sqrt3,\sqrt3]$ $(0,-6/5)$. Constants written $O_\nu$ may depend on $\nu$. Let $q_{k,d}$ be the probability that a midpoint pairwise read with one active and one inactive input exceeds $3/2$; it includes the fluctuation of the signal, not only the centered interference.

\begin{proposition}[Second-order term]\label{sm:prop:ledger}
For a fixed law $\nu$ and $d/k\to\infty$, put $x_K=\sqrt{d/[8(k-1)]}$. Then
\[
\log\frac{q_{k,d}}{\Phibar(x_K)}=c_\nu\frac d{k^2}+O_\nu\Bigl(\frac d{k^3}+\frac1{\sqrt k}+\frac kd\Bigr),\qquad c_\nu=\frac{132+41\gamma_2}{1024}+\frac{\gamma_1^2}{192}.
\]
In particular $q_{k,d}/\Phibar(x_K)\to1$ when $k\ll d\ll k^2$, and $c_\nu\ge50/1024$, with equality for Rademacher entries.
\end{proposition}

The correction dominates the remainder when $d\gg k^{3/2}$. Even for Gaussian entries $c_\nu=132/1024\ne0$: the term corrects the Gaussian tail $\Phibar(x_K)$ of a read, and the entry law enters only through $\gamma_1^2$ and $\gamma_2$, the skewness appearing squared by invariance under a global sign change of the entries.

\begin{corollary}[Sparse entries]\label{sm:cor:sparse}
Let $Z=\epsilon\chi/\sqrt q$ with $\chi\sim\mathrm{Ber}(q)$ and an independent Rademacher sign $\epsilon$, where $q=q(k)\in(0,1]$. If $qk\to\infty$ and $d/k\to\infty$, then $\log q_{k,d}=-\frac d{16k}[1+O((qk)^{-1}+(k/d)\log(d/k))]$. If also $d/(qk^2)\to0$, then $q_{k,d}/\Phibar(x_K)\to1$.
\end{corollary}

At $d=\Theta(k\log(km))$ the first condition corresponds to $dq/\log m\to\infty$ expected nonzeros per column, and the second to $dq/\log^2m\to\infty$; both are sufficient conditions, not sparsity thresholds.

\paragraph{An independent-row proof}
Let $A,B,Z_1,\dots,Z_{k-1}$ be independent with law $\nu$, $S=\sum_{l<k}Z_l$, and
\[
T=(A+B)(A+S),\qquad C=A+B,\qquad P=A^2+AB-1 .
\]
The one-active score is $d^{-1}\sum_{h=1}^dT_h$ with independent copies $T_h$, and $T-1=P+CS$. Write $K_k(t)=\log\Exp e^{t(T-1)}$.

\begin{lemma}[Row cumulants and tilting]\label{sm:lem:tilt}
For a fixed law $\nu$, uniformly on $\abs t\le C_0/k$ for any fixed $C_0$,
\begin{align*}
K_k(t)&=\frac{2k+\gamma_2+1}2t^2+\frac{k(12+3\gamma_2+2\gamma_1^2)+O_\nu(1)}6t^3\\
&\quad+\frac{k^2(24+6\gamma_2)+O_\nu(k)}{24}t^4+O_\nu(k^{-3}).
\end{align*}
The rate at displacement $1/2$ is
\[
I_k=\sup_t\{t/2-K_k(t)\}=\frac1{16k}-\Bigl(\frac{68+41\gamma_2}{1024}+\frac{\gamma_1^2}{192}\Bigr)\frac1{k^2}+O_\nu(k^{-3}).
\]
If $t_k$ solves $K_k'(t_k)=1/2$ and $\eta_k=t_k\sqrt{dK_k''(t_k)}$, then
\[
\log q_{k,d}=-dI_k-\log\eta_k-\tfrac12\log(2\pi)+O_\nu(k/d+k^{-1/2}),
\]
with $\eta_k=\sqrt{d/(8k)}\,[1+O_\nu(k^{-1})]$.
\end{lemma}

\begin{proof}
The variables $A$, $B$ and $S/\sqrt k$ have uniformly bounded sub-Gaussian norms, so $(T-1)/\sqrt k$ has a uniform exponential moment near zero; this also bounds its tilted absolute third moment uniformly for $t=O(1/k)$. Expanding $P+CS$ with $\Exp S^2=k-1$, $\Exp S^3=(k-1)\gamma_1$ and $\Exp S^4=3(k-1)^2+(k-1)\gamma_2$ gives
\begin{gather*}
\kappa_2(T)=2k+\gamma_2+1,\qquad\kappa_3(T)=k(12+3\gamma_2+2\gamma_1^2)+O_\nu(1),\\
\kappa_4(T)=k^2(24+6\gamma_2)+O_\nu(k).
\end{gather*}
The fifth cumulant is $O_\nu(k^2)$, since odd moments of $S$ through order five are $O_\nu(k^2)$ and the remaining terms contain at most four copies of $S$. The exponential-moment bound controls the Taylor remainder from order six by $O_\nu(k^3t^6)$ on this interval, with the same bounds for the derivatives. Since $K_k''(t)=2k+O_\nu(1)$, the saddle point is $t_k=1/(4k)+O_\nu(k^{-2})$, and substituting $1/(4k)$ into $t/2-K_k(t)$ costs only $O_\nu(k^{-3})$, which gives $I_k$.

Under the tilted measure with density $\exp(t_k\sum_h(T_h-1)-dK_k(t_k))$, let $U$ be the centered sum divided by $\sqrt{dK_k''(t_k)}$. The Berry--Esseen theorem gives Kolmogorov distance $O_\nu(d^{-1/2})$ from a standard Gaussian. The function $e^{-\eta_ku}\one\{u>0\}$ has total variation at most two, so
\[
q_{k,d}=e^{-dI_k}\bigl[e^{\eta_k^2/2}\Phibar(\eta_k)+O_\nu(d^{-1/2})\bigr];
\]
this applies to lattice laws as well, with the strict threshold. Since $e^{\eta^2/2}\Phibar(\eta)=(\sqrt{2\pi}\eta)^{-1}[1+O(\eta^{-2})]$, the relative Berry--Esseen error is $O_\nu(\eta_k/\sqrt d)=O_\nu(k^{-1/2})$. This is the exponential-tilting argument behind the Cram\'er--Petrov theory \citep{PetrovBook75,DJS22}.
\end{proof}

\begin{proof}[Proof of \cref{thm:universality}]
By \cref{sm:lem:tilt}, the leading terms give the exponent. For \cref{sm:prop:ledger}, use $\log\Phibar(x_K)=-d/[16(k-1)]-\log x_K-\tfrac12\log(2\pi)+O(k/d)$, $\log(\eta_k/x_K)=O_\nu(k^{-1})$ and $d/[16(k-1)]=d/(16k)+d/(16k^2)+O(d/k^3)$; the extra $64/1024$ turns $68+41\gamma_2$ into $132+41\gamma_2$. Pearson's moment inequality $\gamma_2\ge\gamma_1^2-2$, from the positive semidefiniteness of the covariance matrix of $(Z,Z^2)$, gives $c_\nu\ge50/1024$. For the expected count of the midpoint network, the both-active lower tail has the same leading exponent and the zero-active upper tail has exponent $(9+o(1))d/(16k)$ by the same row argument; union bounds over $O(km)$ and $O(m^2)$ gates give the sufficient dimension $(16+\varepsilon)k\log(km)$.
\end{proof}

\begin{proof}[Proof of \cref{sm:cor:sparse}]
For the sparse law the row moment generating function has the exact finite form
\[
\Exp e^{tT}=\Exp_{A,B}\Bigl[e^{t(A+B)A}\bigl\{1-q+q\cosh\bigl(t(A+B)/\sqrt q\bigr)\bigr\}^{k-1}\Bigr],
\]
with $A,B\in\{0,\pm q^{-1/2}\}$ of probabilities $1-q$, $q/2$, $q/2$. For $\abs t\le C_0/k$ every exponential or hyperbolic argument is $O((qk)^{-1})$, and Taylor expansion inside the braces and then over the nine cases gives, uniformly, $K_k(t)=kt^2[1+O((qk)^{-1})]$, $K_k'(t)=2kt[1+O((qk)^{-1})]$ and $K_k''(t)=2k[1+O((qk)^{-1})]$. The leading $kt^2$ comes from $\Exp(A+B)^2=2$ and $\Exp S^2=k-1$; the remaining quadratic term is $O(t^2/q)$. Hence $t_k=(4k)^{-1}[1+O((qk)^{-1})]$ and $I_k=(16k)^{-1}[1+O((qk)^{-1})]$. The standardized tilted third absolute moment is $O(q^{-1/2})$: $\Exp\abs{A+B}^3=O(q^{-1/2})$, $\Exp\abs S^3=O(k^{3/2}+kq^{-1/2})$ and $\Exp\abs P^3=O(q^{-2})$, divided by $k^{3/2}$ with $qk\to\infty$. The Berry--Esseen step therefore has relative error $O((qk)^{-1/2})$, and
\[
\log q_{k,d}=-dI_k-\log\sqrt{d/(8k)}-\tfrac12\log(2\pi)+O\bigl((qk)^{-1/2}+k/d\bigr),
\]
which proves the first claim; comparison with Mills' formula gives $\log(q_{k,d}/\Phibar(x_K))=O(d/(qk^2)+(qk)^{-1/2}+k/d)$ and the second.
\end{proof}

These results concern one gate and hence the ball. They do not extend the reliability thresholds of \cref{thm:box} or the window profiles to non-Gaussian dictionaries, because the exact law of \cref{thm:exact} is Gaussian. For correlated columns, isotropy and marginal coherence tails alone do not provide the joint moment control used here, and below $qk\asymp1$ the sparse expansion no longer applies.

\subsection{Composition}\label{sm:composition}
\paragraph{Reading a superposed output vector}
Let $N=\binom m2$, $K'=\binom k2$ and $J(S)$ the set of active pairs. Draw independent output directions $f_\gamma\sim N(0,I_p/p)$ for the $N$ gates, independent of the input dictionary, and let $v^*=\sum_{\gamma\in J(S)}f_\gamma$ be the ideal output stream. The set $J(S)$ is independent of the output dictionary, although its indicators are dependent across gates.

\begin{proposition}[Output channel]\label{sm:prop:output}
Suppose $K'\to\infty$, $N/K'\to\infty$, $\log N=o(K')$ and $p/K'\to\infty$. The matched-filter reads $\one\{\ip{f_\gamma}{v^*}>1/2\}$ have a vanishing expected number of errors if $p\ge(8+\varepsilon)K'\log N$. Optimizing their common bias gives the expected-error threshold
\[
P_{\rm out}=2K'\bigl(\sqrt{\log(N-K')}+\sqrt{\log K'}\bigr)^2,\qquad\theta_{\rm out}=\frac{\sqrt{\log(N-K')}}{\sqrt{\log(N-K')}+\sqrt{\log K'}},
\]
and below $(1-\varepsilon)P_{\rm out}$ every assignment of local output probes has a diverging expected number of read errors.
\end{proposition}

\begin{proof}
Condition on the input. For an inactive wire, $f_\gamma$ is independent of $v^*$, and the read has mean zero and variance $K'/p$. For an active wire, $\ip{f_\gamma}{v^*}-1$ has variance $(K'+1)/p$, the extra term coming from the random self-norm. An inactive coordinate row has the law of $AS$ with $A\sim N(0,1)$ and $S\sim N(0,K')$ independent, and an active row the law of $A(A+S)$ with $\Var S=K'-1$; their moment generating functions are $(1-K't^2)^{-1/2}$ and $(1-2t-(K'-1)t^2)^{-1/2}$. Expanding the centered logarithms at displacement $\mu$ gives the rate $\mu^2/(2K')[1+O(1/K')]$, and the tilting argument of \cref{sm:lem:tilt} gives the matching lower bound. At the midpoint both displacements are $1/2$, and summing over the $N$ reads gives the first claim. At a free bias the constraints are $p\theta^2/(2K')>\log(N-K')$ and $p(1-\theta)^2/(2K')>\log K'$, and balancing them gives $P_{\rm out}$ and $\theta_{\rm out}$. For the local converse, each wire is active with probability $K'/N$. Normalize its probe and condition on its own direction; the clean signals are $0$ and $\ip{w_\gamma}{f_\gamma}$, of modulus at most $1+o(1)$ on a high-probability norm event, and the other active directions stay independent Gaussians given any input, with projection variance $(K'-t)/p$ for local state $t\in\{0,1\}$. The two-class argument of \cref{sm:ball} applies with expected class counts $N-K'$ and $K'$, using only one-wire marginals.
\end{proof}

For pairwise conjunction $K'=\Theta(k^2)$, so the output stage can dominate the input dimension even though it reads one wire at a time. The hypotheses hold at the output threshold, for example, when $\log m=o(k)$; they must be imposed separately when combining with the input thresholds of the main text.

\paragraph{Comparison with an ideal execution}
Suppose an input stage produces bits $\widehat y_\gamma$ and writes $\widehat v=\sum_\gamma\widehat y_\gamma f_\gamma$. Let $\mathsf B_{\rm in}$ be the event that some input-stage gate is wrong and $\mathsf B_{\rm out}$ the event that some read of the ideal stream $v^*$ is wrong. Deterministically,
\begin{equation}\label{sm:eq:inclusion}
\{\text{some final read is wrong}\}\subseteq\mathsf B_{\rm in}\cup\mathsf B_{\rm out},
\end{equation}
because outside $\mathsf B_{\rm in}$ the actual stream equals $v^*$. An input budget with failure probability $\delta_{\rm in}$ and an output budget with $\delta_{\rm out}$ therefore compose with failure at most $\delta_{\rm in}+\delta_{\rm out}$, without conditioning on earlier success and without independence of the two events. In particular the box threshold of \cref{thm:box} composes with \cref{sm:prop:output}.

\paragraph{The fresh-write circuit class}
Fix a feedforward Boolean threshold circuit and an input, independent of all dictionaries. Each input wire and each gate output has its own independent direction $f\sim N(0,I_d/d)$, and every active output is written once into the shared stream. A gate $\gamma$ of arity $r_\gamma\le r_0$ reads previously written wires with $w_\gamma=\sum_ja_{\gamma,j}f_{w_j}$ and bias $\theta_\gamma$, where the coefficients represent its clean threshold function, are fixed before the dictionary is drawn, satisfy $\norm{a_\gamma}=1$, and have margin $\tau_\gamma=\min_z(2g_\gamma(z)-1)(\ip{a_\gamma}z-\theta_\gamma)\ge\tau_0>0$. Let $Q$ be the number of reads, $k$ the initial activity and $K$ the largest number of active directions in the ideal stream before a read; without erasure $K$ is the input activity plus all earlier ideal active writes. Put $\kappa_{\min}=\min_\gamma\tau_\gamma^2/2$.

\begin{theorem}[Write budget]\label{thm:writebudget}
Assume $k\to\infty$, $Q\to\infty$ and $\log Q=o(k)$. For fixed $\varepsilon>0$, if $d\ge(1+\varepsilon)\kappa_{\min}^{-1}K\log Q$, every gate of the circuit is evaluated correctly with probability at least $1-Q^{-\varepsilon+o(1)}$. If a fixed schedule exactly erases wires that are no longer needed, $K$ may be replaced by the peak number of live active wires.
\end{theorem}

\begin{proof}[Proof of \cref{thm:writebudget}]
Let the ideal stream before layer $\ell$ contain the true values of all earlier gates. Given the input, its active set is deterministic and independent of the dictionary. For gate $\gamma$, condition on its at most $r_0$ queried directions; the ideal interference from the other active wires is exactly Gaussian with variance at most $K\norm{w_\gamma}^2/d$. On the local Gram event $\norm{G_\gamma-I}\le\eta$ the probe norm is $1+O(\eta)$ and the clean score differs from $\ip{a_\gamma}z$ by $O(\eta)$, uniformly because the arity is bounded. Since $\tau_\gamma\ge\tau_0$, the Gaussian tail is at most $\exp\{-(1-o(1))\kappa_{\min}d/K\}\le Q^{-(1+\varepsilon)+o(1)}$. Choose $\eta=(\log Q/k)^{1/4}\to0$; Gram concentration bounds the bad event by $Ce^{-c\eta^2d}$, and $\eta^2d/\log Q\gtrsim\eta^2k\to\infty$ makes even its union over all $Q$ reads negligible. Let $\mathsf B_\gamma$ be a wrong read of the ideal stream. Outside $\bigcup_\gamma\mathsf B_\gamma$, induction over layers shows that the actual and ideal streams agree before every read, so actual failure is contained in the union, and summing the marginal bounds gives $Q^{-\varepsilon+o(1)}$. Under exact erasure by a fixed schedule the live active set at each ideal read is still determined by the input and the schedule and has at most $K$ elements, and the induction preserves equality through erasure, because the subtracted actual and ideal contributions agree whenever all earlier reads do.
\end{proof}

The theorem counts every read. Replacing $Q$ by the number of near-boundary gates is valid only if the remaining gates have a separately controlled sum of error probabilities. Non-local reads, error correction, reuse of directions, learned correlations and growing arity need further analysis. The connection to reliable computation from unreliable components \citep{vonNeumann56,EvansSchulman99} is at the level of composing error bounds; here the errors come from shared geometry rather than independent faults.

\subsection{Unrestricted encodings}\label{sm:representations}
Let $\mathcal X_k$ be the $k$-subsets of $[m]$, $2\le k\le m-1$, and for $S\in\mathcal X_k$ let the target be the indicator vector of its pairs; distinct supports have distinct targets. Let $\varphi:X'\to\R^d$ be any map on some $X'\subseteq\mathcal X_k$.

\begin{proposition}[Counting floor]\label{prop:counting}
Let $\varphi$ be any map from the $k$-subsets of $[m]$ to $\R^d$, read by $T$ affine threshold units followed by arbitrary Boolean processing. If the outputs of all pairwise conjunctions are correct on a fraction $1-\delta>0$ of supports, then $d\ge\bigl(k\log(m/k)+\log(1-\delta)\bigr)/(1+\log T)$. For polynomial $T$ and $\log k=o(\log m)$ this is $\Omega(k)$.
\end{proposition}

\begin{proposition}[A $2k$-dimensional code]\label{prop:polycode}
Encode a support $S$ by the $2k$ nonleading coefficients of $p_S(y)=\prod_{l\in S}(y-l)^2$. Then $x_ix_j=\one\{p_S(i)+p_S(j)<1/2\}$ is an affine threshold of the code for every pair, so all pairwise conjunctions on $k$-sparse inputs have sign-rank at most $2k+1$.
\end{proposition}

\begin{proof}[Proof of \cref{prop:counting}]
If $T$ affine threshold units followed by arbitrary Boolean processing are correct on $X'$, their sign vectors are distinct on $X'$. On the finite set $\varphi(X')$, ties can be removed without changing any output by moving each bias slightly upward, less than the smallest positive slack. Distinct sign vectors then occupy distinct open cells of an arrangement of $T$ affine hyperplanes in $\R^d$, whose number $C(T,d)$ satisfies $C(T,d)\le C(T-1,d)+C(T-1,d-1)$ with $C(0,d)=C(T,0)=1$, hence $C(T,d)\le\sum_{i\le d}\binom Ti$ \citep{Buck43,Cover65,Winder66}. For $1\le d\le T$, with $\vartheta=d/T$, $\vartheta^d\sum_{i\le d}\binom Ti\le(1+\vartheta)^T\le e^d$, so the sum is at most $(eT/d)^d\le(eT)^d$; for $d>T$ use $2^T\le(eT)^d$. Hence $\abs{X'}\le(eT)^d$, and $\abs{X'}\ge(1-\delta)\binom mk\ge(1-\delta)(m/k)^k$ gives the bound.
\end{proof}

The guarantee concerns whole output vectors and cannot be replaced by a small average fraction of wrong gates: changing one support element changes only $2(k-1)$ pairs. The same count covers $L-1$ hidden ReLU layers of widths $T_1,\dots,T_{L-1}$ followed by $T_L$ threshold units: after removing ties, each layer is affine on every region cut out by the earlier activation patterns, so $\abs{X'}\le\prod_\ell\sum_{i\le d}\binom{T_\ell}i$ and $d\ge\log\abs{X'}/\sum_\ell\log(eT_\ell)$. Constant depth and polynomial width still need $\Omega(k)$ dimensions in the skewed regime.

\begin{proof}[Proof of \cref{prop:polycode}]
Write $p_S(y)=\prod_{l\in S}(y-l)^2=y^{2k}+\sum_{j<2k}c_j(S)y^j$ and $\varphi(S)=(c_0(S),\dots,c_{2k-1}(S))$. At an integer $i\in[m]$, $p_S(i)=0$ if $i\in S$ and $p_S(i)\ge1$ otherwise, so $x_ix_j=\one\{p_S(i)+p_S(j)<1/2\}$. For fixed $(i,j)$ the right side is an affine threshold of the $2k$ coefficients, the leading terms being absorbed into the bias; lifting the bias adds one coordinate, so the sign matrix has sign-rank at most $2k+1$.
\end{proof}

With \cref{prop:counting}, $\Theta(k)$ dimensions are necessary and sufficient for arbitrary encodings with one threshold per pair in the skewed regime. The code is nonlinear in the input bits, its coefficients are large and its normalized margins small. Unrestricted readout complexity alone would make dimension vacuous, since a single real number can encode a finite binary input to arbitrary precision; both the encoding and the readout must be specified when a dimension lower bound is interpreted.

%% file: windowdetail.tex
\section{Window formulas and the expected-count gap}\label{sm:windowdetail}
The profiles are explicit: $\Psi_\beta(x)=\int_0^\infty\exp(-\tau-\tau^{-\beta}e^{-\beta x})\,d\tau$, with $\Psi_1(x)=2e^{-x/2}K_1(2e^{-x/2})$ for the modified Bessel function $K_1$ and $\Psi_\beta(x)\to\exp(-e^{-x})$ as $\beta\to\infty$; and $F_+(x)=e^{-\sqrt c}+\int_0^{\sqrt c}e^{-c/s}\,ds$ with $c=e^{-x}$. The optimized window is narrower than the midpoint window by the factor $4(1+\beta)/(3+\beta)\in[2,4)$. The half-slack bias is optimal to first order, but its profile $\Psi_\beta(x/2)F_+(x/2)$ is in general smaller than the supremum; a bias tuned to a target finite-size success level allocates the slack differently. At the centers the limits are $\Psi_\beta(0)$ for the midpoint and $\Psi_\beta(0)F_+(0)$ for the half-slack bias, with $F_+(0)=0.516$; \cref{sec:experiments} compares both with simulation.

\paragraph{The Markov gap inside the window}
Conditionally on the active reads, the inactive reads are independent, so the number of inactive features that create at least one wrong one-active gate is binomial, and it is asymptotically Poisson with a random intensity $\Lambda$ (\cref{sm:window}). At the midpoint, inside the window and along sequences with $\beta\to\beta_0\in[1,\infty)$, $\Lambda$ converges in law to $\Lambda_\infty=e^{\beta_0(G_A-x)}$. Hence
\[
\Prob(X=0)\to\Exp\,e^{-\Lambda_\infty}=\Psi_{\beta_0}(x)\in(0,1),
\]
while
\[
\Exp\,\Lambda_\infty=e^{-\beta_0x}\int_0^\infty\tau^{-\beta_0}e^{-\tau}\,d\tau=\infty ,
\]
and by Fatou's lemma the expected number of firing inactive features, and with it $\Exp X$, tends to infinity throughout the window. The limiting intensity is heavy-tailed: an unusually large active read lets an unusually large number of inactive features fire together. A typical draw has no error with probability $\Psi_{\beta_0}(x)$, while the mean is carried by rare draws. This is the second-order form of the gap between the ball and the box.

%% file: comparisons.tex
\section{Further comparisons and experiments}\label{sm:comparisons}
\subsection{The four pairwise thresholds}
\begin{figure}[t]
\centering
\input{fan}
\caption{The four first-order thresholds of pairwise conjunction and the threshold of feature recovery, in units of $k\log m$, as functions of $\log m/\log k$. All four conjunction thresholds equal $32k\log k$ when $\log m\sim\log k$; they separate toward $16$, $8$, $4$ and $2$ (dotted). The vertical lines mark the regimes of the simulations in \cref{sec:experiments}.}
\label{fig:fan}
\end{figure}
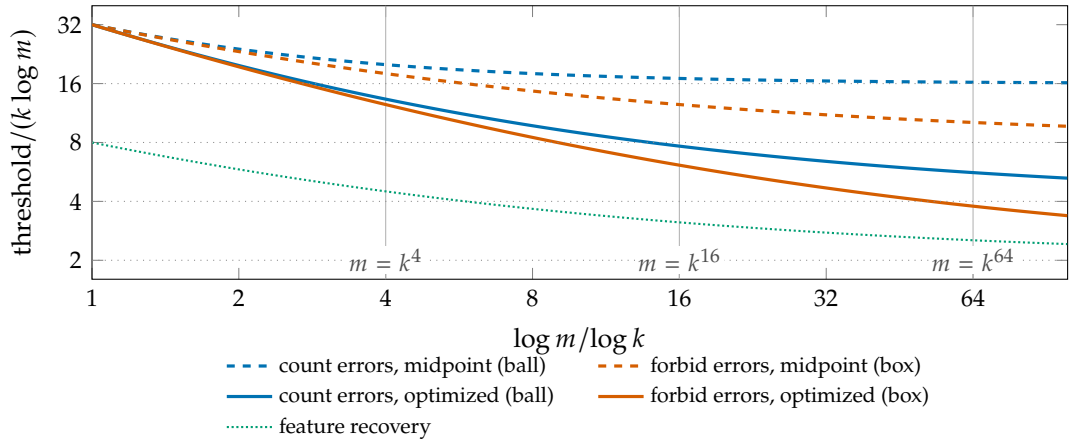

\subsection{Comparison with other decoders}
The ratio reflects the architecture rather than the task. A common bias for $r$-way conjunction must separate the sum of the $r$ smallest active reads from the sum of the $r-1$ largest active reads and the largest inactive read, so its box support contains $2r-1$ active levels, whereas that of recovery contains one. For $m=k^4$, where $\beta=2$, recovery needs $18k\log k$ dimensions and direct pairwise conjunction $50k\log k$; the midpoint pairwise network needs $72k\log k$ to be reliable and $80k\log k$ for a vanishing expected count. Decoders that are not a single threshold layer can do better. The Lasso recovers the support from about $2k\log(m-k)$ measurements \citep{Wainwright09}, and $\ell_1$ minimization recovers a typical nonnegative $k$-sparse vector from about $2k\log(m/k)$ noiseless Gaussian measurements \citep{DonohoTanner09}. When $\log k=o(\log m)$ both are $2k\log m$ to first order, the optimized reliability threshold of direct pairwise conjunction; when $\log m\sim\log k$ the second is $o(k\log k)$, far below $32k\log k$. The constants of this paper are therefore those of one layer of threshold readouts of a fixed random code.

For asymmetric gates the weights matter as well as the bias: at the alternating threshold of the comparator $\one\{\sum_l2^{l-1}z_l\ge T\}$ the optimal weights are Fibonacci numbers, and when all classes count equally binary weights overstate the price by a factor exponential in $r$ (\cref{sm:gates}).

\subsection{Additional experimental figures}
The first two panels of \cref{fig:e3} show the $k=1024$ comparisons. The figures below add the other activities, window profiles and bias rules.
\begin{figure}[t]
\centering
\includegraphics[width=\linewidth]{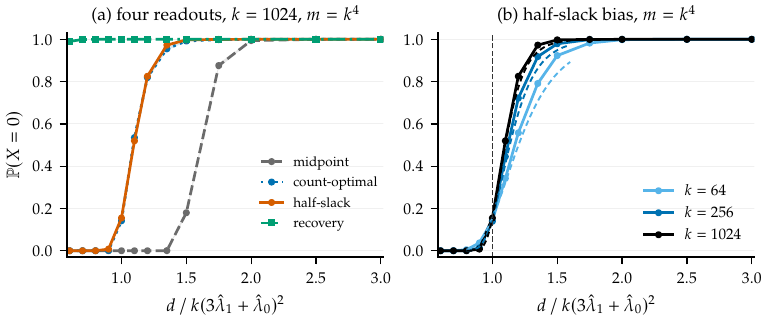}
\caption{Pairwise conjunction, 4096 exact draws per point, pointwise $95\%$ Wilson bands. (a) Three deterministic biases and feature recovery evaluated on the same draws; the count-optimal and half-slack curves nearly coincide, and recovery switches on near ratio $0.39$, left of the plotted range. (b) The half-slack bias of \cref{thm:thr}(a) as $k$ grows; dashed curves are the profiles $\Psi_{\hat\beta}(x/2)F_+(x/2)$ of \cref{thm:window}(ii) at the slack $x=\lhat_1(\sqrt{d/k}-3\lhat_1-\lhat_0)$.}
\label{fig:e1}
\end{figure}
\begin{figure}[t]
\centering
\includegraphics[width=\linewidth]{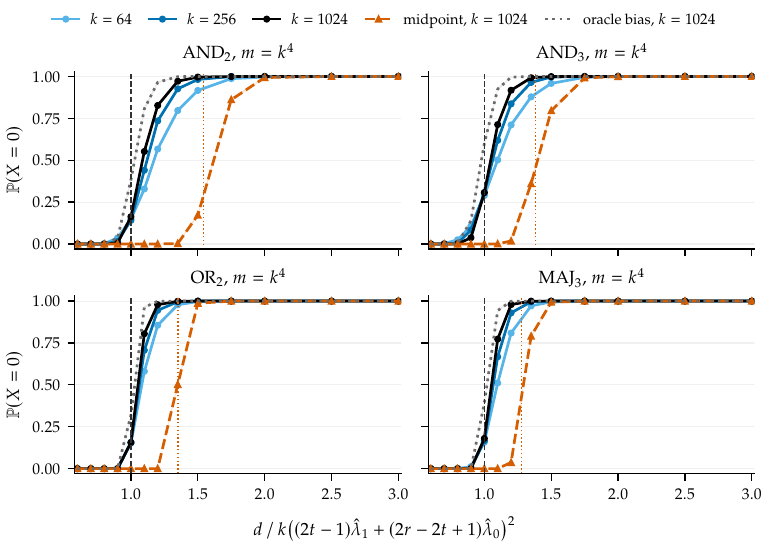}
\caption{Reliability of four threshold gates at $m=k^4$ with the half-slack bias of \cref{thm:thr}(a) for $k=64,256,1024$, and with the midpoint bias and the sample-adaptive oracle bias for $k=1024$. The horizontal axis is the dimension in units of the quantile box center of each gate. Dashed line: that center; dotted line: the midpoint box threshold of \cref{thm:thr}(b) in the same units.}
\label{fig:e2}
\end{figure}
\subsection{Full campaign results}
\paragraph{The bias matters before the limit}
The first campaign (\cref{fig:e1}) evaluates three deterministic biases for pairwise conjunction on shared draws, together with feature recovery, at $k=64,256,1024$ with $m=k^4$ and at $k=1024$ with $m=k^2$, with 13 dimensions per case and 4096 draws per cell. At $k=1024$, $m=1024^4$ and $d=1.35\,k(3\lhat_1+\lhat_0)^2$, the half-slack bias of \cref{thm:thr}(a) succeeds on 3984 of 4096 draws, the bias optimized for counting on 3912, and the midpoint on 2. At the smallest dimension of the grid feature recovery succeeds on 4056 draws and every conjunction rule on none, as \cref{cor:recovery} predicts. At the quantile center the three cases with $m=k^4$ succeed with frequencies $0.143$, $0.139$ and $0.155$, and \cref{thm:window}(ii) predicts $\Psi_{\hat\beta}(0)F_+(0)=0.155$, $0.154$ and $0.154$ for the half-slack bias, with $\hat\beta=\lhat_0/\lhat_1$; the dashed curves in \cref{fig:e1}(b) are the whole predicted profiles, and \cref{tab:window} adds the case $m=k^2$ and an independent replication.

\paragraph{The box beyond conjunction}
The second campaign tests \cref{thm:thr}(a) for $\ANDg_2$, $\ANDg_3$, $\ORg_2$ and $\MAJg_3$ at $m=k^4$ (\cref{fig:e2}). In units of the quantile box center $k((2t-1)\lhat_1+(2r-2t+1)\lhat_0)^2$, all four gates show the same picture at $k=1024$: success is at most $3.8\%$ at $0.9$, between $15.6\%$ and $30.7\%$ at $1$, and between $82.7\%$ and $97.9\%$ at $1.2$, and the curves steepen with $k$. The oracle, which chooses the bias after seeing the reads, reaches $50\%$ success between $0.98$ and $1.03$ times the center, $2\%$ to $7\%$ earlier in dimension than the half-slack rule and consistent with \cref{thm:box}(ii). The midpoint rule of each gate switches on near its own box threshold, marked by a dotted line; for $\ANDg_2$ it succeeds on none of the 4096 draws at $1.2$.

\paragraph{Separation of the two thresholds}
The third campaign fixes the midpoint pairwise network at $k=1024$ and increases $\log m/\log k$ from $4$ to $16$ to $64$, so that $\beta=2,4,8$ (\cref{fig:e3}). Reliability switches on at the box: at the center the success frequencies $0.295$, $0.324$ and $0.340$ match the limits $\Psi_{\hat\beta}(0)=0.298$, $0.326$ and $0.346$ of \cref{thm:window}(i). The expected number of wrong outputs crosses one at about $1.23$, $1.45$ and $1.67$ times the center $d_0$ of \eqref{eq:centers}, tracking the first-order ball positions $2(1+\beta^2)/(1+\beta)^2=1.11$, $1.36$ and $1.60$ with offsets $0.12$, $0.09$ and $0.06$ that shrink as $\beta$ grows. At $m=k^{64}$ and $d=1.5\,d_0$ all 4096 draws are correct, while $\Exp X\approx2.8\times10^{19}$. Values of $m$ this large are a mathematical stress test of the first-order theory, reachable because the exact law never builds the dictionary.

%% file: fan.tex
\begin{tikzpicture}
\begin{axis}[
  width=0.92\linewidth, height=5.2cm,
  xmode=log, ymode=log, log basis y=2, xmin=1, xmax=100, ymin=1.6, ymax=40,
  xlabel={$\log m/\log k$}, ylabel={threshold$/(k\log m)$},
  xtick={1,2,4,8,16,32,64}, xticklabels={$1$,$2$,$4$,$8$,$16$,$32$,$64$},
  ytick={2,4,8,16,32}, yticklabels={$2$,$4$,$8$,$16$,$32$},
  tick label style={font=\footnotesize}, label style={font=\small},
  legend style={font=\scriptsize, draw=none, fill=none, at={(0.5,-0.24)}, anchor=north, cells={anchor=west}, /tikz/every even column/.append style={column sep=14pt}},
  legend columns=2, legend cell align=left,
  samples=200, domain=1:100, clip=true,
]
\draw[black!30, thin] (axis cs:4,1.6) -- (axis cs:4,40);
\draw[black!30, thin] (axis cs:16,1.6) -- (axis cs:16,40);
\draw[black!30, thin] (axis cs:64,1.6) -- (axis cs:64,40);
\addplot[wongblue, very thick, dashed] {16*(1+1/x)};
\addlegendentry{count errors, midpoint (ball)}
\addplot[wongverm, very thick, dashed] {8*(1+1/sqrt(x))^2};
\addlegendentry{forbid errors, midpoint (box)}
\addplot[wongblue, very thick] {4*(sqrt(2/x)+sqrt(1+1/x))^2};
\addlegendentry{count errors, optimized (ball)}
\addplot[wongverm, very thick] {2*(1+3/sqrt(x))^2};
\addlegendentry{forbid errors, optimized (box)}
\addplot[wonggreen, thick, densely dotted] {2*(1+1/sqrt(x))^2};
\addlegendentry{feature recovery}
\addplot[black!45, dotted, forget plot] {16};
\addplot[black!45, dotted, forget plot] {8};
\addplot[black!45, dotted, forget plot] {4};
\addplot[black!45, dotted, forget plot] {2};
\node[font=\footnotesize, text=black!70, anchor=south, fill=white, inner sep=1pt] at (axis cs:4,1.7) {$m=k^4$};
\node[font=\footnotesize, text=black!70, anchor=south, fill=white, inner sep=1pt] at (axis cs:16,1.7) {$m=k^{16}$};
\node[font=\footnotesize, text=black!70, anchor=south, fill=white, inner sep=1pt] at (axis cs:64,1.7) {$m=k^{64}$};
\end{axis}
\end{tikzpicture}

%% file: numerics.tex
\section{Finite checks}\label{sm:numerics}

The checks below test identities, closed forms and implementations on finite instances. The code-and-data archive contains the scripts and outputs.

\paragraph{The box identity}
The script \texttt{check\_box\_identity.py} draws 400 random instances with $r\in\{1,2,3\}$, pools of sizes between $2r$ and $2r+2$, Gaussian reads and Gaussian weights (every third instance with nonnegative weights). For every class $z$ and both error directions it enumerates all ordered tuples of distinct features and compares the maximum of $\ip b{\xi_I}$ with the rearrangement formula of \cref{lem:boxid}. All $3720$ class--direction pairs agree to within $1.8\times10^{-15}$.

\paragraph{The two programs}
The script \texttt{check\_programs.py} solves the ball program \eqref{eq:socp}, with the radii $\norm{\lambda_z}$ in units $\lambda_1=1$, $\lambda_0=\beta$, as a second-order cone program, and the box program of \cref{prop:boxprogram} by bisection over linear programs, in both cases over all representations $(a,\theta)$ without imposing symmetry. For every $\THR^r_t$ with $r\le4$ and every $\beta\in\{1,1.5,2,4,10\}$, the optimal values agree with (a) and (d) of \cref{thm:thr}, and the midpoint prices with (b) and (c), to relative accuracy $8\times10^{-7}$. The optimal box weights returned by the solver are equal across positions up to solver tolerance. At $\beta=1$ all four values equal $4r^2$ in these units.

\paragraph{The reduced sampler}
The script \texttt{check\_thr\_direct.py} compares the exact reduced sampler with direct simulation: for each draw it generates a Gaussian dictionary, forms $u$, and enumerates every $r$-subset. For the largest negative and the smallest positive score, the two-sample Kolmogorov--Smirnov statistics over 4000 draws per method are as follows.

\begin{center}
\small
\begin{tabular}{@{}lcccc@{}}
\toprule
Gate, $(k,m,d)$ & largest negative & $p$ & smallest positive & $p$\\
\midrule
$\ANDg_3$, $(8,30,120)$ & $0.013$ & $0.87$ & $0.025$ & $0.18$\\
$\ORg_2$, $(8,30,120)$ & $0.020$ & $0.42$ & $0.022$ & $0.31$\\
$\MAJg_3$, $(8,30,120)$ & $0.020$ & $0.40$ & $0.024$ & $0.22$\\
recovery, $(8,40,60)$ & $0.015$ & $0.76$ & $0.014$ & $0.81$\\
\bottomrule
\end{tabular}
\end{center}

For pairwise conjunction the frozen campaign of \cref{sm:experiments} used an independent direct-matrix implementation at $(k,m,d)=(8,40,100)$ with 12000 draws per method: the Kolmogorov--Smirnov distances for the largest active read, the sum of the two smallest active reads, the largest inactive read and the largest negative score are $0.0121$, $0.0112$, $0.0075$ and $0.0117$. A separate enumeration confirms the all-pairs success identity of \cref{prop:cloud,lem:boxid} at five biases.

\paragraph{Algebra}
The script \texttt{verify\_algebra\_v2.py} checks all $1013$ nonconstant comparator thresholds of arities $1$ to $9$ against their $348502$ truth-table entries: the unit-gap construction, the Fibonacci bound on every weight, and attainment by the alternating threshold. It evaluates the polynomial code of \cref{prop:polycode} in integer arithmetic on $2427$ supports ($3\le m\le11$, $2\le k\le\min(5,m-1)$), checking $106658$ pair outputs, and verifies the row cumulants and the coefficient $c_\nu$ of \cref{sm:prop:ledger} symbolically.

%% file: sm_experiments.tex
\section{Experimental protocol and additional results}\label{sm:experiments}

\subsection{Sampling}
Each draw samples $H\sim\chi^2_d/d$ and $k$ standard Gaussians, forms the active reads of \eqref{eq:exact-law}, and sorts them. For the inactive sample only its $j$ smallest and $j$ largest values are needed ($j=2$ for pairwise conjunction, $j=r$ in general). Writing each inactive value through its distribution function, the smallest of $n$ uniforms is $1-U^{1/n}$; conditionally on the smallest $j$ values, the other points are uniform above the $j$th, and their largest values follow by the same rule on the shorter interval. The $2j$ uniforms are mapped through $\Phi^{-1}$ in the upper-tail parametrization, using \texttt{expm1} and logarithms to keep precision when $n$ is as large as $k^{64}$. Success is checked by \cref{prop:cloud,lem:boxid} with the strictness of the readout; ties have probability zero. A draw costs $O(k\log k)$ time and $O(k)$ memory, independent of $m$ and $d$, while direct simulation needs $O(dm)$ memory.

Expected error counts are computed by one-dimensional quadrature of \cref{prop:onegate}, $\Exp X=\sum_jC_jp_j$ with the exact class sizes $C_j=\binom kj\binom n{r-j}$ of the unordered family, integrating over $J$ in the logarithmic domain between chi-square quantiles of probability $10^{-15}$ with relative tolerance $10^{-10}$. For pairwise conjunction this reproduces the frozen values of the first campaign, for example $6.7416$ and $0.32739$ at ratios $1.2$ and $1.35$ below.

\subsection{Campaign E1: pairwise conjunction (frozen 22 September 2026)}
The grid uses $m=k^4$ with $k=64,256,1024$ and $m=k^2$ with $k=1024$, the dimension ratios $\{0.6,0.7,0.8,0.9,1,1.1,1.2,1.35,1.5,1.75,2,2.5,3\}$ relative to $k(3\lhat_1+\lhat_0)^2$, rounded to integers, and 4096 draws per cell: 52 cells and 212992 independent draws. Seed 20260922 generates one \texttt{SeedSequence} per cell. The three biases, fixed before sampling, are the midpoint $3/2$, the bias optimized for counting $1+\sqrt{\log N_-}/(\sqrt{\log N_-}+\sqrt{\log N_+})$ with $N_-=kn$ and $N_+=\binom k2$, and the half-slack bias $3/2+\tfrac12\sigd(\lhat_0-\lhat_1)$ of \cref{thm:thr}(a); feature recovery uses $\tfrac12+\tfrac12\sigd(\lhat_0-\lhat_1)$. All rules share the draws of a cell. The sample-adaptive event that the largest negative score lies below the smallest positive score is stored as a feasibility bound and is not counted as a rule. No cell or draw was excluded. The run used Python 3.11.3, NumPy 1.26.4 and SciPy 1.16.3 on one CPU node with one BLAS thread, 81 seconds of CPU time and about 221 MiB of memory; runner and configuration hashes match the recorded run.

\begin{figure}[ht]
\centering
\includegraphics[width=\linewidth]{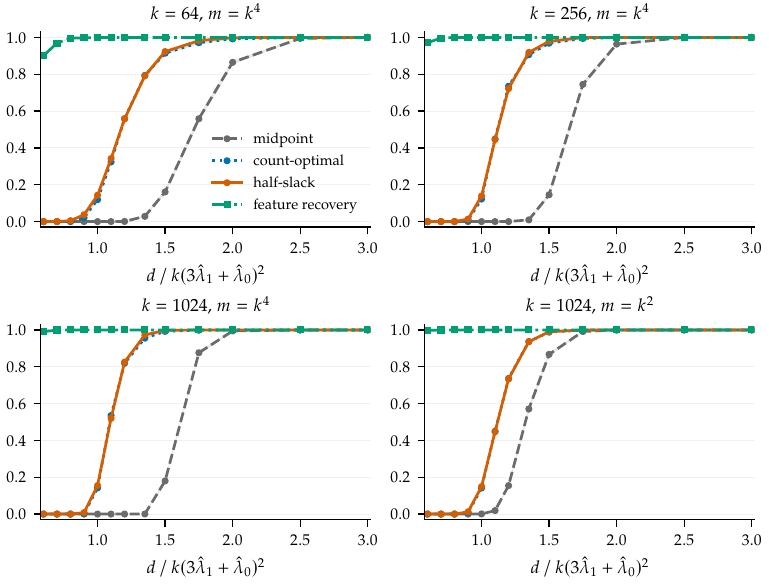}
\caption{All four cases of campaign E1. Each point uses 4096 independent draws; bands are pointwise $95\%$ Wilson intervals. The three bias rules and feature recovery share each cell's draws.}
\label{fig:e1all}
\end{figure}

\paragraph{Estimating the mean is harder than estimating reliability}
The mean count is carried by rare draws, so Monte Carlo estimates of it can mislead even when their standard errors are small. At $k=1024$, $m=1024^4$ and ratio $1.35$, the Rao--Blackwellized sample mean of $X$ is $0.063$ with standard error $0.011$, while quadrature of the exact one-gate law gives $0.327$ (\cref{tab:e1}). The main text therefore reports quadrature values of $\Exp X$ and uses simulation only for $\Prob(X=0)$.

\begin{table}[ht]
\caption{The half-slack bias at $k=1024$, $m=1024^4$ (campaign E1). Mean error is the Rao--Blackwellized sample mean over the inactive reads with its Monte Carlo standard error; the last column is the quadrature of \cref{prop:onegate}. Rare draws carry the mean, so the sample mean and its standard error can both be misleading, as at ratio $1.35$.}
\label{tab:e1}
\centering
\small
\begin{tabular}{@{}rrlrrr@{}}
\toprule
Ratio & $d$ & $\Prob(X=0)$ [95\% interval] & MC mean & MC s.e. & Quadrature\\
\midrule
\input{generated_results}
\end{tabular}
\end{table}

\subsection{Campaigns E2 and E3 (frozen 27 September 2026)}
The configuration \texttt{box-v1.json} was written before any run. Campaign E2 covers $\ANDg_2$, $\ANDg_3$, $\ORg_2$ and $\MAJg_3$ with $k=64,256,1024$ and $m=k^4$, the same 13 dimension ratios relative to the quantile box center $k((2t-1)\lhat_1+(2r-2t+1)\lhat_0)^2$, and 4096 draws per cell; the rules are the half-slack bias $t-\tfrac12+\tfrac12\sigd(\lhat_0-\lhat_1)$, the midpoint $t-\tfrac12$, and the oracle, which succeeds on a draw when some common bias separates its largest negative score from its smallest positive score; by \cref{thm:box}(ii) even this sample-adaptive rule fails below the box threshold. Campaign E3 fixes the midpoint $\ANDg_2$ network at $k=1024$ with $m=k^4,k^{16},k^{64}$ and 11 ratios relative to $4k(\lhat_1+\lhat_0)^2$. Seed 20260927 generates one \texttt{SeedSequence} per cell. The run used Python 3.13.9, NumPy 1.26.4 and SciPy 1.16.3 on a laptop CPU and took 188 seconds; all per-draw extremes are stored.

\begin{table}[ht]
\caption{Campaign E2 at $k=1024$, $m=k^4$: success frequency of the half-slack rule (out of 4096) and of the oracle at selected dimension ratios to the quantile box center.}
\label{tab:e2}
\centering
\small
\input{generated_e2}
\end{table}

\begin{table}[ht]
\caption{Campaign E3, the midpoint pairwise network at $k=1024$: success frequency (out of 4096) and exact expected number of wrong outputs at selected dimension ratios to $d_0=4k(\lhat_1+\lhat_0)^2$.}
\label{tab:e3}
\centering
\small
\input{generated_e3}
\end{table}

\subsection{Window profiles at the center}
\Cref{tab:window} compares the success frequencies at the quantile centers with the limits of \cref{thm:window}, evaluated at $\hat\beta=\lhat_0/\lhat_1$: $\Psi_{\hat\beta}(0)F_+(0)$ for the half-slack bias, whose slack splits evenly, and $\Psi_{\hat\beta}(0)$ for the midpoint. The pairwise rows of campaign E2 use the same rule as E1 with independent seeds. At $k=1024$ all six limits lie inside the $95\%$ intervals; at $k=64$ and $256$ the frequencies lie between $0.006$ and $0.015$ below the limits, a finite-size lag that disappears by $k=1024$.

\begin{table}[ht]
\caption{Success frequency at the quantile center (dimension ratio $1$, 4096 draws) against the limit of \cref{thm:window}.}
\label{tab:window}
\centering
\small
\input{generated_window}
\end{table}

%% file: generated_results.tex
0.60 & 164032 & 0.0000 [0.0000, 0.0009] & 2.462e+06 & 4.199e+05 & 2.485e+06 \\
1.00 & 273387 & 0.1548 [0.1440, 0.1662] & 312.3 & 127.7 & 427.3 \\
1.20 & 328064 & 0.8252 [0.8133, 0.8365] & 10.46 & 8.232 & 6.742 \\
1.35 & 369072 & 0.9727 [0.9672, 0.9772] & 0.06288 & 0.01097 & 0.3274 \\
1.50 & 410081 & 0.9976 [0.9955, 0.9987] & 0.002693 & 0.0008277 & 0.0178 \\
1.75 & 478427 & 0.9995 [0.9982, 0.9999] & 0.0004904 & 0.0003452 & 0.0002069 \\
2.00 & 546774 & 1.0000 [0.9991, 1.0000] & 1.202e-07 & 1.188e-07 & 4.052e-06 \\
\bottomrule

%% file: generated_e2.tex
\begin{tabular}{@{}lrrrrrrr@{}}
\toprule
Ratio & 0.8 & 0.9 & 1 & 1.1 & 1.2 & 1.35 & 1.5\\
\midrule
$\mathrm{AND}_2$, half-slack & 0 & 21 & 668 & 2262 & 3387 & 3980 & 4081\\
$\mathrm{AND}_2$, oracle & 1 & 141 & 1573 & 3305 & 3952 & 4092 & 4096\\
$\mathrm{AND}_3$, half-slack & 1 & 154 & 1258 & 2919 & 3758 & 4065 & 4093\\
$\mathrm{AND}_3$, oracle & 11 & 571 & 2456 & 3788 & 4077 & 4095 & 4096\\
$\mathrm{OR}_2$, half-slack & 0 & 0 & 640 & 3298 & 4009 & 4092 & 4096\\
$\mathrm{OR}_2$, oracle & 0 & 0 & 1308 & 3944 & 4094 & 4096 & 4096\\
$\mathrm{MAJ}_3$, half-slack & 0 & 0 & 738 & 3167 & 4011 & 4095 & 4096\\
$\mathrm{MAJ}_3$, oracle & 0 & 2 & 1602 & 3865 & 4094 & 4096 & 4096\\
\bottomrule
\end{tabular}

%% file: generated_e3.tex
\begin{tabular}{@{}rrlrlrl@{}}
\toprule
 & \multicolumn{2}{c}{$m=k^{4}$} & \multicolumn{2}{c}{$m=k^{16}$} & \multicolumn{2}{c}{$m=k^{64}$}\\
\cmidrule(lr){2-3}\cmidrule(lr){4-5}\cmidrule(l){6-7}
Ratio & succ. & $\mathbb{E}X$ & succ. & $\mathbb{E}X$ & succ. & $\mathbb{E}X$\\
\midrule
0.9 & 33 & $5.81\times10^{3}$ & 0 & $9.03\times10^{18}$ & 0 & $2.01\times10^{89}$\\
1 & 1208 & $422$ & 1327 & $3.29\times10^{15}$ & 1392 & $4.51\times10^{77}$\\
1.1 & 3191 & $30.9$ & 3896 & $1.21\times10^{12}$ & 4091 & $1.02\times10^{66}$\\
1.2 & 3922 & $2.26$ & 4089 & $4.43\times10^{8}$ & 4096 & $2.32\times10^{54}$\\
1.35 & 4084 & $0.0453$ & 4096 & $3.15\times10^{3}$ & 4096 & $7.98\times10^{36}$\\
1.5 & 4095 & $9.10\times10^{-4}$ & 4096 & $0.0225$ & 4096 & $2.76\times10^{19}$\\
1.75 & 4096 & $1.37\times10^{-6}$ & 4096 & $6.01\times10^{-11}$ & 4096 & $2.22\times10^{-10}$\\
2 & 4096 & $2.07\times10^{-9}$ & 4096 & $1.62\times10^{-19}$ & 4096 & $1.79\times10^{-39}$\\
\bottomrule
\end{tabular}

%% file: generated_window.tex
\begin{tabular}{@{}llrrrl@{}}
\toprule
Campaign, rule & $m$ & $k$ & $\hat\beta$ & Limit & Simulation [95\% interval]\\
\midrule
E1, half-slack & $k^{4}$ & 64 & 2.458 & 0.155 & 0.143 [0.133, 0.154]\\
E1, half-slack & $k^{4}$ & 256 & 2.342 & 0.154 & 0.139 [0.129, 0.150]\\
E1, half-slack & $k^{4}$ & 1024 & 2.275 & 0.154 & 0.155 [0.144, 0.166]\\
E1, half-slack & $k^{2}$ & 1024 & 1.538 & 0.147 & 0.149 [0.139, 0.160]\\
E2, half-slack & $k^{4}$ & 64 & 2.458 & 0.155 & 0.149 [0.138, 0.160]\\
E2, half-slack & $k^{4}$ & 256 & 2.342 & 0.154 & 0.141 [0.131, 0.152]\\
E2, half-slack & $k^{4}$ & 1024 & 2.275 & 0.154 & 0.163 [0.152, 0.175]\\
E3, midpoint & $k^{4}$ & 1024 & 2.275 & 0.298 & 0.295 [0.281, 0.309]\\
E3, midpoint & $k^{16}$ & 1024 & 4.730 & 0.326 & 0.324 [0.310, 0.338]\\
E3, midpoint & $k^{64}$ & 1024 & 9.570 & 0.346 & 0.340 [0.325, 0.354]\\
\bottomrule
\end{tabular}

%% file: applications.tex
\section{Applications, implications and open problems}\label{sm:applications}

A threshold in this paper answers a design question: how wide must a shared stored vector be before a family of logical reads becomes dependable? The answer depends on what ``dependable'' means, on the bias, and on what else the vector carries. This section turns the results into five rules that make these dependencies explicit. \Cref{sm:app:criteria} orders the error criteria by the width they need. \Cref{sm:app:capacity} inverts the thresholds into a feature capacity at fixed width. \Cref{sm:app:sizing} converts the window profiles into the width for a prescribed success probability. \Cref{sm:app:audit} checks every gate of a universal layer on any realized representation from a few extreme reads. \Cref{sm:app:unread} separates the part of a threshold paid for activity from the part paid for the read family. \Cref{sm:app:evaluation,sm:app:claims} draw the consequences for evaluation and for claims about superposition, and \cref{sm:app:open} states open problems. The corollaries and propositions below are proved from the results of the paper. Numbers new to this section are computed by \texttt{check\_sizing\_rule.py} from the window profiles and the frozen campaigns of \cref{sm:experiments}, without new sampling; the others are quoted from \cref{sm:experiments,sm:comparisons}.

\subsection{Which failure, which width}\label{sm:app:criteria}

A layer of $\abs\cF$ gates can be judged by the fraction of wrong outputs, by the expected number of wrong outputs, by the probability that some output is wrong, or by correctness on every input. The four criteria protect different objects: a typical gate, an additive loss, a typical input and every input. The widths they require differ by more than constants, and the first criterion carries no logarithm at all.

\begin{corollary}[Average accuracy carries no logarithm]\label{cor:accuracy}
Let $\cF$ be any gate family of fixed arity $r$, let $(a,\theta)\in\Rep(g)$ be fixed, and use the canonical rule. Then $\Exp X/\abs\cF\to0$ if and only if $d/k\to\infty$. No further condition on $m$ and no condition on the structure of $\cF$ is needed.
\end{corollary}

\begin{proof}
A fixed gate has class $z$ with probability $N_z/\abs\cF$, so $\Exp X/\abs\cF=\sum_z(N_z/\abs\cF)\,p_z$ is a convex combination of the $2^r$ class error probabilities. By \eqref{sm:eq:pz}, $p_z=\Exp\Phibar(T_z)$ with $T_z=(s_z-\varsigma_z\ip az(J-1))/(v_z\sqrt{J/d})$ and $(k-r)\norm a^2\le v_z^2\le k\norm a^2$. If $d/k\to\infty$, then $J\to1$ in probability, and $s_z\sqrt d/(\norm a\sqrt k)\to\infty$ because $s_z$ and $a$ are fixed. Hence $T_z\to\infty$ in probability, $p_z\to0$ for every class by bounded convergence, and so does the convex combination. If instead $d\le Ck$ along a subsequence, then on $\{1/2\le J\le2\}$ the bound \eqref{sm:eq:marginbound} gives a numerator at most $3\sqrt r\norm a$ and a denominator at least $\norm a\sqrt{(k-r)/(2d)}$, so $T_z\le6\sqrt{rC}$ once $k\ge2r$. The probability of this event is positive for every $d$ and tends to one, so it is at least some $c_0>0$. Thus $p_z\ge c_0\Phibar(6\sqrt{rC})$ for every class, and the convex combination stays bounded away from zero.
\end{proof}

The average is dominated by the most numerous class. For sparse inputs that class has no active input, and its gates fail only when a Gaussian interference of scale $\sigd$ crosses a fixed margin; no extreme value enters. Campaign E1 shows the size of this blind spot. At $k=1024$, $m=1024^4$ and the quantile box centre, the half-slack network has a wrong output on $85\%$ of draws and an expected count of $427$, while the expected fraction of wrong outputs is $7.1\times10^{-22}$. At $0.6$ times the centre no draw is correct, and the fraction is $4.1\times10^{-18}$. The zero-active class contains all but about a fraction $2k/m$ of the gates, and at the centre it contributes $10^{-54}$ expected errors.

\begin{table}[!ht]
\caption{Four error criteria for the universal family of pairwise conjunctions when $\log k=o(\log m)$, in order of the width they require. The first row is a condition on $d/k$, the middle rows are sharp first-order thresholds, and the last row is sharp up to constants. With optimized biases, the reliability threshold holds against every common representation and the expected-count threshold against every gate-local rule.}
\label{tab:criteria}
\centering
\footnotesize
\renewcommand{\arraystretch}{1.15}
\begin{tabular}{@{}llccl@{}}
\toprule
Criterion & Protects & Midpoint bias & Optimized bias & Source\\
\midrule
$\Exp X/\abs\cF\to0$ & a typical gate & \multicolumn{2}{c}{$d/k\to\infty$ (any fixed bias)} & \Cref{cor:accuracy}\\
$\Prob(X=0)\to1$ & a typical input & $8k\log m$ & $2k\log m$ & \Cref{thm:thr}(a,b)\\
$\Exp X\to0$ & the total additive loss & $16k\log m$ & $4k\log m$ & \Cref{thm:thr}(c,d)\\
correct on every support & every input, chosen after $E$ & $\Theta(k^2\log m)$ & --- & \Cref{prop:uniform}\\
\bottomrule
\end{tabular}
\end{table}

\Cref{tab:criteria} is a specification checklist. A sentence such as ``the layer computes every pairwise conjunction in superposition at width $d$'' becomes checkable once it names a row. At width of order $k\log m$ the first row holds, the last row fails, and the middle rows are decided by constants that range over a factor of eight. The row follows from how the outputs are consumed. A batch of answers that is discarded after one error needs the second row. A cost that accrues per wrong answer needs the third. Inputs that may be chosen with knowledge of the dictionary need the fourth, and for midpoint probes the widths of the middle rows do not buy it: by \cref{prop:uniform}, a midpoint network whose width is of order $k\log m$ fails on some support with probability tending to one, since such a width is eventually below $ck^2\log m$.

\subsection{Feature capacity at fixed width}\label{sm:app:capacity}

Superposition is usually measured by how many features fit into $d$ dimensions. The thresholds can be read the same way: fix $d$ and $k$, and ask for the largest dictionary under which a criterion still holds. When $d/(k\log k)\to\infty$ this capacity is exponential in $d/k$, and the criterion and the bias change the constant in the exponent.

\begin{corollary}[Capacity at fixed width]\label{cor:capacity}
Let $\cF=\cU_r(\THR^r_t)$, let $k\to\infty$ with $d/(k\log k)\to\infty$, and fix $\varepsilon\in(0,1)$. For the four parts of \cref{thm:thr} put
\[
c_{\rm a}=(2r-2t+1)^2,\quad c_{\rm b}=4(r-t+1)^2,\quad c_{\rm c}=4r(r-t+1),\quad c_{\rm d}=r\bigl(\sqrt{r-t}+\sqrt{r-t+1}\bigr)^2 .
\]
For each part $\bullet\in\{{\rm a},{\rm b},{\rm c},{\rm d}\}$: if $m/k\to\infty$ and $\log m\le(1-\varepsilon)\,d/(2c_\bullet k)$, its criterion holds for the representation named there; if $\log m\ge(1+\varepsilon)\,d/(2c_\bullet k)$, its criterion fails, in part (a) for every common representation and in part (d) for every gate-local rule.
\end{corollary}

\begin{proof}
Each threshold of \cref{thm:thr} has the form $kQ(\lambda_1,\lambda_0)(1+o(1))$, where $Q$ is continuous, nondecreasing in each argument and positively homogeneous of degree two, with $Q(0,1)=c_\bullet$; for example, $Q=((2t-1)\lambda_1+(2r-2t+1)\lambda_0)^2$ in part (a) and $Q=r(\rho_t+\rho_{t-1})^2$ in part (d). Suppose first that $\log m\le(1-\varepsilon)d/(2c_\bullet k)$, and put $\Lambda^2=(1-\varepsilon)d/(c_\bullet k)$. Then $\lambda_0\le\Lambda$, and $\lambda_1/\Lambda\to0$ because $d/(k\log k)\to\infty$. Monotonicity and homogeneity give
\[
kQ(\lambda_1,\lambda_0)\le k\Lambda^2\,Q(\lambda_1/\Lambda,1)=(1-\varepsilon)\,d\,(1+o(1)),
\]
so for large $k$ the width exceeds the threshold by a factor at least $1+\varepsilon/2$, and \cref{thm:thr} gives the criterion. Suppose next that $\log m\ge(1+\varepsilon)d/(2c_\bullet k)$. Then $\log m/\log k\to\infty$, so $m/k\to\infty$ and $\lambda_0^2=2\log m\,(1+o(1))$, and
\[
kQ(\lambda_1,\lambda_0)\ge kQ(0,\lambda_0)=c_\bullet k\lambda_0^2\ge(1+\varepsilon)\,d\,(1-o(1)).
\]
The width is thus below the threshold by a fixed factor, and the failure statements and converses of \cref{thm:thr} apply.
\end{proof}

\begin{table}[!ht]
\caption{Capacity constants of \cref{cor:capacity}. At width $d$ and activity $k$, the criterion tolerates dictionaries with $\log m$ up to $d/(2c_\bullet k)$ to first order, so a smaller constant means exponentially more features. Columns follow the parts of \cref{thm:thr}.}
\label{tab:capacity}
\centering
\small
\renewcommand{\arraystretch}{1.15}
\begin{tabular}{@{}lcccc@{}}
\toprule
& \multicolumn{2}{c}{Reliability} & \multicolumn{2}{c}{Expected count}\\
\cmidrule(lr){2-3}\cmidrule(l){4-5}
Gate & (a) half-slack & (b) midpoint & (c) midpoint & (d) optimized\\
\midrule
$\ANDg_2$ & $1$ & $4$ & $8$ & $2$\\
$\ANDg_3$ & $1$ & $4$ & $12$ & $3$\\
$\ORg_2$ & $9$ & $16$ & $16$ & $6+4\sqrt2\approx11.66$\\
$\MAJg_3$ & $9$ & $16$ & $24$ & $9+6\sqrt2\approx17.49$\\
Recovery, $\THR^1_1$ & $1$ & $4$ & $4$ & $1$\\
\bottomrule
\end{tabular}
\end{table}

Because the capacity is exponential, a ratio of constants becomes a power of the dictionary size. For $r$-way conjunction $c_{\rm a}=1$ and $c_{\rm d}=r$: at equal width, the largest dictionary with a vanishing expected count is the $r$-th root of the largest reliable one, in the sense that the ratio of their logarithms tends to $1/r$. For conjunction the midpoint bias costs a fourth root under both criteria, because $c_{\rm b}/c_{\rm a}=c_{\rm c}/c_{\rm d}=4$; one scalar decides whether the layer tolerates $m$ features or $m^{1/4}$. Error criteria are known to change capacity constants in associative memories \citep{McEliece87}; here they change the constant in the exponent of an exponential capacity.

The constant $c_{\rm a}=(2r-2t+1)^2$ counts the inactive positions in the two binding classes, whose level $\lambda_0$ dominates when $m\gg k$. Conjunction binds on the classes with $r-1$ and $r$ active inputs and pays for one inactive position; disjunction binds on the classes with none and one active input and pays for $2r-1$. Reliable $r$-way conjunction and feature recovery share $c_{\rm a}=1$, so in this regime computing a conjunction directly costs no capacity beyond recovering its inputs. At finite $\beta$ the direct route costs the factor $((2r-1+\beta)/(1+\beta))^2$ in width (\cref{cor:recovery}), which tends to one in the regime of \cref{cor:capacity}.

\subsection{Width for a target success probability}\label{sm:app:sizing}

A first-order threshold locates the centre of a transition, not an operating point: at its quantile centre the half-slack network of pairwise conjunction succeeds with probability about $0.15$ (\cref{tab:window}). \Cref{thm:window} supplies the missing quantile. With $\Psi_\infty(x)=\exp(-e^{-x})$, the limit of $\Psi_\beta$ as $\beta\to\infty$, write
\[
P_{\rm mid}(x)=\Psi_{\beta_0}(x),\qquad P_{\rm half}(x)=\Psi_{\beta_0}(x/2)\,F_+(x/2),\qquad P_{\rm best}(x)=\sup_{y\in\R}\Psi_{\beta_0}(y)\,F_+(x-y).
\]

\begin{corollary}[Width for a target success probability]\label{cor:sizing}
Let $\cF$ be the universal family of pairwise conjunctions, $k\to\infty$, $m/k\to\infty$ and $\beta\to\beta_0\in[1,\infty]$. Each of $P_{\rm mid}$, $P_{\rm half}$ and $P_{\rm best}$ is a continuous, strictly increasing distribution function. For $\delta\in(0,1)$ let $x_\bullet(\delta)$ be the $(1-\delta)$-quantile of $P_\bullet$, and set
\begin{equation}\label{eq:sizing}
d_{\rm mid}(x)=4k\Bigl(\lhat_1+\lhat_0+\frac x{\lhat_1}\Bigr)^2,\qquad d_{\rm opt}(x)=k\Bigl(3\lhat_1+\lhat_0+\frac x{\lhat_1}\Bigr)^2 .
\end{equation}
Fix $\eta>0$. For all large $k$, the midpoint network has $\Prob(X=0)>1-\delta$ at $d=d_{\rm mid}(x_{\rm mid}(\delta)+\eta)$ and $\Prob(X=0)<1-\delta$ at $d=d_{\rm mid}(x_{\rm mid}(\delta)-\eta)$. The same holds for the half-slack bias of \cref{thm:thr}(a) with $d_{\rm opt}$ and $x_{\rm half}$, and for $\sup_\theta\Prob(X=0)$ over deterministic common biases with $d_{\rm opt}$ and $x_{\rm best}$.
\end{corollary}

\begin{proof}
$\Psi_\beta$ is the distribution function of $G_A+G_B/\beta$, with $G_B/\infty=0$, and $F_+$ is that of $-\log[\omega_1(\omega_1+\omega_2)]$. Both variables have continuous positive densities on $\R$, so $\Psi_\beta$, $F_+$ and $P_{\rm half}$ are continuous and strictly increasing. For $P_{\rm best}$, the product $\Psi_{\beta_0}(y)F_+(x-y)$ tends to zero as $y\to\pm\infty$, so the supremum is attained, and it increases strictly in $x$ because $F_+$ does; uniform continuity of $F_+$ gives continuity. For every $y$, either $y\le x/2$ or $x-y<x/2$, so $P_{\rm best}(x)\le\max\{\Psi_{\beta_0}(x/2),F_+(x/2)\}\to0$ as $x\to-\infty$; and $P_{\rm best}\ge P_{\rm half}\to1$ as $x\to\infty$.

Next, expanding \eqref{eq:sizing} gives $d_{\rm mid}(x)=d_0+8k(1+\hat\beta)x+4kx^2/\lhat_1^2$ and $d_{\rm opt}(x)=d_0^{\rm opt}+2k(3+\hat\beta)x+kx^2/\lhat_1^2$, with $\hat\beta=\lhat_0/\lhat_1=\beta(1+o(1))$ (\cref{sm:window}). The window units $8k(1+\beta)\ge16k$ and $2k(3+\beta)\ge8k$ dominate the quadratic terms, so $d_\bullet(x)$ is the window point of \cref{thm:window} at $x+o(1)$, uniformly on compact sets. \Cref{thm:window} then gives success probabilities $P_\bullet(x)+o(1)$ with $\Psi_\beta$ in place of $\Psi_{\beta_0}$. As $\beta\to\beta_0$, $G_A+G_B/\beta\to G_A+G_B/\beta_0$ almost surely and the limit has a continuous distribution function, so $\Psi_\beta\to\Psi_{\beta_0}$ uniformly; the same follows for $P_{\rm half}$ and $P_{\rm best}$, since a supremum moves by at most the uniform distance. Strict monotonicity gives $P_\bullet(x_\bullet(\delta)+\eta)>1-\delta>P_\bullet(x_\bullet(\delta)-\eta)$.
\end{proof}

\begin{table}[!ht]
\caption{Width for a target success probability at $k=1024$ and $m=k^4$, from \eqref{eq:sizing} with the profiles evaluated at $\hat\beta=\lhat_0/\lhat_1=2.28$, in units of the optimized quantile centre $k(3\lhat_1+\lhat_0)^2$. The midpoint centre is $1.54$ in these units. The closed-form half-slack bias stays within $0.6\%$ of the best deterministic common bias.}
\label{tab:sizing}
\centering
\small
\renewcommand{\arraystretch}{1.15}
\begin{tabular}{@{}lcccc@{}}
\toprule
Target $\Prob(X=0)$ & $0.5$ & $0.9$ & $0.99$ & $0.999$\\
\midrule
Midpoint bias & $1.61$ & $1.81$ & $2.08$ & $2.35$\\
Half-slack bias & $1.10$ & $1.29$ & $1.54$ & $1.80$\\
Best deterministic common bias & $1.09$ & $1.29$ & $1.54$ & $1.79$\\
\bottomrule
\end{tabular}
\end{table}

\Cref{eq:sizing} is the first-order threshold with a slack $x_\bullet(\delta)/\lhat_1$ added to the extreme-value levels. In relative terms the overhead of the half-slack rule is $(1+x_{\rm half}(\delta)/(\lhat_1(3\lhat_1+\lhat_0)))^2-1$, of order $x_{\rm half}(\delta)/\log k$ when $\beta$ stays bounded: it vanishes, but only logarithmically. For $99\%$ success at $m=k^4$, the half-slack width is $2.20$, $1.75$, $1.54$, $1.35$ and $1.22$ times its centre at $k=2^6$, $2^8$, $2^{10}$, $2^{14}$ and $2^{20}$. \Cref{tab:sizing} adds two design facts at $k=1024$. A bias that needs only $k$, $m$ and $d$ costs at most $0.6\%$ in width against the best deterministic bias. The midpoint bias needs $1.30$ to $1.46$ times the width of the half-slack bias, depending on the target.

The limit profiles are conservative at the simulated sizes. In all $77$ frozen cells above the quantile centre for pairwise conjunction (the half-slack bias in campaigns E1 and E2, the midpoint in E3), the limit lies below ($44$ cells) or inside ($33$ cells) the $95\%$ Wilson interval of the observed success frequency, and never above the observed frequency. At $k=1024$ and $m=k^4$, for instance, the half-slack network at $1.5$ times the centre succeeds on $99.76\%$ of draws, with interval $[99.55\%,99.87\%]$, whereas \eqref{eq:sizing} asks for $1.54$ times the centre to reach $99\%$. The direction has a simple source. The window replaces the upper tail $k\Phibar(\lhat_1+y/\lhat_1)$ of a Gaussian maximum by its Gumbel limit $e^{-y}$, and Mills' ratio gives $k\Phibar(\lhat_1+y/\lhat_1)\approx e^{-y-y^2/(2\lhat_1^2)}$, a lighter tail exactly where a high success target probes. Within the tested range the rule therefore errs toward extra width.

\subsection{Auditing a layer on a realized representation}\label{sm:app:audit}

The box identity (\cref{lem:boxid}) is deterministic. It therefore applies to any stored vector and any read directions, provided that each gate reads a fixed combination of per-feature reads. Verifying a universal layer then needs $2r$ extreme reads per pool instead of an enumeration of $(m)_r$ gates.

\begin{proposition}[Exact audit of a factored layer]\label{prop:audit}
Let $u\in\R^d$ and $v_1,\dots,v_m\in\R^d$ be arbitrary, let $S\subseteq[m]$ satisfy $r\le\abs S\le m-r$, and put $R_i=\ip{v_i}u$. Evaluate $\cU_r(g)$ with readouts $w_I=\sum_la_lv_{I_l}$ and a common bias $\theta$, where $(a,\theta)\in\Rep(g)$, and let $B_z(R)$ be the right side of \cref{lem:boxid} with $\xi=R$ and $b=\varsigma_za$.
\begin{enumerate}[label=\textup{(\roman*)},leftmargin=2em]
\item Every gate is correct if and only if $B_z(R)\le\varsigma_z\theta$ for each class with $g(z)=0$ and $B_z(R)<\varsigma_z\theta$ for each class with $g(z)=1$.
\item Given $R$, the $2^r$ statistics $B_z(R)$ depend only on the $r$ largest and the $r$ smallest reads in $S$ and in $[m]\setminus S$, and they take $O(rm+2^rr\log r)$ operations. For $r=2$ the error count $X$ takes $O(m\log m)$ operations.
\item Suppose $\Delta=\min_z(\varsigma_z\theta-B_z(R))>0$. Every gate stays correct for every read vector $R'$ with $\norm{R'-R}_\infty<\Delta/\norm a_1$, and hence, if some $v_i\ne0$, for every stored vector $u'$ with $\norm{u'-u}<\Delta/(\norm a_1\max_i\norm{v_i})$. Conversely, if a class $z^*$ attains $\Delta$ with a maximizing tuple $I^*$ and $w_{I^*}\ne0$, then $u'=u+t\varsigma_{z^*}w_{I^*}$ makes gate $I^*$ wrong for every $t>\Delta/\norm{w_{I^*}}^2$.
\end{enumerate}
\end{proposition}

\begin{proof}
Gate $I$ outputs $\one\{\ip a{R_I}>\theta\}$, because $\ip{w_I}u=\ip a{R_I}$. For a class with $g(z)=0$, where $\varsigma_z=1$, it is correct exactly when $\ip{\varsigma_za}{R_I}\le\varsigma_z\theta$; for a class with $g(z)=1$, where $\varsigma_z=-1$, exactly when $\ip{\varsigma_za}{R_I}<\varsigma_z\theta$. Taking the maximum over the class and applying \cref{lem:boxid}, whose pool condition is $r\le\abs S\le m-r$, gives (i). For (ii), one pass over each pool with sorted buffers of length $r$ finds its $r$ largest and $r$ smallest reads, and each class then needs the sorted weights of its four groups. For $r=2$, sort each pool; for every read of one pool, binary search counts the partners, other than the read itself, in the same or the other pool that put a gate of the corresponding class on the wrong side of $\theta$. For (iii), each $\ip{\varsigma_za}{R_I}$ is a linear function of $R$ whose coefficients have $\ell_1$ norm $\norm a_1$, because the features of $I$ are distinct. Hence $B_z(R')\le B_z(R)+\norm a_1\norm{R'-R}_\infty<\varsigma_z\theta$ for every class, and (i) applies; the bound for $u'$ follows from $\abs{R'_i-R_i}\le\norm{v_i}\norm{u'-u}$. Finally, $u'=u+t\varsigma_{z^*}w_{I^*}$ raises $\ip{\varsigma_{z^*}a}{R_{I^*}}$ by exactly $t\norm{w_{I^*}}^2>\Delta$, so it exceeds $\varsigma_{z^*}\theta$ and gate $I^*$ is wrong whatever $g(z^*)$ is.
\end{proof}

Under the model of \cref{sec:setting}, with $v_i=e_i$, part (i) is \cref{prop:cloud} without exceptional events, and part (ii) is the reason the exact sampler costs $O(k\log k)$ per draw at any $m$. Part (iii) gives the whole layer a robustness margin, and above the box threshold this margin is of the order of the clean margins. Let $d\ge(1+\varepsilon)\Dbox(a,\theta)$ for the universal family. Since $R_I=Hz+\sigma\xi_I$, we have $\varsigma_z\theta-B_z(R)=s^H_z-\sigma M_z$, and on the event $\mathcal E_\delta$ in the proof of \cref{thm:box} this is at least $s_z-(1+5\delta)\sigd h_{\Boxz z}(a)\ge s_z(1-(1+5\delta)(1+\varepsilon)^{-1/2})$ for every class. Moreover $d\ge\Dbox\ge k\lambda_0^2/(4r)$ by \eqref{sm:eq:marginbound}, so $\log m=o(d)$, and chi-square and Gaussian tails with a union bound give $\max_i\norm{e_i}\to1$ and $\max_I\abs{\norm{w_I}/\norm a-1}\to0$ in probability. With probability tending to one, therefore, every perturbation of the stored vector with
\[
\norm{u'-u}\le\bigl(1-(1+\varepsilon)^{-1/2}-o(1)\bigr)\frac{\min_zs_z}{\norm a_1}
\]
leaves all $(m)_r$ outputs correct, while a perturbation of norm $\Delta/\norm a\,(1+o(1))$ along the probe of a binding gate breaks one. For the representations of \cref{thm:thr}, which have $a=\one$, both radii are of order one, whereas $\norm u\approx\sqrt k$.

Two uses follow. The first is bias calibration from data. The oracle of campaign E2 succeeds when some common bias separates the largest negative score from the smallest positive score; by (ii), this event and a separating bias are computable from the extremes. \Cref{thm:box}(ii) shows that no such read-dependent common bias moves the reliability threshold, and in campaign E2 at $k=1024$ the oracle reaches $50\%$ success only $2\%$ to $7\%$ earlier in width than the closed-form half-slack bias (\cref{tab:e2}). The second is a direct test of the mechanism in learned models, which we propose as an experiment. Toy models of superposition and of learned universal-AND circuits \citep{Elhage22,HanniMVC24,Newgas25} have identifiable features. Fit per-feature read directions $v_i$ and a common representation $(a,\theta)$ to the learned readouts. The fitting residual measures the distance from the factored class, and on the factored part (i) predicts every failing input from $2r$ extremes per pool. Agreement between predicted and observed failures would show that the learned layer fails through shared extremes, as the box describes.

\subsection{Unread activity in the stream}\label{sm:app:unread}

The vector read by a gate family usually carries more than the features that the family reads: other dictionaries, earlier writes, positional content. The exact law extends to independent Gaussian content, and the extension separates the two roles that $k$ plays in every threshold.

\begin{proposition}[Unread content]\label{prop:unread}
Let $u=Ex+\nu g$, where $g\sim N(0,I_d/d)$ is independent of $(E,S)$ and $\nu\ge0$ may depend on $k$. Put $K=k+\nu^2$ and $c=1-\nu/\sqrt K\in(0,1]$. With $H,Z,W$ as in \cref{thm:exact} and $\sigma=\sqrt{KH/d}$,
\[
(A_i,B_j)\overset d=\bigl(H+\sigma(Z_i-c\bar Z),\ \sigma W_j\bigr)
\]
jointly in $i\le k$ and $j\le n$, and a canonical gate of class $z$ has the score law \eqref{eq:onegate} with $k$ replaced by $K$. Consequently \cref{thm:ball,thm:weighted,thm:sup,thm:box,thm:thr,thm:window} hold with $\sigd=\sqrt{K/d}$: the prefactor $k$ of every threshold, centre and window width becomes $K$, the condition $d/k\to\infty$ in \cref{thm:sup} becomes $d/K\to\infty$, and $\lambda_1$, $\lambda_0$, $\lhat_1$, $\lhat_0$, $N_z$, the programs \eqref{eq:socp} and \eqref{eq:vbox} and the window profiles are unchanged.
\end{proposition}

\begin{proof}
Fix $S=[k]$. Each coordinate of $u$ is $N(0,K/d)$, and $\Cov(e_i,u)=I_d/d$ for $i\le k$. Gaussian conditioning gives $e_i=u/K+\eta_i$, where $(\eta_i)_{i\le k}$ is independent of $u$ with $\Cov(\eta_i,\eta_l)=d^{-1}(\one\{i=l\}-1/K)I_d$. Hence $A_i=\norm u^2/K+\ip{\eta_i}u$, and given $u$ the vector $(\ip{\eta_i}u)_{i\le k}$ is centered Gaussian with covariance $(\norm u^2/d)(I_k-\one\one^\top/K)$. Since $(1-c)^2=\nu^2/K=1-k/K$, the vector $(Z_i-c\bar Z)_{i\le k}$ has covariance $I_k-(2c-c^2)k^{-1}\one\one^\top=I_k-\one\one^\top/K$. Put $H=\norm u^2/K\sim\chi^2_d/d$, so that $\norm u/\sqrt d=\sigma$. The inactive columns are independent of $(u,\eta)$, so given $u$ the reads $B_j$ are independent $N(0,\sigma^2)$ variables, independent of the active reads. For one gate, the coordinates of $(w,u)$ are independent pairs with variances $\norm a^2/d$ and $K/d$ and covariance $\ip az/d$, and the proof in \cref{sm:reduction} applies with $K$ in place of $k$.

The proofs in \cref{sm:reduction,sm:ball,sm:box,sm:gates,sm:window} use $k$ in two roles. As the size of the active pool, it fixes $\lambda_1$, $\lhat_1$, $N_z$ and the law of the active order statistics. As a variance, it enters through $\sigd^2=k/d$, the one-gate variance $k\norm a^2-\ip az^2$ and the conditional variance $(k-\abs z)/d$ of a local probe in the converse of \cref{thm:weighted}. The representation above leaves the first role unchanged and replaces $k$ by $K$ in the second; the local-probe variance becomes $(K-\abs z)/d$ because the unread content is independent of every column. It also replaces $\bar Z$ by $c\bar Z$, which only shrinks the centering terms, and every estimate that used $k\to\infty$ in the second role holds a fortiori with $K\ge k$. Auxiliary rates indexed by $k$, such as $k^{-1/4}$ and $1/\log k$ in \cref{sm:ball}, can be kept, since every bound they control only improves when $K\ge k$. These substitutions give the stated forms.
\end{proof}

Width is paid for activity; the logarithms are paid for the read family. Independent content of squared norm $\nu^2$ adds $\nu^2$ units to the interference variance but no extremes: the levels, the ratio between ball and box, the binding classes and the window profiles are those of the read family alone. If the stream also carries $\nu^2$ active directions of another independent Gaussian dictionary, whose sum is exactly $N(0,\nu^2I_d/d)$, every threshold in \cref{tab:four} is multiplied by $1+\nu^2/k$. When the read family accounts for a quarter of the active directions, the widths quadruple. The half-slack bias must be computed with $\sigd=\sqrt{K/d}$ as well; a bias calibrated to the read activity alone sits too close to the midpoint. A fixed number of families on independent dictionaries can share one stream: each sees the others as unread content, and a union bound over families composes their reliability. \Cref{prop:unread} is the sharp one-layer counterpart of the stream activity $K$ in \cref{thm:writebudget}.

\subsection{Evaluating shared errors}\label{sm:app:evaluation}

The gap between the ball and the box is a gap between a mean and a typical value. When $\Prob(X>0)>0$,
\begin{equation}\label{eq:burden}
\Exp X=\Prob(X>0)\,\Exp[X\mid X>0].
\end{equation}
For the universal family between the two thresholds, $(1+\varepsilon)\Dbox\le d\le(1-\varepsilon)\Dball$, the failure probability tends to zero and the mean to infinity, so the burden $\Exp[X\mid X>0]$ of a failed draw grows faster than the mean. For $B$ independent evaluation draws, $\Prob(\text{no draw fails})=\Prob(X=0)^B\ge1-B\,\Prob(X>0)$. Whenever $B\,\Prob(X>0)\to0$, the sample mean of $X$ and its standard error are both zero with probability tending to one, while $\Exp X\to\infty$. The frozen campaigns show the effect at both ends. At $m=k^{64}$ and $d=1.5\,d_0$, all $4096$ draws are correct while quadrature gives $\Exp X\approx2.8\times10^{19}$ (\cref{tab:e3}). At $m=k^4$ and $1.35$ times the centre, even the Rao--Blackwellized sample mean is $0.063$ with standard error $0.011$, against the quadrature value $0.327$ (\cref{tab:e1}).

Four quantities make an evaluation of a superposed layer informative. The first is whole-output success with a binomial interval. The second is the expected count from an exact one-gate law (\cref{prop:onegate}) rather than from a sample mean. The third is the burden $\Exp[X\mid X>0]$ together with the class of each wrong output, which separates frequent small failures from rare large ones. The fourth is the slack $\Delta$ of \cref{prop:audit} on every draw. The slack is continuous, so its empirical distribution is informative even when no draw fails, and in the window its law is explicit. For the midpoint network at $d=d_0+8k(1+\beta)x$ with $\beta\to\beta_0\in[1,\infty]$, the computation in \cref{sm:window} gives
\[
\frac{\lhat_1\Delta}{\sigd}\Rightarrow x-G_A-G_B/\beta_0 ,
\]
because the slack of the one-active class is $\tfrac32-\max_iA_i-\max_jB_j$, whose expansion in \cref{sm:window} gives the limit, while the other classes keep a diverging slack in these units. A Gumbel fit to observed slacks therefore extrapolates the failure probability below the level that the draws resolve.

\subsection{Implications for claims about superposition}\label{sm:app:claims}

Models of computation in superposition \citep{Elhage22,HanniMVC24,AdlerShavit24} are used to interpret what a layer of a network computes. \citet{ShiLHH26} survey the geometry, learning and computation of superposed representations and examine what evaluations of feature-recovery methods establish, noting that accurate reconstruction of activations does not by itself establish feature identity or causal use. The results here give claims about computation in superposition a checkable form in the same spirit.

\paragraph{Name the criterion}
Losses and error rates averaged over gates measure the first row of \cref{tab:criteria}, which does not depend on $m$. A layer can have an expected error fraction below $10^{-21}$ and fail on most inputs (\cref{sm:app:criteria}). Learned solutions of the universal-AND task \citep{Newgas25} and of related compressed-computation tasks \citep{BhagatMGH26} need not use the mechanism of theoretical constructions. A claim that a learned layer computes in superposition should therefore state which row of the table it establishes.

\paragraph{Reliability is not robustness}
The reliability laws average over supports drawn independently of the dictionary. Two adversaries break them at very different costs. A support chosen after the dictionary raises the width to order $k^2\log m$ (\cref{prop:uniform}). A perturbation of the stored vector along the probe of a binding gate needs a norm of order one, out of $\norm u\approx\sqrt k$ (\cref{prop:audit}(iii)). A guarantee certified on random inputs is a statement about typical use.

\paragraph{Lower bounds concern accessibility}
The $\Theta(k\log m)$ scale belongs to threshold reads of a random linear code. Without the linear dictionary, $2k$ coordinates carry all pairwise conjunctions with one threshold unit each (\cref{prop:polycode}), and order $k$ coordinates are necessary (\cref{prop:counting}). A width far below the number of features is therefore consistent with several mechanisms, and a dimension count alone does not identify computation in superposition.

\subsection{Open problems}\label{sm:app:open}

\paragraph{Reliability for gate-local rules}
\Cref{thm:box}(ii) covers every common representation, even one chosen after the reads are observed, and \cref{thm:weighted} covers every gate-local rule for the expected count. Can a gate-local rule, for instance the dual frame of \cref{sm:rem:dual}, make the network reliable below $k\Vbox$? A converse must control the joint extremes of adapted probes. Marginal error bounds cannot force a failure on a typical draw, since $\Prob(X>0)$ can vanish while $\Exp X$ diverges.

\paragraph{Sharp suprema for structured families}
Disjoint tuples give the ball, the universal family gives the box, and \cref{prop:star} interpolates for one level of private partners. For a family defined by a sharing hypergraph with several levels, generic chaining determines $\Gamma_z$ up to constants \citep{Talagrand14}. A sharp variational formula, possibly of the type of the generalized random energy model \citep{Derrida85}, would give the reliability price of any query family, for example relation queries in which each feature has a bounded number of partners at each level.

\paragraph{Universality of the box}
The one-gate exponent is universal (\cref{thm:universality}), but the box rests on the Gaussian joint law. A first target is the Rademacher dictionary with midpoint pairwise conjunction: is its reliability threshold $4k(\lambda_1+\lambda_0)^2(1+o(1))$? The required input is joint extreme-value control of the active and inactive reads, which are dependent through $u$.

\paragraph{Reliability of composed circuits}
\Cref{thm:writebudget} is a union bound over an ideal execution, that is, a ball argument. In a two-layer universal circuit the second layer reads $\binom k2$ active outputs, whose errors are correlated through the extremes of the first layer. A box law for the second layer needs the joint law of its reads given the first layer's realized errors; \cref{prop:unread} indicates, through the variance term, that it pays for its full stream activity.

\paragraph{Designed and learned linear dictionaries}
The Gaussian dictionary needs order $k\log m$ dimensions, and arbitrary encodings need order $k$ (\cref{prop:counting,prop:polycode}) at the price of large coefficients and small normalized margins. For uniform recovery, every linear encoder needs $\Omega(k^2\log(m/k)/\log k)$ dimensions when $k<\sqrt m$ \citep{GKP26}. For random supports, is order $k\log m$ necessary for every linear dictionary read by threshold units, or can a dictionary adapted to the support law, with bounded coefficients and margins, approach order $k$?

\paragraph{Non-asymptotic sizing}
The limit profiles of \cref{thm:window} were conservative in every simulated cell above the centre, and the Gaussian tail factor $e^{-y^2/(2\lhat_1^2)}$ explains the direction (\cref{sm:app:sizing}). A second-order window law with explicit error terms would turn \eqref{eq:sizing} into a certified width at finite $k$. For general threshold gates, even the first-order window requires the joint law of $r$ upper and $r$ lower order statistics in both pools.

%% file: refs.bib
@misc{AdlerShavit24,
  author = {Adler, Micah and Shavit, Nir},
  title = {On the Complexity of Neural Computation in Superposition},
  year = {2024},
  howpublished = {arXiv:2409.15318},
  url = {https://arxiv.org/abs/2409.15318},
}

@article{ArratiaGoldsteinGordon1989,
  author = {Arratia, Richard and Goldstein, Larry and Gordon, Louis},
  title = {Two Moments Suffice for {P}oisson Approximations: The {C}hen--{S}tein Method},
  journal = {The Annals of Probability},
  volume = {17},
  number = {1},
  pages = {9--25},
  year = {1989},
  doi = {10.1214/aop/1176991491},
}

@book{BarbourHolstJanson1992,
  author = {Barbour, A. D. and Holst, Lars and Janson, Svante},
  title = {Poisson Approximation},
  volume = {2},
  series = {Oxford Studies in Probability},
  publisher = {Oxford University Press},
  address = {Oxford},
  year = {1992},
  doi = {10.1093/oso/9780198522355.001.0001},
}

@misc{BhagatMGH26,
  author = {Bhagat, Jai and Molas-Medina, Sara and Giglemiani, Giorgi and Heimersheim, Stefan},
  title = {Compressed Computation is (Probably) not Computation in Superposition},
  year = {2026},
  howpublished = {arXiv:2606.14673},
  url = {https://arxiv.org/abs/2606.14673},
}

@article{Buck43,
  author = {Buck, R. Creighton},
  title = {Partition of Space},
  journal = {The American Mathematical Monthly},
  volume = {50},
  number = {9},
  pages = {541--544},
  year = {1943},
  doi = {10.1080/00029890.1943.11991447},
}

@article{CandesTao06,
  author = {Cand\`{e}s, Emmanuel J. and Tao, Terence},
  title = {Near-Optimal Signal Recovery from Random Projections: Universal Encoding Strategies?},
  journal = {IEEE Transactions on Information Theory},
  volume = {52},
  number = {12},
  pages = {5406--5425},
  year = {2006},
  doi = {10.1109/TIT.2006.885507},
}

@article{Cover65,
  author = {Cover, Thomas M.},
  title = {Geometrical and Statistical Properties of Systems of Linear Inequalities with Applications in Pattern Recognition},
  journal = {IEEE Transactions on Electronic Computers},
  volume = {EC-14},
  number = {3},
  pages = {326--334},
  year = {1965},
  doi = {10.1109/PGEC.1965.264137},
}

@article{Donoho06,
  author = {Donoho, David L.},
  title = {Compressed Sensing},
  journal = {IEEE Transactions on Information Theory},
  volume = {52},
  number = {4},
  pages = {1289--1306},
  year = {2006},
  doi = {10.1109/TIT.2006.871582},
}

@article{DJS22,
  author = {D\"oring, Hanna and Jansen, Sabine and Schubert, Kristina},
  title = {The Method of Cumulants for the Normal Approximation},
  journal = {Probability Surveys},
  volume = {19},
  pages = {185--270},
  year = {2022},
  doi = {10.1214/22-PS7},
  note = {arXiv:2102.01459},
}

@misc{Elhage22,
  author = {Elhage, Nelson and Hume, Tristan and Olsson, Catherine and Schiefer, Nicholas and Henighan, Tom and Kravec, Shauna and Hatfield-Dodds, Zac and Lasenby, Robert and Drain, Dawn and Chen, Carol and Grosse, Roger and McCandlish, Sam and Kaplan, Jared and Amodei, Dario and Wattenberg, Martin and Olah, Christopher},
  title = {Toy Models of Superposition},
  year = {2022},
  howpublished = {Transformer Circuits Thread (Anthropic)},
  note = {arXiv:2209.10652},
  url = {https://transformer-circuits.pub/2022/toy_model/index.html},
}

@article{EvansSchulman99,
  author = {Evans, William S. and Schulman, Leonard J.},
  title = {Signal Propagation and Noisy Circuits},
  journal = {IEEE Transactions on Information Theory},
  volume = {45},
  number = {7},
  pages = {2367--2373},
  year = {1999},
  doi = {10.1109/18.796377},
}

@inproceedings{GKP26,
  author = {Garg, Nikhil and Kleinberg, Jon and Peng, Kenny},
  title = {How Many Features Can a Language Model Store Under the Linear Representation Hypothesis?},
  booktitle = {Proceedings of Thirty Ninth Conference on Learning Theory},
  editor = {Hanneke, Steve and Lattimore, Tor},
  series = {Proceedings of Machine Learning Research},
  volume = {336},
  pages = {5358--5376},
  publisher = {PMLR},
  year = {2026},
  url = {https://proceedings.mlr.press/v336/garg26a.html},
}

@misc{HanniMVC24,
  author = {H\"anni, Kaarel and Mendel, Jake and Vaintrob, Dmitry and Chan, Lawrence},
  title = {Mathematical Models of Computation in Superposition},
  year = {2024},
  howpublished = {ICML 2024 Mechanistic Interpretability Workshop; arXiv:2408.05451},
  url = {https://arxiv.org/abs/2408.05451},
}

@article{Hastad94,
  author = {H{\aa}stad, Johan},
  title = {On the Size of Weights for Threshold Gates},
  journal = {SIAM Journal on Discrete Mathematics},
  volume = {7},
  number = {3},
  pages = {484--492},
  year = {1994},
  doi = {10.1137/S0895480192235878},
}

@book{LeadbetterLindgrenRootzen1983,
  author = {Leadbetter, M. R. and Lindgren, Georg and Rootz\'en, Holger},
  title = {Extremes and Related Properties of Random Sequences and Processes},
  series = {Springer Series in Statistics},
  publisher = {Springer-Verlag},
  address = {New York},
  year = {1983},
  doi = {10.1007/978-1-4612-5449-2},
}

@book{Muroga71,
  author = {Muroga, Saburo},
  title = {Threshold Logic and Its Applications},
  publisher = {Wiley-Interscience},
  address = {New York},
  year = {1971},
  isbn = {978-0-471-62530-8},
}

@misc{Newgas25,
  author = {Newgas, Adam},
  title = {Compressed Computation: Dense Circuits in a Toy Model of the Universal-{AND} Problem},
  year = {2025},
  howpublished = {arXiv:2507.09816},
  url = {https://arxiv.org/abs/2507.09816},
}

@book{PetrovBook75,
  author = {Petrov, Valentin V.},
  title = {Sums of Independent Random Variables},
  volume = {82},
  series = {Ergebnisse der Mathematik und ihrer Grenzgebiete},
  publisher = {Springer-Verlag},
  address = {Berlin},
  year = {1975},
  doi = {10.1007/978-3-642-65809-9},
}

@incollection{vonNeumann56,
  author = {von Neumann, John},
  title = {Probabilistic Logics and the Synthesis of Reliable Organisms from Unreliable Components},
  booktitle = {Automata Studies},
  editor = {Shannon, Claude E. and McCarthy, John},
  number = {34},
  pages = {43--98},
  series = {Annals of Mathematics Studies},
  publisher = {Princeton University Press},
  year = {1956},
  doi = {10.1515/9781400882618-003},
}

@article{Winder66,
  author = {Winder, Robert O.},
  title = {Partitions of {$N$}-Space by Hyperplanes},
  journal = {SIAM Journal on Applied Mathematics},
  volume = {14},
  number = {4},
  pages = {811--818},
  year = {1966},
  doi = {10.1137/0114068},
}

@misc{IvanitskiyJKM26,
  author = {Ivanitskiy, Michael I. and Jasper, John and King, Emily J. and Mixon, Dustin G.},
  title = {Towards a Mathematical Theory of Superposition},
  year = {2026},
  howpublished = {arXiv:2608.27540},
  url = {https://arxiv.org/abs/2608.27540},
}

@misc{Barvinok2010,
 author = {Barvinok, Alexander},
 title = {{MATH 669}: Combinatorics of Polytopes},
 year = {2010},
 note = {Winter 2010 lecture notes, University of Michigan, Theorem 15.2},
 url = {https://sites.lsa.umich.edu/barvinok/wp-content/uploads/sites/1434/2025/05/polynotes669.pdf},
}

@article{Borell75,
  author = {Borell, Christer},
  title = {The {B}runn--{M}inkowski inequality in {G}auss space},
  journal = {Inventiones Mathematicae},
  volume = {30},
  number = {2},
  pages = {207--216},
  year = {1975},
  doi = {10.1007/BF01425510},
}

@incollection{TIS76,
  author = {Tsirel'son, Boris S. and Ibragimov, Ildar A. and Sudakov, Vladimir N.},
  title = {Norms of {G}aussian sample functions},
  booktitle = {Proceedings of the Third Japan--USSR Symposium on Probability Theory},
  series = {Lecture Notes in Mathematics},
  volume = {550},
  pages = {20--41},
  publisher = {Springer},
  address = {Berlin},
  year = {1976},
  doi = {10.1007/BFb0077482},
}

@article{Slepian62,
  author = {Slepian, David},
  title = {The one-sided barrier problem for {G}aussian noise},
  journal = {Bell System Technical Journal},
  volume = {41},
  number = {2},
  pages = {463--501},
  year = {1962},
  doi = {10.1002/j.1538-7305.1962.tb02419.x},
}

@incollection{Fernique75,
  author = {Fernique, Xavier},
  title = {R\'egularit\'e des trajectoires des fonctions al\'eatoires gaussiennes},
  booktitle = {\'Ecole d'\'Et\'e de Probabilit\'es de Saint-Flour IV---1974},
  series = {Lecture Notes in Mathematics},
  volume = {480},
  pages = {1--96},
  publisher = {Springer},
  address = {Berlin},
  year = {1975},
  doi = {10.1007/BFb0080190},
}

@book{BoucheronLugosiMassart13,
  author = {Boucheron, St\'ephane and Lugosi, G\'abor and Massart, Pascal},
  title = {Concentration Inequalities: A Nonasymptotic Theory of Independence},
  publisher = {Oxford University Press},
  address = {Oxford},
  year = {2013},
  doi = {10.1093/acprof:oso/9780199535255.001.0001},
}

@book{Resnick87,
  author = {Resnick, Sidney I.},
  title = {Extreme Values, Regular Variation, and Point Processes},
  series = {Applied Probability},
  volume = {4},
  publisher = {Springer-Verlag},
  address = {New York},
  year = {1987},
}

@article{Derrida81,
  author = {Derrida, Bernard},
  title = {Random-energy model: An exactly solvable model of disordered systems},
  journal = {Physical Review B},
  volume = {24},
  number = {5},
  pages = {2613--2626},
  year = {1981},
  doi = {10.1103/PhysRevB.24.2613},
}

@book{Talagrand14,
  author = {Talagrand, Michel},
  title = {Upper and Lower Bounds for Stochastic Processes: Modern Methods and Classical Problems},
  series = {Ergebnisse der Mathematik und ihrer Grenzgebiete. 3. Folge},
  volume = {60},
  publisher = {Springer},
  address = {Berlin},
  year = {2014},
  doi = {10.1007/978-3-642-54075-2},
}

@article{DonohoTanner09,
  author = {Donoho, David L. and Tanner, Jared},
  title = {Counting faces of randomly projected polytopes when the projection radically lowers dimension},
  journal = {Journal of the American Mathematical Society},
  volume = {22},
  number = {1},
  pages = {1--53},
  year = {2009},
  doi = {10.1090/S0894-0347-08-00600-0},
}

@article{LaurentMassart00,
  author = {Laurent, B\'eatrice and Massart, Pascal},
  title = {Adaptive estimation of a quadratic functional by model selection},
  journal = {The Annals of Statistics},
  volume = {28},
  number = {5},
  pages = {1302--1338},
  year = {2000},
  doi = {10.1214/aos/1015957395},
}

@book{Vershynin18,
  author = {Vershynin, Roman},
  title = {High-Dimensional Probability: An Introduction with Applications in Data Science},
  publisher = {Cambridge University Press},
  address = {Cambridge},
  year = {2018},
  doi = {10.1017/9781108231596},
}

@article{JosephBarron14,
  author = {Joseph, Antony and Barron, Andrew R.},
  title = {Fast Sparse Superposition Codes Have Near Exponential Error Probability for {$R<\mathcal{C}$}},
  journal = {IEEE Transactions on Information Theory},
  volume = {60},
  number = {2},
  pages = {919--942},
  year = {2014},
  doi = {10.1109/TIT.2013.2289865},
}

@inproceedings{BarronJoseph11,
  author = {Barron, Andrew R. and Joseph, Antony},
  title = {Analysis of fast sparse superposition codes},
  booktitle = {2011 IEEE International Symposium on Information Theory Proceedings},
  pages = {1772--1776},
  year = {2011},
  doi = {10.1109/ISIT.2011.6033853},
}

@article{Derrida85,
  author = {Derrida, Bernard},
  title = {A generalization of the {R}andom {E}nergy {M}odel which includes correlations between energies},
  journal = {Journal de Physique Lettres},
  volume = {46},
  number = {9},
  pages = {401--407},
  year = {1985},
  doi = {10.1051/jphyslet:01985004609040100},
}

@article{Olah20,
  author = {Olah, Chris and Cammarata, Nick and Schubert, Ludwig and Goh, Gabriel and Petrov, Michael and Carter, Shan},
  title = {Zoom In: An Introduction to Circuits},
  journal = {Distill},
  volume = {5},
  number = {3},
  year = {2020},
  doi = {10.23915/distill.00024.001},
}

@article{Sudakov71,
  author = {Sudakov, Vladimir N.},
  title = {Gaussian random processes and measures of solid angles in {H}ilbert space},
  journal = {Soviet Mathematics Doklady},
  volume = {12},
  pages = {412--415},
  year = {1971},
}

@article{Wainwright09,
  author = {Wainwright, Martin J.},
  title = {Sharp Thresholds for High-Dimensional and Noisy Sparsity Recovery Using $\ell_1$-Constrained Quadratic Programming ({L}asso)},
  journal = {IEEE Transactions on Information Theory},
  volume = {55},
  number = {5},
  pages = {2183--2202},
  year = {2009},
  doi = {10.1109/TIT.2009.2016018},
}

@article{FletcherRanganGoyal09,
  author = {Fletcher, Alyson K. and Rangan, Sundeep and Goyal, Vivek K.},
  title = {Necessary and Sufficient Conditions for Sparsity Pattern Recovery},
  journal = {IEEE Transactions on Information Theory},
  volume = {55},
  number = {12},
  pages = {5758--5772},
  year = {2009},
  doi = {10.1109/TIT.2009.2032726},
}

@article{CandesPlan09,
  author = {Cand\`{e}s, Emmanuel J. and Plan, Yaniv},
  title = {Near-ideal model selection by $\ell_1$ minimization},
  journal = {The Annals of Statistics},
  volume = {37},
  number = {5A},
  pages = {2145--2177},
  year = {2009},
  doi = {10.1214/08-AOS653},
}

@article{McEliece87,
  author = {McEliece, Robert J. and Posner, Edward C. and Rodemich, Eugene R. and Venkatesh, Santosh S.},
  title = {The capacity of the {H}opfield associative memory},
  journal = {IEEE Transactions on Information Theory},
  volume = {33},
  number = {4},
  pages = {461--482},
  year = {1987},
  doi = {10.1109/TIT.1987.1057328},
}

@article{Plate95,
  author = {Plate, Tony A.},
  title = {Holographic reduced representations},
  journal = {IEEE Transactions on Neural Networks},
  volume = {6},
  number = {3},
  pages = {623--641},
  year = {1995},
  doi = {10.1109/72.377968},
}

@article{ThomasDasguptaRosing21,
  author = {Thomas, Anthony and Dasgupta, Sanjoy and Rosing, Tajana},
  title = {A Theoretical Perspective on Hyperdimensional Computing},
  journal = {Journal of Artificial Intelligence Research},
  volume = {72},
  pages = {215--249},
  year = {2021},
  doi = {10.1613/jair.1.12664},
}

@misc{ShiLHH26,
  author = {Shi, Dai and Li, Xiaoyu and Han, Andi and Hern{\'a}ndez-Lobato, Jos{\'e} Miguel},
  title = {Feature Superposition in Neural Networks: From Theory to Practice},
  year = {2026},
  howpublished = {arXiv:2609.06862},
  url = {https://arxiv.org/abs/2609.06862},
}
